\documentclass[10pt,twocolumn,journal]{IEEEtran}

\usepackage{xr}
\usepackage[colorlinks=true,linkcolor=black,citecolor=black,urlcolor=black]{hyperref}
\newcommand{\subparagraph}{}

\usepackage[mathscr]{eucal}
\usepackage{epsfig,epsf,psfrag}
\usepackage{amssymb,amsmath,amsfonts,latexsym}
\usepackage{graphicx}
\usepackage{epstopdf}
\usepackage{bm,xcolor,url}
\usepackage[caption=false]{subfig} 
\usepackage{fixltx2e}
\usepackage{array}
\usepackage{verbatim}
\usepackage{bm}
\usepackage{algpseudocode, cite}
\usepackage{algorithm}
\usepackage{verbatim}
\usepackage{textcomp}
\usepackage{mathrsfs}
\usepackage{relsize}
\usepackage{subfig}
\usepackage{amsthm}
\usepackage{enumerate}
\usepackage{setspace}
\usepackage{cases}
\usepackage{footnote}
\usepackage{tabularx,tabu}
\usepackage{tablefootnote}
\usepackage{threeparttable}
\usepackage{tikz}
\usetikzlibrary{datavisualization}
\usetikzlibrary{datavisualization.formats.functions}
\usetikzlibrary{decorations.pathreplacing}
\usepackage{textcomp}
\usepackage{multicol}
\usepackage{lipsum}
 
\catcode`~=11 \def\UrlSpecials{\do\~{\kern -.15em\lower .7ex\hbox{~}\kern .04em}} \catcode`~=13 

\allowdisplaybreaks[3]

\newcommand{\vecz}{\mathbf{0}}

\newcommand{\norm}[1]{\left\Vert#1\right\Vert}
\newcommand{\normt}[1]{\Vert#1\Vert}

\newcommand{\abs}[1]{\left\lvert#1\right\rvert}

\newcommand{\lrang}[1]{\left\langle#1\right\rangle}

\newcommand{\nn}{\nonumber}

\newcommand{\defeq}{\triangleq}

\newcommand{\tilell}{\widetilde{\ell}}

\newcommand{\convu}{\xrightarrow{{\rm u}}}
\newcommand{\inter}{\mathbf{int}\,}
\newcommand{\bdr}{\mathbf{bd}\,}
\newcommand{\cl}{\mathbf{cl}\,}

\newcommand{\convcl}{\overline{\mathbf{conv}\,}}

\newcommand{\diam}{\mathbf{diam}\,}
\newcommand{\dist}{\mathbf{dist}\,}
\newcommand{\hpartial}{\hat{\partial}}
\newcommand{\scHN}{\mathscr{H\hspace{-.15cm}N\hspace{-.15cm}}}

\newcommand{\barbV}{\overline{\bf V}}

\newcommand{\calA}{\mathcal{A}}
\newcommand{\calB}{\mathcal{B}}
\newcommand{\calC}{\mathcal{C}}
\newcommand{\calD}{\mathcal{D}}

\newcommand{\calF}{\mathcal{F}}
\newcommand{\calG}{\mathcal{G}}
\newcommand{\calH}{\mathcal{H}}
\newcommand{\calI}{\mathcal{I}}

\newcommand{\calK}{\mathcal{K}}
\newcommand{\calL}{\mathcal{L}}

\newcommand{\calN}{\mathcal{N}}

\newcommand{\calP}{\mathcal{P}}
\newcommand{\calQ}{\mathcal{Q}}

\newcommand{\calS}{\mathcal{S}}
\newcommand{\calT}{\mathcal{T}}
\newcommand{\calU}{\mathcal{U}}
\newcommand{\calV}{\mathcal{V}}
\newcommand{\calW}{\mathcal{W}}
\newcommand{\calX}{\mathcal{X}}
\newcommand{\calY}{\mathcal{Y}}
\newcommand{\calZ}{\mathcal{Z}}

\newcommand{\barcalD}{\overline{\calD}}

\newcommand{\bA}{\mathbf{A}}

\newcommand{\bg}{\mathbf{g}}
\newcommand{\bG}{\mathbf{G}}
\newcommand{\bh}{\mathbf{h}}
\newcommand{\bH}{\mathbf{H}}

\newcommand{\bM}{\mathbf{M}}

\newcommand{\bN}{\mathbf{N}}

\newcommand{\bT}{\mathbf{T}}

\newcommand{\bv}{\mathbf{v}}
\newcommand{\bV}{\mathbf{V}}

\newcommand{\bW}{\mathbf{W}}
\newcommand{\bx}{\mathbf{x}}
\newcommand{\bX}{\mathbf{X}}
\newcommand{\by}{\mathbf{y}}
\newcommand{\bY}{\mathbf{Y}}

\newcommand{\bZ}{\mathbf{Z}}

\newcommand{\rmd}{\mathrm{d}}

\newcommand{\bbE}{\mathbb{E}}

\newcommand{\bbN}{\mathbb{N}}

\newcommand{\bbP}{\mathbb{P}}

\newcommand{\bbR}{\mathbb{R}}

\newcommand{\scA}{\mathscr{A}}
\newcommand{\scB}{\mathscr{B}}

\newcommand{\scF}{\mathscr{F}}
\newcommand{\scG}{\mathscr{G}}

\newcommand{\scN}{\mathscr{N}}

\newcommand{\scP}{\mathscr{P}}

\newcommand{\scU}{\mathscr{U}}

\DeclareMathAlphabet{\mathbsf}{OT1}{cmss}{bx}{n}
\DeclareMathAlphabet{\mathssf}{OT1}{cmss}{m}{sl}

\DeclareSymbolFont{bsfletters}{OT1}{cmss}{bx}{n}  
\DeclareSymbolFont{ssfletters}{OT1}{cmss}{m}{n}
\DeclareMathSymbol{\bsfGamma}{0}{bsfletters}{'000}
\DeclareMathSymbol{\ssfGamma}{0}{ssfletters}{'000}
\DeclareMathSymbol{\bsfDelta}{0}{bsfletters}{'001}
\DeclareMathSymbol{\ssfDelta}{0}{ssfletters}{'001}
\DeclareMathSymbol{\bsfTheta}{0}{bsfletters}{'002}
\DeclareMathSymbol{\ssfTheta}{0}{ssfletters}{'002}
\DeclareMathSymbol{\bsfLambda}{0}{bsfletters}{'003}
\DeclareMathSymbol{\ssfLambda}{0}{ssfletters}{'003}
\DeclareMathSymbol{\bsfXi}{0}{bsfletters}{'004}
\DeclareMathSymbol{\ssfXi}{0}{ssfletters}{'004}
\DeclareMathSymbol{\bsfPi}{0}{bsfletters}{'005}
\DeclareMathSymbol{\ssfPi}{0}{ssfletters}{'005}
\DeclareMathSymbol{\bsfSigma}{0}{bsfletters}{'006}
\DeclareMathSymbol{\ssfSigma}{0}{ssfletters}{'006}
\DeclareMathSymbol{\bsfUpsilon}{0}{bsfletters}{'007}
\DeclareMathSymbol{\ssfUpsilon}{0}{ssfletters}{'007}
\DeclareMathSymbol{\bsfPhi}{0}{bsfletters}{'010}
\DeclareMathSymbol{\ssfPhi}{0}{ssfletters}{'010}
\DeclareMathSymbol{\bsfPsi}{0}{bsfletters}{'011}
\DeclareMathSymbol{\ssfPsi}{0}{ssfletters}{'011}
\DeclareMathSymbol{\bsfOmega}{0}{bsfletters}{'012}
\DeclareMathSymbol{\ssfOmega}{0}{ssfletters}{'012}

\newcommand{\tild}{\widetilde{d}}

\newcommand{\hatf}{\widehat{f}}

\newcommand{\hatbv}{\widehat{\bv}}

\newcommand{\tilW}{\widetilde{W}}

\newcommand{\hatbW}{\widehat{\bW}}

\newcommand{\tilZ}{\widetilde{Z}}

\newcommand{\bard}{\overline{d}}

\newcommand{\bart}{\overline{t}}

\newcommand{\barv}{\overline{v}}

\newcommand{\barx}{\overline{x}}

\newcommand{\barG}{\overline{G}}

\newcommand{\barW}{\overline{W}}

\newcommand{\barZ}{\overline{Z}}

\newcommand{\bUpsilon}{\bm{\Upsilon}}

\newcommand{\iid}{i.i.d.\ }

\newcommand{\convas}{\xrightarrow{\mathrm{a.s.}}}

\newcommand{\lrangle}[2]{\left\langle{#1},{#2}\right\rangle}

\newcommand{\eqa}{\stackrel{(a)}{=}}
\newcommand{\eqb}{\stackrel{(b)}{=}}

\newcommand{\leb}{\stackrel{(b)}{\le}}

\DeclareMathOperator*{\argmax}{arg\,max}
\DeclareMathOperator*{\argmin}{arg\,min}

\DeclareMathOperator{\st}{s.t.\;}

\newtheorem{theorem}{Theorem} 
\newtheorem*{theorem*}{Theorem}
\newtheorem{lemma}{Lemma}

\newtheorem{corollary}{Corollary}
\newtheorem*{assump}{Assumptions}

\theoremstyle{definition}
\newtheorem{definition}{Definition} 

\theoremstyle{remark}
\newtheorem{remark}{Remark}

\newenvironment{psketch}{\noindent{\em Proof Sketch.}\hspace*{1em}}{\qed\bigskip\\}

\newcommand{\qednew}{\nobreak \ifvmode \relax \else
      \ifdim\lastskip<1.5em \hskip-\lastskip
      \hskip1.5em plus0em minus0.5em \fi \nobreak
      \vrule height0.75em width0.5em depth0.25em\fi}

\graphicspath{{./figures/}}
\newcommand{\ol}{{\sf OL} }
\newcommand{\batch}{{\sf Batch} }
\begin{document}
\title{\vspace{-.5cm} Online Nonnegative Matrix Factorization with General Divergences}

\author{
Renbo~Zhao,~\IEEEmembership{Member,~IEEE,}
        ~Vincent~Y.~F.~Tan,~\IEEEmembership{Senior~Member,~IEEE,}
        ~Huan~Xu   \vspace{-1cm}     

\thanks{R.~Zhao and V.~Y.~F.~Tan are with the Department of Electrical and Computer Engineering and the Department of Mathematics, National University of Singapore (NUS). They are supported in part by an NUS Young Investigator Award (R-263-000-B37-133). H.~Xu is with the Department of Industrial and Systems Engineering, NUS. He is supported by an A*STAR SERC PSF Grant (R-266-000-101-305).}
}

\markboth{IEEE Transactions on Signal Processing,~Vol.~, No.~, 2016}%
{Shell \MakeLowercase{\textit{Zhao et al.}}: Online Nonnegative Matrix Factorization with General Divergences}

\maketitle

\begin{abstract}
We develop a unified and systematic framework for performing online nonnegative matrix factorization under a wide variety of important divergences. The online nature of our algorithm makes it particularly amenable to large-scale data.  We prove that the sequence of learned dictionaries converges almost surely to the set of critical points of the expected loss function. We do so by leveraging the theory of stochastic approximations and projected dynamical systems. This result substantially generalizes the previous results obtained only for the squared-$\ell_2$ loss.
Moreover, the novel techniques involved in our analysis open new avenues for analyzing similar matrix factorization problems.
The computational efficiency and the quality of the learned dictionary of our algorithm are verified empirically on both synthetic and real datasets. 
In particular, on the tasks of topic learning, shadow removal and image denoising, our algorithm achieves superior trade-offs between the quality of learned dictionary and running time over the batch and other online NMF algorithms. 
\end{abstract}


\vspace{-.3cm}
\section{Introduction}
Over the past 20 years, Nonnegative Matrix Factorization (NMF) has been a popular dimensionality reduction and data analysis technique. This is, in part, due to its non-subtractive and parts-based interpretation on the learned basis \cite{Lee_99}. Given a nonnegative matrix $\bV$ with dimension ${F\times N}$, one seeks a nonnegative dictionary matrix $\bW$ 
and a nonnegative coefficient matrix $\bH$ 
such that $\bV\approx\bW\bH$, by solving 
\begin{align}
&\min \left[D(\bV\Vert\bV^o)\defeq\sum_{n=1}^N d(\bv_n\Vert\bv^o_n)\right]\label{eq:batch_NMF}\\
&\st  \bV^o = \bW\bH, \bW\ge 0, \bH\ge 0,\nn
\end{align}
where $\bv_n$ (resp.\ $\bv^o_n$) denotes the $n$-th column of $\bV$ (resp.\ $\bV^o$) and $d(\cdot\Vert\cdot)$ denotes a divergence between two nonnegative vectors. In the NMF literature~\cite{Dhillon_06,Cichoc_08,Cichoc_11,Fev_11,Li_12,Ding_11,Wang_13,Shen_14,Gao_15}, in addition to the squared-$\ell_2$ loss, i.e., $d(\bx\Vert\by)=\norm{\bx-\by}^2_2$, a vast variety of divergences have been proposed for two main purposes. Firstly, from the statistical inference viewpoint, given the observation noise of a particular distribution, there exists a divergence such that solving \eqref{eq:batch_NMF} correspond to the maximum-likelihood (ML) estimation of ground-truth $\bV^o$ under observation $\bV$. For example, if the distribution of the observation noise belongs to the exponential family, then the corresponding divergence belongs to the class of Bregman divergences \cite{Banerjee_05}. Moreover, it has been shown empirically that if the divergence used in \eqref{eq:batch_NMF} does not match the distribution of the noise, the results will be inferior~\cite{Fev_09}; thus it is imperative to use the correct divergence. Therefore, various divergences have been proposed to optimize the empirical performances of NMF on numerous applications, including music analysis \cite{Fev_09}, source separation \cite{Durrieu_11}, topic modeling \cite{Lee_99}, hyperspectral unmixing \cite{Yuan_15} and gene expression analysis \cite{Tasla_12}. 
Secondly, many robust divergences (or more specifically, metrics) have been proposed in order to overcome the sensitivity of the squared-$\ell_2$ loss to outliers in the data matrix $\bV$. 
These metrics include the $\ell_1$ loss\cite{Ding_11}, the $\ell_2$ loss\cite{Ding_11}, the Huber loss\cite{Wang_13} and the capped $\ell_2$ loss\cite{Gao_15}. With the robust metrics , NMF has been successfully applied to image inpainting \cite{Ding_11}, visual tracking \cite{Wang_13} and outlier image detection \cite{Gao_15}. Despite the success of the NMF algorithms with the aforementioned {\em general divergences}, the {\em batch} data processing mode intrinsic to the algorithms prohibits them from being applied to {\em large-scale} data, i.e., finite data collections with a large number of data samples or even streaming data.\footnote{In this work, we do not simultaneously consider the data with high ambient dimensions. An attempt on this problem in the context of dictionary learning with the squared-$\ell_2$ loss has been made in \cite{Mensch_16}.}
The reasons are twofold---one being that the storage space might be insufficient to store the entire set of   samples, and the other being that the computational complexity of each iteration increases at least linearly with the number of   samples $N$. 
However, due to the significant advances in data acquisition, 
large-scale data is increasingly ubiquitous today, e.g., 
the Netflix movie-rating database and Google images. Therefore, it is important to devise   scalable NMF algorithms under general divergences that is able to 
handle large-scale data efficiently, both in terms of computation and storage.
\vspace{-.2cm}
\subsection{Prior Works}
Many variants of   batch NMF algorithms have been proposed in previous works to tackle large-scale data. They can be classified into three types, namely {\em online NMF}  \cite{Guan_12,Lef_11}, {\em distributed NMF} \cite{Gemulla_11,Chen_15} and {\em compressed NMF} \cite{Halko_11,Tepper_16}. See \cite[Section~I-A]{Zhao_16b} for an overview of the three types of algorithms. In particular, the online NMF algorithms aim to continuously refine the dictionary $\bW$ each time new data samples arrive without storing past data samples, thereby overcoming both computational and storage barriers. 
However, most of the works focus on the case where $d(\cdot\Vert\cdot)$ in \eqref{eq:batch_NMF} is the squared-$\ell_2$ loss, with only a few exceptions, e.g., \cite{Lef_11,Chen_15}. In these works, the Itakura-Saito (IS) divergence  and the Huber loss are also considered. 
Due to the wide applicability of   other divergences, it is natural to devise a {\em unified} framework 
that   performs NMF under a wide range of divergences systematically for large-scale data. The proposed algorithm should also enjoy  some form of convergence    guarantees. 
However, no such theoretical framework has yet been developed thus far.

\vspace{-.2cm}
\subsection{Challenges and Main Contributions}
In this paper, we develop  a framework termed {\em online NMF with general divergences} that learns the dictionary $\bW$ in \eqref{eq:batch_NMF} in an online manner under a variety of divergences, 
including    Csisz\'{a}r $f$-divergences,  Bregman divergences and various robust metrics. In  online NMF, an effective framework called {\em stochastic Majorization-Minimization (MM)} \cite{Mairal_10,Mairal_13,Raza_16} has been proposed to gracefully tackle online matrix factorization problems (including NMF) with the squared-$\ell_2$ loss. However, this method, based on {\em sample average approximation} (SAA), cannot be applied to most of the general divergences we consider, since crucially, sufficient statistics in the method cannot be formed. (See Section~\ref{sec:discuss_algo} for details.) Therefore, we leverage the {\em stochastic approximation} (SA) framework \cite{Ben_90,Kushner_03,Borkar_08} to develop an algorithm that does not need to compute the sufficient statistics, so it can effectively handle general divergences while being amenable to convergence analysis. 
Our analysis shows that the sequence of learned dictionaries converges almost surely to the set of critical points of the expected loss function \eqref{eq:exp_loss}. This serves as a substantial generalization of the results in \cite{Mairal_10,Raza_16}. However, note that since stochastic MM and our proposed algorithm are based on two different frameworks (SAA and SA), the convergence analyses of the two methods are vastly different. Despite the differences, our analysis can indeed acts as an alternative when the divergence specializes to the squared-$\ell_2$ loss. 
To illustrate the computational  and storage  efficiency and the quality of the learned dictionary of our algorithm, we conduct numerical experiments on both  synthetic and real datasets. 
Our     results demonstrate that on the tasks of topic learning, shadow removal and image denoising, our algorithm achieves superior trade-offs between the quality of learned dictionary and running time over the batch and other  online NMF algorithms. 

In sum, our main contributions are twofold:
\begin{enumerate}
\item We propose a computationally efficient framework for online NMF with general divergences. This framework systematically generalizes the prior works on online NMF with the squared-$\ell_2$ loss. In particular, our framework provides an approach for performing {\em online robust NMF} using robust metrics. This complements the previous approach based on the $\ell_1$ regularization on outliers\cite{Zhao_16b}. 

\item We perform convergence analysis of the proposed framework based on the theory on optimal-value functions~\cite{Shimizu_97}, stochastic approximations, projected dynamical systems \cite{Dupuis_93}, Lyapunov stability criterion \cite{Haddad_08,Teschl_12} and variational analysis \cite{Rock_98}. The result of our analysis 
substantially generalizes previous known results~\cite{Mairal_10,Guan_12,Shen_14b}. Moreover, our analysis opens new avenues for analyzing similar online matrix factorization problems. Note that due to the nonconvex nature of the NMF problem, global convergence results are few and far between, in both batch and online settings (especially for divergences beyond the squared-$\ell_2$ loss). The present work establishes a global almost sure convergence result for the sequence of learned dictionaries in the online setting.\footnote{In the batch (deterministic) setting \eqref{eq:batch_NMF}, global convergence results pertain to the convergence of the sequence of iterates $\{(\bW_k,\bH_k)\}_{k=1}^\infty$. In the online (stochastic) setting, the corresponding results pertain to the (almost sure) convergence of the discrete-time stochastic process $\{\bW_t\}_{t=1}^\infty$, since the coefficient vectors are not stored (hence their convergence is not of interest).}  
\end{enumerate}

\section{Notations and Overview of Divergences}
\subsection{Notations}
We use boldface capital letters, boldface lowercase letters and plain lowercase letters to denote matrices, column vectors and scalars respectively. In particular, we use $\vecz$ to denote a zero vector or a zero matrix. 
Given a matrix $\bX\in\bbR^{F\times K}$, we denote its $i$-th row as $\bX_{i:}$, $j$-th column as $\bX_{:j}$ and $(i,j)$-th entry by $x_{ij}$.
For a column vector $\bx\in\bbR^F$, its $i$-th entry is denoted by $x_i$ or $(\bx)_i$.  
We use $\norm{\cdot}$ to denote both the Euclidean norm of a vector and the Frobenius norm of a matrix. We denote the $\ell_1$ norm of $\bx$ as $\norm{\bx}_1$ and the spectral norm of $\bX$ as $\norm{\bX}_2$. We use $\lrangle{\cdot}{\cdot}$ to denote the Euclidean inner product between vectors or matrices. 
We denote the (Euclidean) projection operator onto a set $\calS$ as $\Pi_\calS$. 
We denote the set of nonnegative real numbers, the set of positive real numbers and the set of natural numbers (excluding zero) as $\bbR_+$, $\bbR_{++}$ and  $\bbN$ respectively. For $N\in\bbN$, define $[N]\defeq\{1,2,\ldots,N\}$.
Moreover, given a sequence of functions $\{f_n\}_{n\in\bbN}$ and a function $f$, $f_n\convu f$ 
denotes the uniform convergence of $f_n$ to $f$. 
In the context of this work, we use $\bv$, $\bW$ and $\bh$ to denote the (nonnegative) data vector, dictionary matrix and coefficient vector respectively. 
The ambient and latent data dimensions are denoted as $F$ and $K$ respectively, and are assumed to be {\em independent of time}. 

Due to space constraints, some proofs and additional figures are relegated to the supplemental material. All the equations, lemmas, definitions and sections with indices beginning with an `S' will appear in the supplemental material.
\subsection{Overview of Divergences} \label{sec:intro_div}
In this section, we briefly overview the basic definitions and properties of various divergences employed in the literature of NMF. We classify these divergences into three categories: the Csisz\'{a}r $f$-divergence, the Bregman divergence and the robust metrics. 
For further details, see \cite[Chapter~2]{Cichoc_09}. 
\subsubsection{Csisz\'{a}r $f$-divergence}
The Csisz\'{a}r $f$-divergence with generating function $\varphi:\bbR_{++}\to\bbR$ between two vectors $\bx,\by\in\bbR_{++}^n$, 
$d^{\rm(c)}_\varphi(\bx\Vert\by) \defeq \sum_{i\in[n]} y_i \varphi\left(x_i/y_i\right)$,
where $\varphi$ is convex on $\bbR_{++}$ such that $\varphi(1)=0$. Several important instances of Csisz\'{a}r $f$-divergence have been used in NMF, including the $\ell_1$ distance \cite{Guan_12b} and the family of $\alpha$-divergences \cite{Cichoc_08,Cichoc_11}. In particular, the $\alpha$-divergences include Hellinger distance ($\alpha=1/2$) \cite{Hanlon_15,Hanlon_16} and the (generalized) Kullback-Leibler (KL) divergence ($\alpha\to 1$) \cite{Lee_00,Yang_11}.  Moreover, for these special cases, $d^{\rm (c)}_\varphi(\bx\Vert\cdot)$ is convex on $\bbR_{++}^n$. The expressions of these divergences are summarized in Table~\ref{table:csiszar}. 
\subsubsection{Bregman divergence}
The Bregman divergence with generating function $\phi:\bbR_{++}^n\to\bbR$ between two vectors $\bx,\by\in\bbR_{++}^n$, 
$d^{\rm (b)}_\phi(\bx\Vert\by) \defeq \phi(\bx) - \phi(\by) - \lrangle{\nabla\phi(\by)}{\bx-\by}$,
where $\phi$ is strictly convex and continuously differentiable on $\bbR_{++}^n$. In the literature on NMF, several important instances of Bregman divergence have been used, including Mahalanobis distance \cite{Paa_94} and the family of $\beta$-divergence \cite{Fev_11,Nakano_10}. In particular, the $\beta$-divergence includes 
IS divergence ($\beta\to 0$) \cite{Fev_09}, the (generalized) KL divergence ($\beta\to 1$) \cite{Lee_00,Yang_11} and the squared-$\ell_2$ loss ($\beta=2$) \cite{Lee_00,Guan_12c}. Moreover, $d^{\rm (b)}_\phi(\bx\Vert\cdot)$ is convex on $\bbR_{++}^n$ for $1\le \beta\le 2$ and analytic on $\bbR_{++}^n$ for $\beta\in\bbR$. The expressions of the important divergences are summarized in Table~\ref{table:bregman}. 
\subsubsection{Robust metrics}
In the literature of robust NMF (and dictionary learning), various robust metrics, denoted as $d^{\rm (r)}(\cdot\Vert\cdot):\bbR_+^n\times\bbR_+^n\to\bbR$, have been proposed to tackle potential outliers in the data matrices. Important cases of these metrics include the $\ell_1$ distance \cite{Kasi_12,Pan_14}, the $\ell_2$ distance \cite{Ding_11}, 
and the Huber loss \cite{Wang_13,Zhang_15b,Chen_15}. The expressions of the important robust metrics are summarized in Table~\ref{table:robust_metrics}. 

In the sequel, we aim to treat the Csisz\'{a}r $f$-divergence, the Bregman divergence and the robust metrics in a unified way. 
To simplify discussions, we focus on the important cases shown in Table~\ref{table:csiszar}, \ref{table:bregman} and \ref{table:robust_metrics} as these divergences share some regularity properties (see Remark~\ref{rmk:def_div}). 
We denote the set of these divergences as $\calD$ with common domain $\bbR_{++}^n\times\bbR_{++}^n$, and use $d(\cdot\Vert\cdot)$ to denote any divergence in $\calD$. 
\begin{remark} \label{rmk:def_div}
We first remark on the regularity properties of the divergence class $\calD$. 
We notice that $\calD$ is the union of two divergence classes, $\calD_1$ and $\calD_2$ defined respectively as
\begin{align*}
\calD_1 &= \{d(\cdot\Vert\cdot)\in\calD\,|\,\forall\,\bx\in\bbR_{++}^n,\forall\,\mbox{compact }\calY\subseteq\bbR_{++}^n,\\
&d(\bx\Vert\cdot)\mbox{ is differentiable and } \nabla d(\bx\Vert\cdot) \mbox{ is Lipschitz on } \calY \},\\
\calD_2 &= \{d(\cdot\Vert\cdot)\in\calD\,|\,\forall\,\bx\in\bbR_{++}^n,d(\bx\Vert\cdot) \mbox{ is convex on } \bbR^n_{++}\}.
\end{align*}
The class $\calD_1$ includes 
all the divergences in $\calD$ except the $\ell_1$ and $\ell_2$ distances, whereas 
the class $\calD_2$ includes all the divergences in $\calD$ except the sub-family of the $\beta$-divergence with $\beta\in(-\infty,1)\cup(2,\infty)$. 
Second, the domain of $d(\cdot\Vert\cdot)$ can be relaxed to $\bbR_+^n\times\bbR_+^n$ for some specific cases, e.g., the robust metrics.
\end{remark}


\begin{table}\footnotesize
\caption{Expressions of some important cases of Csisz\'{a}r $f$-divergence} 
\label{table:csiszar}
\centering
\hspace*{-.2cm}\setlength\tabcolsep{1.5pt}
\begin{tabular}{|c|c|}\hline
Name & $d^{\rm (c)}_\varphi(\bx\Vert\by)$ \\\hline
$\ell_1$ distance  & $\sum_i \abs{x_i-y_i}$\\\hline
$\alpha$-divergence $\left(\alpha\in \bbR\setminus\{0,1\}\right)$ & $\frac{1}{\alpha(\alpha-1)}\sum_i\left(y_i\left[\left(\frac{x_i}{y_i}\right)^\alpha \!-\! 1\right]-\alpha(x_i\!-\!  y_i)\right)$\\\hline
Hellinger distance ($\alpha=\frac{1}{2}$) & $2\sum_i(\sqrt{x_i}-\sqrt{y_i})^2$\\\hline
 KL divergence ($\alpha\to 1$) & $\sum_i x_i\log(x_i/y_i)-x_i+y_i$\\\hline
\end{tabular}\vspace{-.3cm}
\end{table}


\begin{table}\footnotesize
\caption{Expressions of some important cases of Bregman divergence}
\label{table:bregman}
\centering
\setlength\tabcolsep{2pt}
\begin{tabular}{|c|c|}\hline
Name &   $d_\phi^{\rm(b)}(\bx\Vert\by)$ \\\hline
Mahalanobis distance  & $(\bx-\by)^T\bA(\bx-\by)/2$\\\hline
$\beta$-divergence $\left(\beta\in \bbR\setminus\{0,1\}\right)$ &  $\frac{1}{\beta(\beta-1)}\sum_i \left(x_i^{\beta}-y_i^{\beta}- \beta y_i^{\beta-1}(x_i-y_i)\right)$\\\hline
IS divergence ($\beta\to 0$)  & $\sum_i \left(\log(y_i/x_i) + x_i/y_i -1\right)$  \\\hline
KL divergence ($\beta\to 1$)  &  $\sum_i x_i\log(x_i/y_i)-x_i+y_i$ \\\hline
Squared $\ell_2$ distance ($\beta=2$)  &$\norm{\bx-\by}_2^2/2$\\\hline
\end{tabular}\vspace{-.4cm}
\end{table}

\begin{table}[t!]\footnotesize
\centering
\begin{threeparttable}
\caption{Expressions of some important robust metrics} \label{table:robust_metrics}
\begin{tabular} {|>{\centering\arraybackslash}m{3cm}|>{\centering\arraybackslash}m{4cm}|}\hline
Name & $d^{\rm (r)}(\bx\Vert\by)$ \\\hline
$\ell_1$ distance & $\norm{\bx-\by}_1$\\\hline
$\ell_2$ distance & $\norm{\bx-\by}$\\\hline
Huber loss\footnotemark & $\sum_i h_\alpha(x_i-y_i)$  
\\\hline
\end{tabular}
\begin{tablenotes}
\item[1] The (scalar) Huber loss function, $h_\alpha:\bbR\to\bbR$ is defined as $
h_\alpha(u)=\left\{\hspace{-.2cm}\begin{array}{ll}
u^2/2,&\abs{u}\le \alpha\\
\alpha\left(\abs{u}-\alpha/2\right), &\text{otherwise} 
\end{array},\right.
$ where $\alpha>0$. 
\end{tablenotes}
\end{threeparttable}\vspace{-.5cm}
\end{table}

\vspace{-.35cm}
\section{Related Works}
\subsection{Online Matrix Factorization Beyond the Squared-$\ell_2$ Loss} \label{sec:OMF_beyond_l2}
In the literature of online matrix factorization \cite{Mairal_10,Guan_12,Feng_13,Shen_14,Shen_16}, it is  assumed that \iid (independent and identically distributed)  data samples $\{\bv_t\}_{t\in\bbN}$ (drawn from a common distribution $\bbP$) arrive in a streaming manner, and the storage space does not scale with time. 
Under such setting, we aim to solve the following {\em stochastic program}, i.e., minimize the expected loss 
\begin{equation}
\min_{\bW\in\calC} \bbE_{\bv\sim\bbP}[\tilell(\bv,\bW)],
\end{equation}
where $\bv$ is the (random) data vector with distribution $\bbP$, $\bW$ the basis matrix constrained in the set $\calC$ and $\tilell$ the loss function with respect to (w.r.t.) a single data sample $\bv$. 
Most of the literature on online matrix factorization (including online NMF \cite{Guan_12}, online dictionary learning \cite{Mairal_10} and online low-rank representation \cite{Shen_16}) focus on the case where the data fidelity term is the squared $\ell_2$ loss, i.e., $\tilell$ is defined as 
\begin{equation}
\tilell(\bv,\bW)  \defeq \min_{\bh\in\calH}\norm{\bv-\bW\bh}^2 + \lambda\psi(\bh),
\end{equation}
where $\bh$ is the coefficient vector constrained in the set $\calH$ and $\lambda\psi(\bh)$ is some regularizer on $\bh$ with penalty parameter $\lambda>0$. However, the literature with other forms of data fidelity terms is relative scarce. Among them, some works on real-time music signal processing \cite{Lef_11,Dessein_11} consider minimizing the IS divergence in an online manner. Other works on visual tracking \cite{Wang_13,Zhang_15b} consider the online minimization of the Huber loss. However, almost all of methods proposed in these works are {\em heuristic} in nature, in the sense that the global convergence of the sequence (or any subsequence) of the dictionaries $\{\bW_t\}_{t\in\bbN}$ cannot be guaranteed (either a.s.\ or with high probability). Furthermore, since most of these works are conducted on an {\em ad hoc} basis, the approaches {therein cannot be easily generalized to other divergences in a straightforward manner. As different divergences are suited to different applications in practice (see Section~\ref{sec:intro_div}), a unified framework 
is needed to systematically study the convergence properties of NMF for various divergences.}

\subsection{Stochastic Projected Subgradient Descent (SPSGD) Applied to Online Matrix Factorization}
As discussed in Section~\ref{sec:OMF_beyond_l2}, only {\em differentiable} data fidelity terms (the IS divergence and the Huber loss) are considered in the literature of online matrix factorization. Thus, only the stochastic projected gradient descent (SPGD) method  has been employed   in the prior works \cite{Chen_15,Mairal_10,Wang_13,Zhang_15b,Guan_12}.
In particular, the efficacy of such method with the squared-$\ell_2$ loss and the Huber loss has been empirically verified in \cite{Mairal_10} and \cite{Wang_13,Zhang_15b}  respectively. 
In \cite{Chen_15}, SPGD was employed on online dictionary learning over distributed models, with both squared-$\ell_2$ loss and the Huber loss. In \cite{Guan_12}, the authors leverage the {robust stochastic approximation} method \cite{Nemi_09}, a variant of SPGD, and consider both  the squared $\ell_2$ loss and the IS divergence. However, for all the abovementioned works, convergence guarantees on the sequence (or any subsequence) of the dictionaries $\{\bW_t\}_{t\in\bbN}$ generated by the SPGD algorithm have not been established. 

\section{Problem Formulation}
As introduced in Section~\ref{sec:OMF_beyond_l2}, consistent with the problem formulation in the literature, we consider the problem of learning the (nonnegative) dictionary $\bW$ {in a  streaming data setting}. Specifically, we assume the data stream $\{\bv_t\}_{t\in\bbN}\subseteq\bbR_+^F$ is generated \iid from a distribution $\bbP$.\footnote{Most of real data do not strictly satisfy the independence assumption, since they may be  {weakly dependent}. However, we {make the \iid assumption here for convenience of analysis}.} We also assume limited storage space, i.e., the memory size does not scale with time. {Under such a setting}, we aim to minimize the expected loss 
\begin{equation}
\min_{\bW\in\calC} \left[f(\bW)\defeq \bbE_{\bv\sim\bbP}[\ell(\bv,\bW)]\right], \label{eq:exp_loss}
\end{equation}
where $\calC \defeq \{\bW\in\bbR_+^{F\times K}\,|\,\norm{\bW_{i:}}_1\ge\epsilon, \norm{\bW_{:j}}_\infty\le 1, \forall\,(i,j)\in[F]\times [K]\}$ for some $0<\epsilon<1$ and 
\begin{equation}
\ell(\bv,\bW) \defeq \min_{\bh\in\calH} d(\bv\Vert\bW\bh). \label{eq:single_loss}
\end{equation}
Here 
$\calH \defeq \{\bh\in\bbR_+^K\,|\, \epsilon'\le h_i \le U, \forall\,i\in[K]\}$ 
for some positive constants $\epsilon'$ and $U$. 
In words, $\ell(\bv,\bW)$ is the loss function of $\bW$ w.r.t.\ a single random data sample $\bv\sim\bbP$. 
As introduced in Section~\ref{sec:intro_div}, $d(\cdot\Vert\cdot)$ can be any divergence in  $\calD$. 
\begin{remark} \label{rmk:prob_form}
Several remarks are in order. First we notice in general, $d(\cdot\Vert\cdot)$ is asymmetric about its arguments. In this work, we only consider minimizing $d(\bv\Vert\bW\bh)$ in $\bh$ (and $\bW$) since this corresponds to an ML estimation of the ground-truth data from the noisy data sample $\bv$ under various statistical models. 
For simplicity we omit regularizations on $\bW$ in \eqref{eq:exp_loss} and on $\bh$ in \eqref{eq:single_loss}. See \cite[Section~5]{Mairal_10} for possible regularizations.
Next, we explain the rationale behind the choice of the constraint sets $\calC$ and $\calH$. The constraints $\norm{\bW_{:j}}_\infty\le 1$, for all $j\in[K]$ and $\norm{\bh}_\infty\le U$ bound the scale of $\bW$ and $\bh$ respectively. Similar constraints are common in previous works \cite{Mairal_10,Guan_12,Bao_15}. Based on real applications, $U<\infty$ can be set to suitably large values. 
Furthermore, since the domain of $d(\cdot\Vert\cdot)$ is $\bbR_{++}^F\times\bbR_{++}^F$, for the sake of numerical stability, we require $\left(\bW\bh\right)_i\ge\epsilon''$, for all $i\in[F]$ and some small number $\epsilon''\in(0,1)$ within the precision tolerance of the numerical software. Thus, the constraints $\norm{\bW_{i:}}_1\ge\epsilon$, for all $i\in[F]$ and $h_i\ge\epsilon'$, for all $i\in[K]$ decouple the above constraint on $(\bW,\bh)$
with $\epsilon\epsilon'\ge \epsilon''$. 
For the reasons above, in this work we set $\epsilon=\epsilon'=1\times 10^{-8}$ and $U=1\times 10^{8}$. 
These two constraints can be removed if the domain of $d(\cdot\Vert\cdot)$ can be relaxed to $\bbR_+^F\times\bbR_+^F$ for some specific divergences. 
Furthermore, the constructions of $\calC$ and $\calH$ also enable efficient projections onto both sets. See Section~\ref{sec:proj_CH} for details. 
Finally, since $d(\bv\Vert\cdot)$ is in general not well-defined if $\bv$ has zero entries, we can simply set these zero entries to small positive numbers so that $\bv\in\bbR_{++}^F$. 
\end{remark}

\vspace{-.5cm}
\section{Algorithm}
The outline of the algorithm for online NMF with general divergences is shown in Algorithm~\ref{algo:ONMFG}. In the following we first define some important concepts and functions in Algorithm~\ref{algo:ONMFG}. Then we explain the choice of input arguments in Algorithm~\ref{algo:ONMFG}. Next we illustrate the details of Algorithm~\ref{algo:solve_h}, the algorithm for learning the coefficient vector $\bh_t$. Finally we discuss the rationale for using SPSGD method and compare it with other possible methodologies. 

\vspace{-.25cm}
\subsection{Definitions}
In this section, we let $\calX$ be a finite-dimensional real Banach space, e.g., $\bbR^F$ or $\bbR^{F\times N}$. We denote the topological dual space of $\calX$ as $\calX^*$. We also consider a function $f:\calX\to\bbR$. 
\begin{definition}[{\cite[Section~1.1]{Kruger_03}}]
The  {\em Fr\'{e}chet subdifferential} at $x\in\calX$, $\hpartial f(x)$ is defined as
\begin{equation*}
\hpartial f(x) \defeq \left\{g\in\calX^*\,\Big\vert\,\liminf_{y\to x,y\in\calX} \frac{f(y)-f(x)-\lrangle{g}{y-x}}{\norm{y-x}}\ge 0\right\}, \label{eq:def_frechet}
\end{equation*}
where 
$\lrangle{\cdot}{\cdot}$ defines the canonical pairing between $\calX$ and $\calX^*$. 
\end{definition} 
\begin{remark}\label{rmk:property_frechet}
If $f$ is differentiable at $x\in\calX$, then $\hpartial f(x) = \{\nabla f(x)\}$. The Fr\'{e}chet subdifferential serves as a generalization of the subdifferential in the convex analysis, i.e., if $f$ is convex on $\calX$, then for any $x\in\calX$, $\hpartial f(x) = \partial f(x)$, where 
\begin{equation*}
\partial f(x) \defeq \left\{g\in\calX^*\,\Big\vert\,f(x)+\lrangle{g}{y-x}\le f(y), \forall\,y\in\calX\right\}.
\end{equation*}
\end{remark}
\begin{definition}[{\cite{Shap_90}}]\label{def:direc_deriv}
The  {\em G\^{a}teaux directional derivative} of $f$ at $x\in\calX$ along direction $d\in\calX$, $f'(x;d)$ is defined as
\begin{equation}
f'(x;d) \defeq \lim_{\delta\downarrow 0} \frac{f(x+\delta d)-f(x)}{\delta}.
\end{equation} 
Furthermore, $f$ is called (G\^{a}teaux) directionally differentiable if $f'(x;d)$ exists for any $x\in\calX$  and any $d\in\calX$. 
\end{definition}
\begin{definition}[{\cite[Section~2]{Raza_13}}]
Assume $f$ to be directionally differentiable.
Let $\calK\subseteq\calX$ be a convex set. 
A point $x^*\in\calK$ is a {\em critical point} of the constrained optimization problem $\min_{x\in\calK} f(x)$ if 
for any $x\in\calK$, 
\begin{equation}
f'(x^*;x-x^*)\ge 0. \label{eq:cond_stat_pt}
\end{equation}
\end{definition}
\begin{remark}
If $f$ is differentiable on $\calX$, then \eqref{eq:cond_stat_pt} is equivalent to the variational inequality $\lrangle{\nabla f(x^*)}{x-x^*}\ge 0$. Now assume $\calK$ to be open. If $f$ is convex, then \eqref{eq:cond_stat_pt} degenerates to $0\in\partial f(x^*)$. If $f$ is differentiable, then \eqref{eq:cond_stat_pt} becomes $\nabla f(x^*)=0$.
\end{remark}

In addition to the concepts, for any $t\in\bbN$, we also define two important functions $\tild_t:\calC\to\bbR$ and $\bard_t:\calH\to\bbR$ as
\begin{align*}
\tild_t(\bW)&\defeq d(\bv_t\Vert\bW\bh_t), \forall\,\bW\in\calC,\\
\bard_t(\bh)&\defeq d(\bv_t\Vert\bW_{t-1}\bh), \forall\,\bh\in\calH,
\end{align*}
where $\{\bv_t,\bW_t,\bh_t\}_{t\in\bbN}$ are generated per Algorithm~\ref{algo:ONMFG}.
\begin{remark}
From Remark~\ref{rmk:def_div}, we observe that $\tild_t$ or $\bard_t$ is either differentiable or convex (or both). Thus by Remark~\ref{rmk:property_frechet}, finding a Fr\'{e}chet subgradient of $\tild_t$ or $\bard_t$ is straightforward. 
\end{remark}

\vspace{-.5cm}
\subsection{Choice of Input Arguments}
The  {matrix used for initializing our algorithm  $\bW_0$ can be chosen to be any element of $\calC$}. The number of iterations $T$ is usually chosen {to be} the size of the  dataset. 
The step size sequence $\{\eta_t\}_{t\in\bbN}$ is chosen to satisfy
\begin{equation}
\sum_{t=1}^\infty \eta_t = \infty \quad \mbox{and} \quad \sum_{t=1}^\infty \eta_t^2 < \infty. \label{eq:step_size}
\end{equation}
\begin{remark}
In this work we use the classical (diminishing) step size policy \eqref{eq:step_size} as first proposed in \cite{Rob_51}. 
We notice that in the literature \cite{Wang_13,Zhang_15b}, the constant step policy has been used. However, as shown in \cite{Luo_91}, the (projected) SSGD algorithm may diverge even if the objective function $f$ is convex. We also note that for (strongly) convex stochastic programs, many variants of the SSGD algorithms have been proposed, including trajectory averaging \cite{Polyak_92}, gradient averaging \cite{Roux_12,Blatt_07}, robust step size policy \cite{Nemi_09} and second-order method \cite{Bottou_10}. These methods {typically enjoy} faster convergence rates than the original SSGD algorithm. However, since our objective function $f$ is nonconvex, the acceleration of convergence may not be applicable to our {problem setting}.\footnote{The convergence rate in the convex case is w.r.t.\ the Lyapunov criterion $\bbE[f(\bW_t)-f^*]$, where $f^*$ denotes the global minimum of $f$ on $\calC$. However for nonconvex $f$, this criterion is ill-defined since it is generally hard to find $f^*$ within a reasonable amount of time.}
Although for nonconvex stochastic programs, other step size policies have also been proposed \cite{Tseng_98,George_06}, for simplicity of analysis, we use a {policy that satisfies \eqref{eq:step_size}.} 
\end{remark}

\vspace{-.5cm}
\subsection{Learning Coefficient Vectors}
The algorithm for learning $\bh_t$, based on projected subgradient descent (PSGD),  is shown in Algorithm~\ref{algo:solve_h}. 
In this algorithm, the initial coefficient vector $\bh_t^0$ can be chosen {to be} any point in $\calH$. 
For the divergences $d(\cdot\Vert\cdot)\in\calD_1$, the corresponding function $\bard_t$ is differentiable on $\calH$ and $\nabla \bard_t$ is Lipschitz on $\calH$ with Lipschitz constant $L_t>0$.
For {these} cases, there are two ways to choose the step sizes $\{\beta_t^k\}_{k\in\bbN}$ such that the sequence of iterates $\{\bh^k_t\}_{k\in\bbN}$ converges to the set of critical points of \eqref{eq:solve_h} as $k\to\infty$.\footnote{Given a metric space $(\calX,d)$, a sequence $(x_n)$ in $\calX$ is said to converge to a set $\calA\subseteq\calX$ if $\lim_{n\to\infty} \inf_{a\in\calA}d(x_n,a)= 0$.} The first approach is the well-known Armijo rule, which applies to all the continuously differentiable $g_t$ (see \cite[Theorem~2.4]{Calamai_87} for details).  The implementation of Armijo rule is shown in Algorithm~\ref{algo:armijo}, where we set $\alpha=0.01$ and $\gamma= 0.1$ following the suggestions in \cite{Lin_07b}. We also set $q=10$. 
The second approach is to use constant step sizes, i.e.,  $\beta^k_t = \beta_t$, for all $k\in\bbN$. If $\beta_t\in(0,1/L_t]$, then Algorithm~\ref{algo:solve_h} can be interpreted as an MM algorithm \cite{Parikh_14}, and the convergence is guaranteed by \cite[Theorem~1]{Raza_13}. (In this work, we set $\beta_t = 1/L_t$ for simplicity.) We now provide some guidelines for choosing between these two approaches. 
The second approach is suitable for the functions $\nabla \bard_t$ whose smallest Lipschitz constant  on any subset $\calU\subseteq\calH$,\footnote{For any $t\in\bbN$, the smallest Lipschitz constant of $\nabla \bard_t$ on $\calU$, $L^*_t(\calU)\defeq\inf\{L\,|\,\normt{\nabla \bard_t(\bh_1)-\nabla \bard_t(\bh_2)}\le L\norm{\bh_1-\bh_2},\forall\,\bh_1,\bh_2\in\calU\}$.} $L^*_t(\calU)$ does not vary much across all the subsets of $\calH$. Examples of the corresponding divergences include the Huber loss and the squared $\ell_2$ loss. However, the gradients $\nabla \bard_t$ corresponding to some other divergences (e.g., the IS and the KL divergences) 
in general have much larger $L^*_t(\calU)$ when $\calU$ is in the vicinity of $\bdr\calH$ than elsewhere. 
Since $L_t\ge \sup_{\calU\subseteq\calH}L^*_t(\calU)$, the constant step size $\beta_t$ will be very small even when $\bh^k_t$ lies in the ``center'' of $\calH$, where $\bard_t$ is relatively smooth. Under such scenario, it is more appropriate to use Armijo rule  especially when the evaluation of $\bard_t$ is not expensive. 
Now we consider the divergences $d(\cdot\Vert\cdot)\not\in\calD_1$, i.e., the $\ell_1$ and $\ell_2$ losses. For the $\ell_2$ loss, the first approach above is still applicable since $\norm{\cdot}$ is non-differentiable only at $\vecz$. For the $\ell_1$ loss, we employ the {\em modified Polyak's step size policy} with tolerance parameter $\delta_{\rm tol}$ (set to $0.01$ in this work) \cite[Section~6.3.1]{Angelia_08,Bert_99} due to efficiency considerations. Although this step size policy can only guarantee $\liminf_{k\to\infty} \bard_t(\bh_t^k)\le\min_{\bh\in\calH} \bard_t(\bh)+\delta_{\rm tol}$, as shown in Section~\ref{sec:numericals}, 
it performs reasonably well empirically.

\subsection{Discussions}\label{sec:discuss_algo}
In this work we employ the seemingly rudimentary SPSGD method to learn the dictionary in an online manner. In some previous works on online matrix factorization (with squared $\ell_2$ loss) \cite{Mairal_10,Feng_13,Shen_14b}, a different approach has been employed to update the dictionary matrix. Namely, at time $t$, $\bW_t$ is the matrix that minimizes $f_t:\calC\to\bbR$, the majorant\footnote{For a function $g$ with domain $\calG$, its majorant at $\kappa\in\calG$, $G$ is the function that satisfies i) $G\ge g$ on $\calG$ and ii) $G(\kappa) = g(\kappa)$.} for the SAA of $f$, $\hatf_t:\calC\to\bbR$ defined as
\begin{equation}
\hatf_t(\bW) \defeq \frac{1}{t}\sum_{i=1}^t \ell(\bv_i,\bW).
\end{equation}
At {a high} level, this approach belongs to the class of stochastic MM algorithms \cite{Mairal_13,Raza_16}. As noted in \cite[Section~3]{Mairal_13}, direct minimization of $f_t$ is possible only when $f_t$ can be parameterized by variables of small and constant size (known as sufficient statistics in \cite{Mairal_10}) for each $t\in\bbN$. Unfortunately this condition does not hold for most divergences beyond the squared $\ell_2$ loss, including those in class $\calD$. 
However, if we assume for each $\bv$, $\ell(\bv,\cdot)$ has Lipschitz gradient on $\calC$ and choose $f_t$ as a quadratic majorant of $\hatf_t$, then the recursive update form of $\bW_t$ via the stochastic MM approach can be regarded as {\em a special case} of our method. See \cite[Section~4]{Raza_16} for details.


\begin{algorithm}[t]
\caption{Online NMF with General Divergences} \label{algo:ONMFG}
\begin{algorithmic} 
\State {\bf Input}: 
Initial dictionary matrix $\bW_0$,  number of iterations $T$, sequence of step sizes $\{\eta_t\}_{t\in\bbN}$
\State {\bf for} $t$ = 1 to $T$ {\bf do}
\State \quad 1) Draw a data sample $\bv_t$ from $\bbP$.
\State \quad 2) 
Learn the coefficient vector $\bh_t$ such that $\bh_t$ is a critical point  of the optimization problem
\begin{equation}
\min_{\bh\in\calH} \left[\bard_t(\bh)\defeq d(\bv_t\Vert\bW_{t-1}\bh)\right]. \label{eq:solve_h}
\end{equation}
\State \quad 3) Update the dictionary matrix from $\bW_{t-1}$ to $\bW_t$ 
\begin{equation}
\bW_t := \Pi_\calC\Big\{\bW_{t-1} - \eta_t\bG_t\Big\}, \label{eq:upd_dict}
\end{equation}
\quad\quad\; where $\bG_t$ is any element in $\hpartial \tild_t(\bW_{t-1})$.
\State {\bf end for}
\State {\bf Output}: Final dictionary matrix $\bW_T$
\end{algorithmic}
\end{algorithm}

\begin{algorithm}[t]
\caption{Learning $\bh_t$} \label{algo:solve_h}
\begin{algorithmic} 
\State {\bf Input}: Dictionary matrix $\bW_{t-1}$, data sample $\bv_t$, initial coefficient vector $\bh_t^0$, sequence of step size $\{\beta_t^k\}_{k\in\bbN}$
\State {\bf Initialize} $k:=0$
\State {\bf repeat}
\begin{align}
\bh_t^k &:= \Pi_\calH\Big\{\bh^{k-1}_t - \beta^k_t\bg_t^{k}\Big\},\;\mbox{where}\nn\\
&\hspace{2.2cm}\bg_t^{k}\mbox{ is any element in }\hpartial\bard_t(\bh^{k-1}_t)\label{eq:PGD_h}\\
k &:= k+1\nn\end{align}
\State {\bf until} some convergence criterion is met
\State {\bf Output}: Final coefficient vector $\bh_t$
\end{algorithmic}
\end{algorithm}

\begin{algorithm}[t]
\caption{Armijo rule for step size selection} \label{algo:armijo}
\begin{algorithmic} 
\State {\bf Input}: Dictionary matrix $\bW_{t-1}$, data sample $\bv_t$, coefficient vector $\bh_t^k$, maximum number of iterations $q$
\State {\bf Initialize} $\alpha\in(0,0.5),\gamma\in(0,1),i:=0,\xi^0:=1$
\State {\bf while} $\bard_t(\bh_t^k-\xi^i\nabla\bard_t(\bh_t^k))>\bard_t(\bh_t^k)-\alpha\xi^i\normt{\nabla\bard_t(\bh_t^k)}^2$\\
\hspace{7cm} {\bf and} $i\le q$
\begin{align*}
\xi^{i+1} := \gamma\xi^i,\quad i := i+1
\end{align*}
\State {\bf end} 
\State {\bf Output}: Final step size $\beta_t^k\defeq\xi^i$
\end{algorithmic}
\end{algorithm}

\section{Main Convergence Theorem}

Our main convergence theorem concerns the divergences in class $\calD_2$ (see Remark~\ref{rmk:def_div}), i.e., the divergences $d(\cdot\Vert\cdot)$ that are convex in the second argument. The technical difficulties (and possible approaches) for proving such convergence results for divergences in class $\calD_1\setminus\calD_2$ are discussed in Section~\ref{sec:conv_discussions}.\footnote{However, the efficacy of our algorithm for this case will be empirically verified in Section~\ref{sec:numericals}.}

Before presenting our main theorem, we first make the following assumptions. 
\begin{assump}\quad
\begin{enumerate}
\item The support set $\calV\subseteq\bbR_{++}^F$ for the data generation distribution $\bbP$ is compact. \label{assum:comp_supp}
\item For all $(\bv,\bW)\in\calV\times\calC$, $d(\bv\Vert\bW\bh)$ is $m$-strongly convex in $\bh$ for some constant $m>0$ if $d(\cdot\Vert\cdot)\in\calD_2$.\label{assum:sc_h}
\end{enumerate}
\end{assump}

\begin{remark}
The abovementioned two assumptions are reasonable in the following sense. Assumption~\ref{assum:comp_supp} naturally holds for real data, which are uniformly bounded entrywise. We have $\calV\subseteq\bbR_{++}^F$ as per discussion in Remark~\ref{rmk:prob_form}. 
Assumption~\ref{assum:sc_h} is a classical assumption in literature \cite{Mairal_10,Feng_13,Shen_14b}. It ensures the minimizer of \eqref{eq:solve_h} is unique. 
This assumption can be satisfied by simply adding a Tikhonov regularizer $\frac{m}{2}\norm{\bh}_2^2$ to $d(\bv\Vert\bW\bh)$, but we omit such {a regularization term in the objective function in our analysis}. 
\end{remark}

We now state our main theorem. 
\begin{theorem}\label{thm:main}
As $t\to\infty$, the sequence of dictionaries $\{\bW_t\}_{t\in\bbN}$ converges almost surely to the set of critical points of \eqref{eq:exp_loss} formulated with any divergence in class $\calD_2$. 
\end{theorem}

\begin{remark}
We notice that the same convergence guarantees have been proved in previous works in which {the divergence term in the NMF objective function is} the squared $\ell_2$ loss \cite{Mairal_10,Shen_14b}. Therefore, our result here can be considered as a substantial {\em generalization} of the previous results, since the class $\calD_2$ 
covers many more important divergences, as discussed in Remark~\ref{rmk:def_div}. At a higher level, our problem falls within the scope of {\em stochastic (block) nonconvex  optimization}. 
Without additional assumptions on the regularity of the problem, 
convergence guarantees to the global optima are in general out-of-reach.
Indeed, the state-of-the-art convergence guarantees on such problems 
\cite{Ghad_13,Ghad_16a,Ghad_16b,Reddi_16} are stated in terms of the critical points. Although being suboptimal, the critical points subsume   global minima and are often empirically appealing, especially for matrix factorization problems \cite{Mairal_08,Mairal_10,Bao_15}.
\end{remark}


\section{Convergence analysis} \label{sec:conv_analyses}
This section is devoted to the proof of Theorem~\ref{thm:main}. For simplicity and ease of understanding, {here we focus on the divergences in class $\calD_1\cap\calD_2$.} The proof for the divergences in $\calD_2\setminus\calD_1$ can be similarly established, but with slightly more involved mathematical machinery. We defer the proof for the  divergences in  $\calD_2\setminus\calD_1$ to the supplemental material. 

This section is organized as follows. We first introduce some important notations, concepts, preliminary lemmas, as well as continuous-time interpolations of some discrete-time stochastic processes in Algorithm~\ref{algo:ONMFG}. Then we state the key lemmas that lead to the theorem, together with {sketches of their proofs}. Some technical discussions are provided {at} the end. 

\vspace{-.25cm}
\subsection{Notations and Concepts}
We denote the underlying probability space for the whole stochastic process $\{\bv_t,\bW_t,\bh_t\}_{t\in\bbN}$ generated {by} Algorithm~\ref{algo:ONMFG} as $(\Omega,\scB,\mu)$. 
In the sequel, we need to perform continuous-time interpolations for some discrete-time processes. To distinguish between these two types of processes, we use $t\in\bbN$ 
as the discrete time index and $s\in\bbR_+$ as the continuous time index.
For any $\omega\in\Omega$, we use $\bX_t(\omega)$ and $\bX(\omega,s)$ to denote the values of $\bX_t$ and $\bX(s)$ evaluated at $\omega$ respectively. 

Next we introduce some important concepts   in the analysis.
\begin{definition}[Equicontinuity and asymptotic equicontinuity \cite{Yin_05}]
A sequence of functions $\{f_n\}_{n\in\bbN}$, defined on a common real Banach space $(\calX,\norm{\cdot}_\calX)$ and mapped to a common real Banach space $(\calY,\norm{\cdot}_\calY)$, is {\em equicontinuous} (e.c.) at $x\in\calX$ if for any $\epsilon>0$, there exists $\delta>0$ such that
\begin{equation}
\sup_{n\in\bbN} \;\sup_{x'\in\calX:\norm{x-x'}_\calX<\delta} \norm{f_n(x)-f_n(x')}_\calY < \epsilon
\end{equation} 
and  {\em asymptotically equicontinuous} (a.e.c.) at $x\in\calX$ if
\begin{equation}
\limsup_{n\to\infty} \;\sup_{x'\in\calX:\norm{x-x'}_\calX<\delta} \norm{f_n(x)-f_n(x')}_\calY < \epsilon. 
\end{equation} 
If $\{f_n\}_{n\ge 1}$ is e.c. (resp. a.e.c.) at each $x\in\calX$, then $\{f_n\}_{n\ge 1}$ is e.c. (resp. a.e.c.) on $\calX$. 
\end{definition}

\begin{definition}[Projected dynamical system, limit set and stationary points \cite{Dupuis_93,Teschl_12}] \label{def:PDS}
Given a closed and convex set $\calK$ in a (finite-dimensional) real Banach space $\calX$, and a continuous function $g:\calK\to\calX$, the projected dynamical system (PDS) (on an interval $\calI\subseteq\bbR_+$) associated with $\calK$ and $g$ with initial value $x_0\in\calK$ is defined as
\begin{equation}
\frac{{\rm d}}{{\rm d}s} x(s) = \pi_\calK\Big[x(s),g(x(s))\Big], \; x(0) = x_0,\;s\in\calI, \label{eq:PDS}
\end{equation}
where 
\begin{equation}
\pi_\calK[x,v] \defeq \lim_{\delta\downarrow 0} \frac{\Pi_\calK(x+\delta v)-x}{\delta}, \forall\,x\in\calK,\;\forall\,v\in\calX.
\end{equation}
Denote $\calP(g,\calK,x_0)$ as the solution set of \eqref{eq:PDS}. The limit set of \eqref{eq:PDS}, $\calL(g,\calK,x_0)$ is defined as
\begin{align*}
&\calL(g,\calK,x_0) \defeq \bigcup_{x(\cdot)\in\calP(g,\calK,x_0)} \Big\{y\in\calK\,\Big|\,\exists\,\{s_n\}_{n\in\bbN}\subseteq\bbR_+,\;\\ 
&\hspace{5.5cm} \; s_n\uparrow\infty,\,x(s_n)\to y\Big\}.
\end{align*}
Moreover, the set of stationary points associated with $g$ and $\calK$, $\calS(g,\calK)$ is defined as
\begin{equation}
\calS(g,\calK) \defeq \left\{x\in\calK\,\Big|\,\pi_\calK\Big[x,g(x)\Big]=0\right\}.
\end{equation}
\end{definition}
Finally, since we focus only on the divergences in $\calD_1\cap\calD_2$, the Fr\'{e}chet subdifferential $\hpartial \tild_t(\bW_{t-1})=\{\nabla_\bW d(\bv_t\Vert\bW_{t-1}\bh_t)\}$ (see Algorithm~\ref{algo:ONMFG}).  
\subsection{Preliminary Lemmas}
We first present two lemmas that together establish that 
the stochastic (noisy) gradient $\nabla_\bW d(\bv_t\Vert\bW_{t-1}\bh_t)$ in Algorithm~\ref{algo:ONMFG} acts as an unbiased estimator of the ``true'' gradient $\nabla f(\bW_{t-1})$, for any $t\in\bbN$.

\begin{lemma} \label{lem:regularity}
Given any $(\bv,\bW)\in\calV\times\calW$, 
$\ell$ is differentiable at $(\bv,\bW)$ and $(\bv,\bW)\mapsto\nabla_\bW\ell(\bv,\bW)$ is continuous at $(\bv,\bW)$. Moreover, let $\bh^*(\bv,\bW)\defeq\min_{\bh\in\calH}d(\bv\Vert\bW\bh)$ (by Assumption~\ref{assum:sc_h}), then $\nabla_\bW\ell(\bv,\bW) = \nabla_\bW d (\bv\Vert\bW\bh^*(\bv,\bW))$. Consequently, there exists $M\in (0,\infty)$ such that $\normt{\nabla_\bW\ell(\bv,\bW)}\le M$, 
for all $(\bv,\bW)\in\calV\times\calW$. 
\end{lemma}
\begin{proof}
It is easy to check that i) $(\bv,\bW) \mapsto d(\bv\Vert\bW\bh)$ is differentiable on $\calV\times\calC$, for each $\bh\in\calH$, ii) $(\bv,\bW, \bh) \mapsto d(\bv\Vert\bW\bh)$ is continuous on $\calV\times\calC\times\calH$ and iii)  $(\bv,\bW, \bh) \mapsto \nabla_\bW d(\bv\Vert\bW\bh)$ and $(\bv,\bW, \bh) \mapsto \nabla_\bv d (\bv\Vert\bW\bh)$ are both continuous on $\calV\times\calC\times\calH$. 
Furthermore, Assumption~\ref{assum:sc_h} implies $\bh^*(\bv,\bW)$ is a unique minimizer of \eqref{eq:solve_h} for each $(\bv,\bW)\in\calV\times\calC$. 
Then by the compactness of $\calH$ and the maximum theorem (see Lemma~\ref{lem:maximum}), $\bh^*(\bv,\bW)$ is continuous on $\calV\times\calW$.
By Danskin's theorem (see Lemma~\ref{lem:Danskin}) and again by the compactness of $\calH$, $\ell(\bv,\bW)$ is differentiable on $\calV\times\calW$ and $\nabla_\bW\ell(\bv,\bW) = \nabla_\bW d (\bv\Vert\bW\bh^*(\bv,\bW))$, which is continuous on $\calV\times\calW$. Since $\calV\times\calW$ is compact (by Assumption~\ref{assum:comp_supp}), 
there exists $M\in(0,\infty)$ such that $\normt{\nabla_\bW \ell(\bv,\bW)}\le M$, for all $(\bv,\bW)\in\calV\times\calW$. 
\end{proof}

\begin{lemma}\label{lem:unbiased}
The expected loss (objective) function $f$ is continuously differentiable on $\calC$ and $\nabla f(\bW) = \bbE_{\bv\sim\bbP}\left[\nabla_\bW\ell(\bv,\bW)\right]$ for each $\bW\in\calC$. 
\end{lemma}
\begin{proof}
Since both $\bv\mapsto\ell(\bv,\bW)$ and $\bv\mapsto\nabla_\bW\ell(\bv,\bW)$ are continuous on $\calV$ (by Lemma~\ref{lem:regularity}) and $\calV$ is compact, both of them are Lebesgue integrable. Thus, by Leibniz integral rule (see Lemma~\ref{lem:Leib_int}), we have $\nabla f(\bW) = \bbE_\bv\left[\nabla_\bW\ell(\bv,\bW)\right]$ for each $\bW\in\calC$. The continuity of $\nabla f$ on $\calC$ is implied by the continuity of $\nabla_\bW\ell(\bv,\bW)$ on $\calV\times\calW$.
\end{proof}

\begin{corollary}\label{cor:bounded}
We have $\sup_{t\in\bbN}\bbE\left[\normt{\nabla_\bW\ell(\bv_t,\bW_{t-1})}^2\right]\le M^2$. Moreover, there exists  $M'\in(0,\infty)$ such that for each $\omega\in\Omega$,  $\sup_{t\in\bbN}\normt{\nabla f(\bW_{t-1}(\omega))}\le M'$. 
\end{corollary}

Now, define the ``noise'' part in the stochastic gradient\footnote{$\bN_t$ is a function of both $\bv_t$ and $\bW_{t-1}$, but we omit such dependence to make notations uncluttered.} $\nabla_\bW d (\bv_t\Vert\bW_{t-1}\bh_t)$ in \eqref{eq:upd_dict}, $\bN_t$ as
\begin{equation}
\bN_t\defeq \nabla_\bW\ell(\bv_t,\bW_{t-1}) - \nabla f(\bW_{t-1}).
\end{equation}
We also define a filtration $\{\scF_t\}_{t\ge 0}$ such that $\scF_t \defeq \sigma\{\bv_i,\bW_i,\bh_i\}_{i=1}^t$ for all $t\ge 1$ and $\scF_0 = \{\emptyset,\Omega\}$. 

\begin{lemma}
There exists a constant $M''\in(0,\infty)$ such that 
\begin{equation}
\sup_{t\in\bbN}\bbE\left[\normt{\bN_t}^2\right]\le {M''}^2.  \label{eq:bounded_noise}
\end{equation}  
Moreover, $\{\bN_t\}_{t\ge 1}$ is a martingale difference sequence adapted to $\{\scF_t\}_{t\ge 0}$.
\end{lemma}
\begin{proof}
The bound in \eqref{eq:bounded_noise} is an immediate consequence of Corollary~\ref{cor:bounded}. It also implies for each $t\in\bbN$, $\bbE[\normt{\bN_t}]<\infty$. Moreover, with probability one,
\begin{align*}
\bbE\left[\bN_t\vert\scF_{t-1}\right] &= \bbE\left[\nabla_\bW\tilell(\bv_t,\bW_{t-1})\vert\scF_{t-1}\right] - \nabla f(\bW_{t-1}) \\
&= \nabla f(\bW_{t-1}) - \nabla f(\bW_{t-1}) 
= \vecz. \qedhere
\end{align*}
\end{proof}

\subsection{Continuous-time Interpolations}
Observe that \eqref{eq:upd_dict}, which lies at the central part in our analysis, is a discrete-time PDS. We find it more convenient to  analyze a continuous-time analogue of it, so we perform (continuous-time) constant interpolation on  \eqref{eq:upd_dict}. 
Specifically, we first explicitly model the projection $\Pi_\calC$ in \eqref{eq:upd_dict} in terms of an additive noise term $\bZ_t$, i.e.,
\begin{align}
\bW_t 
&:= \bW_{t-1} - \eta_t \nabla f(\bW_{t-1}) - \eta_t\bN_t + \eta_t\bZ_t, \label{eq:diff_eqn}
\end{align}
where 
\begin{align}
\bZ_t &\defeq \frac{1}{\eta_t}\Pi_\calC\Big\{\bW_{t-1} - \eta_t \nabla_\bW\tilell(\bv_t,\bW_{t-1})\Big\}\nn\\
&\hspace{1.5cm} - \frac{1}{\eta_t}\Big\{\bW_{t-1} - \eta_t\nabla_\bW\tilell(\bv_t,\bW_{t-1})\Big\}. \label{eq:def_Z}
\end{align}
Then we define three sequences of 
functions $\{F^t\}_{t\in\bbN}$, $\{N^t\}_{t\in\bbN}$ and $\{Z^t\}_{t\in\bbN}$ with common domain $\bbR_{+}$ as 
\begin{align}
F^t(s) &\defeq -\sum_{i=t}^{m(s_t+s)-1} \eta_{i+1}\nabla f(\bW_{i}),\\
N^t(s) &\defeq -\sum_{i=t+1}^{m(s_t+s)} \eta_i\bN_i, \\
Z^t(s) &\defeq \sum_{i=t+1}^{m(s_t+s)} \eta_i\bZ_i,  \label{eq:def_Zt_CT}
\end{align}
for $s>0$ and $F^t(0) = N^t(0) = Z^t(0)\defeq 0$, where 
\[ s_t \defeq  \left\{\hspace{-.2cm}\begin{array}{ll}
0, \; &t = 0\\
\sum_{i=1}^{t}\eta_i, \; &t\ge 1
\end{array},\right. \quad
m(s) \defeq \left\{\hspace{-.2cm}\begin{array}{ll}
0, \; &s=0\\
t, \; &s \in (s_{t-1},s_t]
\end{array}.\right.
\]
For illustration purpose, one realization of $Z^t(s)$, $Z^t(\omega,s)$ is plotted in Figure~\ref{fig:plot_CT}. 
Define $W^t(s) \defeq \bW_{m(s_t+s)-1}$. 
By \eqref{eq:diff_eqn}, for any $t\in\bbN$ we have for all $s\ge 0$,
\begin{equation}
W^t(s) = W^t(0) + F^{t-1}(s) + N^{t-1}(s) + Z^{t-1}(s). \label{eq:def_Wt}
\end{equation}

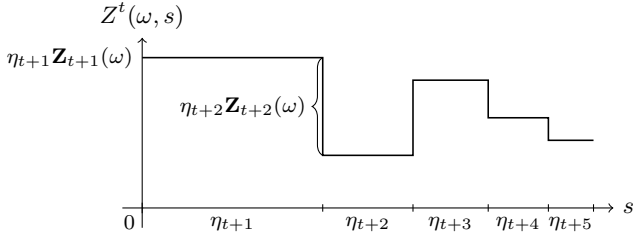
\begin{figure}[t]
\centering
\begin{tikzpicture}
\datavisualization [school book axes,
                    visualize as line,
                    y axis={ticks={major={at={0}}},label={$Z^t(\omega,s)$}},
                    x axis={ticks={minor={at = {2.4,3.6,4.6,5.4,6}},major={at={0}}},label={$s$}}
                    ]

data {
       x, y
       0, 2
       2.4, 2
       2.4, 0.7
       3.6, 0.7
       3.6, 1.7
       4.6, 1.7
       4.6, 1.2
       5.4, 1.2
       5.4, 0.9
       6, 0.9
      };
      
info{
\node[below] at (1.2,0) {\footnotesize $\eta_{t+1}$};
\node[below] at (3,0) {\footnotesize $\eta_{t+2}$};
\node[below] at (4.1,0) {\footnotesize $\eta_{t+3}$};
\node[below] at (5,0) {\footnotesize $\eta_{t+4}$};
\node[below] at (5.7,0) {\footnotesize $\eta_{t+5}$};
\node[left] at (2.3,1.35) {\footnotesize $\eta_{t+2}\bZ_{t+2}(\omega)$};
\node[left] at (0,2) {\footnotesize $\eta_{t+1}\bZ_{t+1}(\omega)$};
\draw[decorate,decoration={brace,amplitude=4pt}] (2.4,0.7) -- (2.4,2);
};
\end{tikzpicture}\vspace{-.3cm}
\caption{Plot of $Z^t(\omega,\cdot)$ on $\bbR_+$ for some $\omega\in\Omega$. Here $\bZ_t(\omega)$ refers to a scalar. Also, $Z^t(\omega,\cdot)$ is always left-continuous by \eqref{eq:def_Zt_CT}, but we omit such less important details in the plot.} \vspace{-.5cm}\label{fig:plot_CT}
\end{figure}

\vspace{-.5cm}
\subsection{Key Lemmas}

This section consists of three lemmas. In particular, our main theorem is an immediate consequence of Lemmas~\ref{lem:conv_limit_set} and \ref{lem:char_limit_set}. 
Before presenting the two lemmas, we first present Lemma~\ref{lem:asym_equi} since it lays the foundations for proving Lemma~\ref{lem:conv_limit_set}. 
We provide the proof sketch of Lemma~\ref{lem:conv_limit_set} and  defer the detailed proofs of all the lemmas to the supplemental material. 

\begin{lemma}[Almost sure asymptotic equicontinuity of important functions]\label{lem:asym_equi}
For any $t\in\bbN$, define 
\begin{align}
G^t(s) &\defeq -\int_0^s \nabla f(W^t(\tau)) \,d\tau, \;\;s\ge 0, \label{eq:def_tilG}\\
Y^t(s) &\defeq \int_0^s Z^t(\tau) \,d\tau,\;\;s\ge 0. 
\end{align}
Then we have
\begin{enumerate}
\item $N^t\convu \vecz$ on $\bbR_+$ almost surely. 
\item $\Delta_1^t\defeq F^t-G^t\convu \vecz$ on $\bbR_+$ almost surely.
\item $\{G^t\}_{t\in\bbN}$ is equicontinuous on $\bbR_+$ almost surely. 
\item $\Delta_2^t\defeq Z^t-Y^t\convu \vecz$  and $\{Y^t\}_{t\in\bbN}$ is equicontinuous on $\bbR_+$ almost surely. 
\end{enumerate}
Consequently, $\{N^t\}_{t\in\bbN}$, $\{F^t\}_{t\in\bbN}$, $\{Z^t\}_{t\in\bbN}$ and $\{W^t\}_{t\in\bbN}$ are asymptotically equicontinuous on $\bbR_+$ almost surely.
\end{lemma}
\begin{proof}
See Section~\ref{sec:proof_asym_equi} in the supplemental material. 
\end{proof}
\begin{lemma}[Almost sure convergence to the limit set] \label{lem:conv_limit_set}
The stochastic process $\{\bW_t\}_{t\in\bbN}$ generated in Algorithm~\ref{algo:ONMFG} converges almost surely to $\calL(-\nabla f,\calC,\bW_0)$, the limit set of the following projected dynamical system 
\begin{equation}
\frac{{\rm d}}{{\rm d}s}W(s) = \pi_\calC\Big[W(s),-\nabla f(W(s))\Big], \;W(0) = \bW_0, \;s\ge 0. \label{eq:PDS_W0}
\end{equation}
\end{lemma}
\begin{psketch}
First we rewrite \eqref{eq:PDS_W0} as
\begin{equation}
\frac{{\rm d}}{{\rm d}s}W(s) = -\nabla f(W(s)) + z(s), \;W(0) = \bW_0, \;s\ge 0, \label{eq:PDS2_W0}
\end{equation}
where $z(s)\defeq \pi_\calC\Big[W(s),-\nabla f(W(s))\Big] + \nabla f(W(s)), s\in\bbR_+$. 
For convenience, we will analyze the integral form of \eqref{eq:PDS2_W0}:
\begin{align}
&W(s) = -\int_0^s\nabla f(W(\tau)) {\rm d}\tau + \int_0^s z(\tau) \rmd\tau,\nn \\
&\hspace{4.5cm}\;W(0) = \bW_0, \;s\ge 0. \label{eq:PDS2_W0_int}
\end{align}
By Lemma~\ref{lem:asym_equi}, there exists an almost sure set\footnote{An almost sure set $\calA\subseteq\Omega$ is the set with probability one, i.e., $\mu(\calA)=1$.} $\calA\subseteq\Omega$ such that for each $\omega\in\calA$, $\{W^t(\omega,\cdot)\}_{t\in\bbN}$ and $\{Z^t(\omega,\cdot)\}_{t\in\bbN}$ are a.e.c. and uniformly bounded. Fix any $\omega\in\calA$ and $S\in(0,\infty)$. By the (generalized) Arzel\`{a}-Ascoli Theorem (see Lemma~\ref{lem:arzela}), there exists a subsequence $\{t_k\}_{k\in\bbN}$ of $\{1,2,3,\ldots\}$ with $t_k\uparrow\infty$ and continuous functions $\barW(\omega,\cdot)$ and $\barZ(\omega,\cdot)$ (both defined on $[0,S]$) such that $W^{t_k}(\omega,\cdot)\convu \barW(\omega,\cdot)$ and $Z^{t_k}(\omega,\cdot)\convu \barZ(\omega,\cdot)$ on $[0,S]$. The continuity of $\nabla f$ implies $\nabla f(W^{t_k}\left(\omega,\cdot)\right) \convu \nabla f\left(\barW(\omega,\cdot)\right)$ on $[0,S]$. This further implies  $G^{t_k}(\omega,\cdot)\convu \barG(\omega,\cdot)$ on $[0,S]$, where
\begin{equation}
\barG(\omega,s)\defeq-\int_0^s \nabla f\left(\barW(\omega,\tau)\right)\,d\tau, \;s\ge 0.
\end{equation}
Since $\Delta_1^{t_k}(\omega,\cdot)\convu \vecz$ and $N^{t_k}(\omega,\cdot)\convu \vecz$ on $[0,S]$, 
we have
\begin{align}
&\barW(\omega,s) = \barW(\omega,0) - \int_0^s \nabla f\left(\barW(\omega,\tau)\right)\,d\tau + \barZ(\omega,s), \nn\\
&\hspace{6cm}\;s\in[0,S].  \label{eq:barWZ_int}
\end{align}
Now,  to show $\left\{\barW(\omega,\cdot),\barZ(\omega,\cdot)\right\}$ satisfies \eqref{eq:PDS2_W0_int} on $[0,S]$, it remains to show 
\begin{equation}
\barZ(\omega,s) = \int_0^s z(\tau)\;d\tau, s\in[0,S]. \label{eq:barZ_int}
\end{equation}
By \cite[Theorem 3.1, Chapter 4]{Kushner_03} (see Lemma~\ref{lem:suff_pode}), it suffices to show  $\barW(\omega,\cdot)$ is Lipschitz on $[0,S]$ and for any $\tau\in[0,S]$ the following two conditions hold:
\begin{enumerate}
\item $\barZ(\omega,\tau)=\vecz$ if $\barW(\omega,s)\in\inter\calC$ for almost all $s\in[0,\tau]$, 
\item $\barZ(\omega,\tau)\in \convcl \left[\bigcup_{s\in[0,\tau]}\calN\left(\barW(\omega,s)\right)\right]$,
\end{enumerate}
where $\convcl\calX$ denotes the closed convex hull of a set $\calX$ and the correspondence (see Definition~\ref{def:correspondence}) $\calN:\calC\rightrightarrows\bbR^{F\times K}$ is defined as $\calN(\bW) \defeq \{\bN\in\bbR^{F\times K}\,|\,\norm{\bN}\le M,\,\lrangle{\bN}{\bW'-\bW}\ge 0,\forall\,\bW'\in\calC\}$. 
In words, $\calN(\bW)$ indicates the set of (bounded) inward normals at $\bW\in\calC$. 

The Lipschitzness of $\barW(\omega,\cdot)$ on $[0,S]$ follows from i) $\barZ(\omega,\cdot)$ is Lipschitz on $[0,S]$ and ii) $\nabla f\left(\barW(\omega,\cdot)\right)$ is bounded on $[0,S]$. Moreover, condition (1) above essentially follows from the continuity of $\barZ(\omega,\cdot)$ on $[0,S]$. To show condition (2), we make use of the upper semicontinuity (see Definition~\ref{def:usc}) of $\calN$. To show this property, by \cite[Section~1.1]{Aubin_84} (see Lemma~\ref{lem:suff_usc}), it suffices to show 
\begin{equation}
\bigcap_{\delta>0}\convcl\Big[\bigcup_{\bW'\in\calB_\delta(\bW)}\calN(\bW')\Big] \subseteq \calN(\bW),\,\forall\,\bW\in\calC, \label{eq:suff_usc}
\end{equation}
 where $\calB_\delta(\bW)\defeq \{\bW'\in\calC\,|\,\norm{\bW-\bW'}<\delta\}$. 

Based on \eqref{eq:barWZ_int} and \eqref{eq:barZ_int},
we can choose a sequence of nested intervals $[0,S_n]$ (with $S_n\uparrow\infty$) and by repeatedly passing to further subsequences, we can show there exists a subsequence $\{\bart_k\}_{k\in\bbN}$ of $\{t_k\}_{k\in\bbN}$ and continuous functions $\tilW(\omega,\cdot)$ and $\tilZ(\omega,\cdot)$ (both defined on $\bbR_+$) such that $W^{\bart_k}(\omega,\cdot)\convu \tilW(\omega,\cdot)$ and $Z^{\bart_k}(\omega,\cdot)\convu \tilZ(\omega,\cdot)$ on $\bbR_+$. Moreover, $\left\{\tilW(\omega,\cdot),\tilZ(\omega,\cdot)\right\}$ satisfies \eqref{eq:PDS2_W0_int} on $\bbR_+$. Then it follows that each subsequential limit of $\bW_t(\omega)$ belongs to $\calL(-\nabla f,\calC,\bW_0)$. 
\end{psketch}
\vspace{-.8cm}
\begin{lemma}[Characterization of the limit set]\label{lem:char_limit_set}
In \eqref{eq:PDS_W0}, we have $\calL(-\nabla f,\calC,\bW_0)\subseteq \calS(-\nabla f,\calC)$, i.e., every limit point of \eqref{eq:PDS_W0} is a stationary point associated with $-\nabla f$ and $\calC$. Moreover, each $\bW\in\calS(-\nabla f,\calC)$ satisfies the following variational inequality 
\begin{equation}
\lrangle{\nabla f(\bW)}{\bW'-\bW} \ge 0, \,\forall\,\bW'\in\calC. \label{eq:variat}
\end{equation}
This implies each stationary point in  $\calS(-\nabla f,\calC)$ is a critical point of \eqref{eq:exp_loss}. 
\end{lemma}
\begin{proof}
See Section~\ref{sec:proof_char_limit_set} in the supplemental material. 
\end{proof}

\begin{remark}
Lemma~\ref{lem:conv_limit_set} and \ref{lem:char_limit_set} together imply a two-step approach to prove Theorem~\ref{thm:main}. Specifically, Lemma~\ref{lem:conv_limit_set} shows $\{\bW_t\}_{t\in\bbN}$ converges almost surely to the limit set $\calL(-\nabla f,\calC,\bW_0)$  and Lemma~\ref{lem:char_limit_set} characterizes every element in the limit set $\calL(-\nabla f,\calC,\bW_0)$  as a critical point of \eqref{eq:exp_loss}. 
\end{remark}

\vspace{-.5cm}
\subsection{Discussions} \label{sec:conv_discussions}
We first remark that for the divergences in class $\calD_1\cap\calD_2$, it might be possible to analyze the convergence of Algorithm~\ref{algo:ONMFG} under the stochastic MM framework, by choosing the quadratic majorant of the sample average of $f$, as per discussion in \cite[Section~4]{Raza_16}. However, our analysis based on stochastic approximation theory and projected dynamical systems \cite{Dupuis_93,Teschl_12} serve as a more direct approach, since we need not transform Algorithm~\ref{algo:ONMFG} as a stochastic MM algorithm a priori. 
Next, we discuss the difficulties to tackle the divergences in class $\calD_1\setminus\calD_2$. In such case \eqref{eq:solve_h} may not have a unique minimizer even if some strongly convex regularizer is added. Thus, $f$ would be nonsmooth and nonconvex. Without additional assumptions,\footnote{For example, in \cite{Ghad_16a,Ghad_16b,Reddi_16,Raza_16}, the authors assume the objective function $f$ in \eqref{eq:exp_loss} can be decomposed into two parts, one being nonconvex but smooth and the other being nonsmooth but convex. However, such assumption does not cover our case.} proving asymptotic convergence guarantees to stationary points is still an open question in the literature. Moreover, nonconvexity makes solving \eqref{eq:solve_h} NP-hard. If we assume there exists an oracle that can solve \eqref{eq:solve_h}, a possible approach to prove convergence to critical points would be to generalize the convergence analysis for the SPSGD method to the nonconvex problems.\footnote{In such case, the subgradient should be defined as in the context of nonconvex analysis, e.g., see \cite[Section~3.1]{Bao_15}.}

\vspace{-.4cm}
\section{Numerical Experiments and Applications} \label{sec:numericals}
\subsection{Experimental Setup}
The experiments in this section consist of three parts. We tested our online algorithm (denoted as {\sf OL}) on synthetic data 
with a broad class of divergences $\barcalD$, such that at least one divergence is from the Csisz\'{a}r, Bregman or robust categories. (The divergences in $\barcalD$ are listed in Table~\ref{table:div_numericals}.) 
Next we applied \ol to real applications  in which large-scale data is commonplace. In particular, we applied \ol with the KL divergence to topic learning on text datasets and \ol with the Huber loss to shadow and noise removal on face datasets. 
For all the experiments and each divergence in $\barcalD$, we compared the performances of our online algorithm \ol to that of its batch counterpart, {\sf Batch}. All the batch algorithms have been derived based on the multiplicative updates (MU) in previous works \cite{Fev_11,Wang_13,Ding_11}. For the IS, squared-$\ell_2$ and Huber losses, we additionally compared \ol with other online algorithms proposed previously \cite{Lef_11,Guan_12,Wang_13}, denoted as {\sf OL-Lef}, {\sf OL-Guan} and {\sf OL-Wang} respectively. 
All the experiments were run in 64-bit $\mbox{Matlab}^\circledR$ (R2015b) on a machine with $\mbox{Intel}^\circledR$ $\mbox{Core}$ i7-4790 3.6 GHz CPU and 8 GB RAM.  {We intend to make the code publicly available if and when the paper is accepted.}

\begin{table*}[t]\footnotesize
\centering
\begin{threeparttable}
\caption{Divergences in class $\barcalD$ and their corresponding noise generation procedures}\label{table:div_numericals}
\begin{tabular}{|c|c|c|c|}\hline
 & Data Generation &Expressions of Distributions\textsuperscript{a} & Parameter Value\\\hline
IS & $\barv_{ij}\sim\scG(v_{ij};\kappa,v^o_{ij}/\kappa)$ & $\scG(x;\kappa,\theta) \defeq x^{\kappa-1}e^{-x/\theta}/(\theta^\kappa\Gamma(\kappa)),x\in\bbR_+$ &$\kappa = 1000$\\ \hline
KL & $\barv_{ij}\sim\scP(v_{ij};v^o_{ij})$ &$\scP(k;\lambda) \defeq \lambda^k e^{-\lambda}/k!,\,k\in\bbN\cup\{0\}$ & --- \\\hline
Squared-$\ell_2$ & $\barv_{ij}\sim\scN(v_{ij};v^o_{ij},\varsigma^2)$ &$\scN(x;\mu,\varsigma^2)\defeq1/\sqrt{2\pi\varsigma^2}\exp\{-(x-\mu)^2/(2\varsigma^2)\},x\in\bbR$ &$\varsigma=30$\\ \hline
Huber, $\ell_1$, $\ell_2$ &$\barv_{ij}\sim\scU(v_{ij};v^o_{ij},\lambda),(i,j)\in\calQ$ &$\scU(v_{ij};v^o_{ij},\lambda)\defeq 1/(2\lambda),x\in[v^o_{ij}-\lambda,v^o_{ij}+\lambda]$ & $\lambda=2000$\\ \hline
\end{tabular}
\begin{tablenotes}
\item[a] The function $\Gamma(\cdot)$ in the expressions of distributions denotes the Gamma function. 
\end{tablenotes}
\end{threeparttable}\vspace{-.5cm}
\end{table*}

\vspace{-.1cm}
\subsection{Heuristics}\label{sec:heuristics}
In the practical implementations of online matrix factorization algorithms, many heuristics have been proposed in previous works\cite{Mairal_10,Guan_12}. In the following  experiments we mainly used three heuristics---namely,   mini-batch input, dataset aggregation and random permutation. Mini-batch input refers to the practice to input $\tau\in\bbN$ data samples at each time. This helps to improve the stability of the dictionary $\bW$ by preventing it from being updated too frequently. Moreover, in reality it may be difficult to find suitable {\em benchmarking} datasets of sufficiently large size to be considered as ``Big Data''. Therefore, one can first replicate an available dataset $p\in\bbN$ times and then add \iid observation noise to each element in the replicated dataset. In this work, $p$ is chosen such that the number of data samples $N\approx 1\times10^5$.
Next, these noisy data samples are randomly permuted and projected onto a compact set $\calV$. In such way, the data samples can be considered to be generated in an \iid fashion from a continuous distribution\footnote{The continuity follows from that most types of observation noise (such as Gaussian, Gamma, etc.) have continuous support, except Poisson noise.} $\bbP$ with compact support. 

\vspace{-.5cm}
\subsection{Parameter Settings}\label{sec:param_setting}
We describe the choices of some important parameters in our algorithms. These parameters include the mini-batch size $\tau$, the latent dimension $K$ and the sequence of step sizes $\{\eta_t\}_{t\in\bbN}$. The {\em canonical setting} of these parameters includes $\tau = 1\times10^{-4}N$, $K=40$ and $\eta_t= a/(\tau t+b)$, where $a=2\times 10^4$, $b=2\times 10^4$ and $t\in\bbN$. This setting will be used in all experiments unless otherwise mentioned. 

We next explain the choices of the parameter values. The mini-batch size $\tau$ controls the frequency of dictionary update. In the online NMF literature there are no principled ways to select $\tau$, since this parameter is typically data dependent \cite{Mairal_10}. Here we followed the rule-of-thumb proposed in \cite{Zhao_16b}, which suggests to choose $\tau\le 4\times10^{-4}N$. 
For the latent dimension $K$, there are several ways to choose it. The most direct way leverages domain knowledge. For example, if the data matrix is the term-document matrix (see Section~\ref{sec:shadow_rmv_denoising} for detailed descriptions), then $K$ corresponds to the number of topics (categories) that the documents belong to (such that each document can be viewed as a linear combination of keywords in each topic). Since this number is known for most text datasets, the value of $K$ can be directly obtained. Otherwise, some works \cite{Tan_13a,Bishop_98} propose to choose $K$ using Bayesian modeling. However, the computational burden introduced by the complex modeling is prohibitive especially for large-scale data. Hence,  we set $K=40$ unless a more accurate estimate can be obtained from the domain knowledge. 
Lastly we discuss the choice of the step size $\eta_t$. From \eqref{eq:step_size}, a straightforward expression for $\eta_t$ would be $\eta_t = a/(\tau t+b)$, where 
$a$ and $b$ are both positive numbers. In the initial phase where $t$ is small, the step size approximately equals the constant $a/b$. Therefore the value of $b$ determines the duration of this phase. Similar to $\tau$, the choice of $b$ is also data-dependent and lacks clear guidelines. As such, we fixed $b=2\times 10^4$ as we found this value gave us satisfactory results in practice. Moreover, we also set $a=2\times 10^4$. 
In Section~\ref{sec:synthetic}, we will show our online algorithm {\sf OL} is insensitive to $\tau$, $K$ and $a$ when we varied these parameters over wide ranges for all the divergences in Table~\ref{table:div_numericals}.

\vspace{-.3cm}
\subsection{Synthetic Experiments} \label{sec:synthetic}
\subsubsection{Data Generation}
To generate the (noisy) data matrix $\bV$, we first generated the ground-truth data matrix $\bV^o\defeq\bW^o\bH^o$, where $\bW^o\in\bbR_+^{F\times K^o}$ and $\bH^o\in\bbR_+^{K^o\times N}$ denotes the ground-truth dictionary and coefficient matrices respectively. We set the ground-truth latent dimension $K^o=40$.\footnote{Note that in general $K\ne K^o$, i.e., the latent dimension $K$ given a priori in the algorithm may not match the ground-truth $K^o$.} The entries of $\bW^o$ and $\bH^o$ were generated \iid from the shifted half-normal distribution $\scHN_\varkappa(\sigma^2)$.\footnote{$\scHN_\varkappa\left(\sigma^2\right)$ denotes the shifted half-normal distribution with scale parameter $\sigma^2$ and offset $\varkappa>0$, i.e., $\scHN_\varkappa(y ; \sigma^2)=\frac{\sqrt{2}}{\sigma\sqrt{\pi}}\exp(-\frac{(y-\varkappa)^2}{2\sigma^2})$ for $y\ge \varkappa$ and $0$ otherwise.} We set $\varkappa=1$ to prevent entries of $\bW^o$ and $\bH^o$ from being arbitrarily small. 
Next we contaminated $\bV^o$ with entrywise \iid noise to obtain $\barbV$. For the IS, KL and squared-$\ell_2$ divergences, the distributions of the noise were chosen to be multiplicative Gamma, Poisson and additive Gaussian respectively, so that 
the ML estimation of $\bV^o$ from $\barbV$ is equivalent  to solving the (batch) NMF problem \eqref{eq:batch_NMF} \cite{Fev_09}. The parameters of these distributions were chosen such that the signal-to-noise ratio (SNR), ${\rm SNR}\defeq 20\log_{10}(\normt{\bV^o}/\normt{\barbV-\bV^o})$ was approximately $30$ dB. In particular, we chose $\sigma=5$ to ensure the SNR for the Poisson noise satisfied the condition.\footnote{By assuming the entries of $\bV^o$ are \iid and using the law of large numbers, all the distribution parameters (and $\sigma$) can be analytically estimated. See \cite{Tan_13a} for details.} Since the other divergences considered (Huber, $\ell_1$ and $\ell_2$) are mainly used in the robust NMF, we added  outliers to the ground-truth $\bV^o$ as follows. We first randomly selected an index set $\calQ\defeq\Pi_{i\in[N]}\calQ_i$ such that for any $i\in[N]$, $\calQ_i\in[F]\times\{i\}$ and $\abs{\calQ} = 0.3F$. Then each entry $v^o_{ij}$ with $(i,j)\in\calQ$ was contaminated with (symmetric) uniform noise with magnitude $\lambda$. We chose $\lambda = 2\bbE[v^o_{ij}] = 2K^o(\varkappa+\sigma\sqrt{2/\pi})^2\approx 2000$. 
The noise generation procedures for all the abovementioned divergences are summarized in Table~\ref{table:div_numericals}. The final data matrix $\bV$ was obtained by projecting $\barbV$ onto a compact set $\calV\defeq [0,4000]^{F\times N}$ since $4000\approx4\bbE[v^o_{ij}]$.

\subsubsection{Efficiency Comparison with Other NMF Algorithms}
We compared the convergence speeds of \ol with the batch (and other online) algorithms for each divergence in $\barcalD$. 
The convergence speeds are demonstrated in the plots of objective values versus time. The objective values of the batch algorithms at each iteration are well-defined. For the online algorithms, at time $t$, the objective value was defined as the {\em empirical loss} w.r.t.\ the $\bW_t\in\calC$, i.e., $1/t\sum_{i\in[t]} d(\bv_i\Vert\bW_t\bh_i)$. Moreover, in all the comparisons, we used the canonical parameter setting for \ol.  The results are shown in Figure~\ref{fig:comparison}. From the results we observe in general, our online algorithm \ol converges significantly faster than \batch for all the divergences. Furthermore, it either converges as fast as (for the Huber loss)   or significantly faster than (for the IS and squared-$\ell_2$ losses) other state-of-the-art online NMF algorithms. This demonstrates the superior {\em computational efficiency} of \ol compared to the batch algorithms and some other online algorithms for all the divergences in $\barcalD$.
Note that the results in Figure~\ref{fig:comparison} were obtained using one initialization of $\bW$ for each divergence. We observed that different initializations led to similar results.
\subsubsection{Insensitivity to Key Parameters} \label{sec:insensitivity}
To examine the sensitivity of \ol to the key parameters $\tau$, $K$ and $a$, for each divergence in $\barcalD$, we varied one parameter at each time in log-scale while keeping the other two fixed as in the canonical setting. From the plots of objective values versus time with different values of $\tau$, $K$ and $a$ (shown in Figures~\ref{fig:param_tau}, \ref{fig:param_K} and \ref{fig:param_a} in the supplemental material respectively), we observe that the convergence speeds of \ol for all the divergences exhibit small (or even unnoticeable) variations across different values of $\tau$, $K$ and $a$. This shows the performance of our online algorithm is relatively insensitive to these key parameters. Therefore, in the following experiments on real data, we will use the canonical values of $\tau$, $K$ and $a$ unless mentioned otherwise. Note that similarly to Figure~\ref{fig:comparison}, the results shown in Figures~\ref{fig:param_tau}, \ref{fig:param_K} and \ref{fig:param_a} are also relatively insensitive to different initializations of $\bW$. 

\vspace{-.3cm}
\subsection{Application I: Topic Learning}\label{sec:topic_learn} 
We applied \ol with the KL divergence to the topic learning task on two text datasets, {\tt BBCNews} \cite{Greene_06} and {\tt 20NewsGroups} \cite{Lang_95}, since it has been shown empirically in  \cite{Lee_99,Yang_11} that the batch NMF algorithm with the KL divergence achieves promising performance on such task. The small-scale {\tt BBCNews} dataset has $2225$ documents from five categories, while the large-scale {\tt 20NewsGroups} dataset has $18774$ documents from $20$ categories. (For both datasets, the labels of ground-truth categories are provided.) Therefore, the latent dimension $K$ was set to $5$ and $20$ respectively for the {\tt BBCNews} and the {\tt 20NewsGroups} datasets, as discussed in Section~\ref{sec:param_setting}.
The five categories 
in the {\tt BBCNews} dataset are rather distinct. However, some of the $20$ categories in the {\tt 20NewsGroups} dataset are highly correlated, {thus resulting in  the learning of topics  of} this dataset more difficult. Each dataset consists of a term-document matrix $\bUpsilon\in\bbR_+^{m\times n}$ such that $\upsilon_{ij}$ denotes the frequency that the $i$-th word appears in the $j$-th document. We transformed  $\bUpsilon$ into the term frequency-inverse document frequency (TF-IDF) matrix $\bUpsilon'$, such that for any $(i,j)\in[m]\times[n]$, $\upsilon'_{ij} \defeq (1+\log\upsilon_{ij})\log(n/\norm{\bUpsilon_{i:}}_0)$ if $\upsilon_{ij}\ne 0$ and $\upsilon'_{ij}=0$ otherwise. Then for both datasets, we further selected $1000$ most frequent words  (in terms of TF-IDF) in the corpus, so the resulting matrix $\bUpsilon''$ consists of 1000 rows of $\bUpsilon'$ with largest $\ell_1$ norm. Next, we scaled $\bUpsilon''$ entrywise by factor $l$ and contaminated entries of $\bUpsilon''$ with \iid Poisson noise such that the $\mbox{SNR}\approx 30$ dB.\footnote{The SNR is controlled by the scaling factor $l$.} Finally, we replicated the resulting matrix column-wise followed by random permutation and entrywise projection onto the compact interval $[0,2l\max_{i,j}\upsilon''_{ij}]$ per the discussion in Section~\ref{sec:heuristics}.

Next, we compare the performances of \ol and \batch in terms of both quality of learned dictionary $\bW$ and running time. To show the quality of $\bW$, we select eight entries in each column of $\bW$ with largest coefficients and display the corresponding keywords (in decreasing order of their coefficients) in Table~\ref{table:BBCNews} and \ref{table:20News} for the {\tt BBCNews} and the {\tt 20NewsGroups} datasets respectively. As shown in \cite{Lee_99}, the eight selected keywords for each column suffice to indicate the topic that the column corresponds to. 
Due to space constraints, we only show five columns of $\bW$ learned from the {\tt 20NewsGroups} dataset. 
From Table~\ref{table:BBCNews}, we observe that both \ol and \batch are able to learn five distinct topics from the {\tt BBCNews} dataset, and the topics learned exactly coincide with the ground truth. Due to the existence of critical points, the set and order of top eight keywords learned for each topic by \ol and \batch are slightly different. However, the small differences cause no ambiguities for topic identifications.
For the {\tt 20NewsGroups} dataset, as some of the categories in this dataset are highly coupled, the keywords shown in Table~\ref{table:20News} only correspond to the ``general'' (ground-truth) categories. However, this suffices for most practical purposes of topic modeling. Similar to Table~\ref{table:BBCNews}, the results in Table~\ref{table:20News} also indicate the comparable quality of the learned dictionaries by \ol and {\sf Batch}. Despite the similar quality of the learned dictionaries, the running times of \ol are significantly shorter than those of \batch on both datasets, as shown in Table~\ref{table:running_time_text}. This suggests that on the topic learning tasks, \ol can achieve similar results as \batch with much greater computational efficiency. 

\vspace{-.3cm}
\subsection{Application II: Shadow Removal and Image Denoising}\label{sec:shadow_rmv_denoising}
In this section we applied \ol with the Huber loss to shadow and salt-and-pepper noise removal on  the {\tt YaleB} face dataset \cite{YaleB_01}. This dataset consists of 8-bit gray-scale face images of 38 subjects with different poses and illumination conditions. In particular, almost all subjects have $64$ face images and shadows of different areas prevalently exist among these images. Since it is well-known that shadows can be treated as outliers \cite{Candes_11}, we use the Huber loss as the robust loss function in both \ol and {\sf Batch}. Due to storage constraints (for implementing {\sf Batch}), we downsampled all the face images to resolution $32\times 32$. 
To increase the difficulty of image reconstruction, for each image, we uniformly randomly selected $30\%$ of pixels and added \iid salt-and-pepper noise with (symmetric) uniform distribution $\scU(0,255)$ to these pixels. The contaminated pixels were then projected to $[0,255]$. As discussed in Section~\ref{sec:heuristics}, for each subject (now with $64$ contaminated images), we replicated his/her images $p=1500$ times followed by random permutation. The images were then vectorized as columns and stacked to form our data matrix $\bV$. 
We  reconstructed a particular image (vector) $\hatbv_i$ from {\sf OL} by $\hatbv_i=\hatbW\bh_i$, where $\hatbW$ denotes the final output dictionary and $\bh_i$ denotes the coefficient vector at time $i$. In this way, the quality of image reconstruction serves as a good indicator of the quality of $\bW$. 

Now, we randomly pick four subjects and show the reconstruction results of their images in Figure~\ref{fig:image_recons}. For each subject, we select three images with different illuminations and show the images reconstructed by the online algorithms \ol and {\sf OL-Wang} and the batch  algorithm {\sf Batch}. From Figure~\ref{fig:image_recons}, we observe that for all subjects, all the three algorithms are able to remove the salt-and-pepper noise. 
For shadow removal, \batch achieves the best visual quality, in the sense that shadows are completely removed and artifacts (e.g., glares and distortions) introduced are minimal. In contrast, for subjects A, B and C, {\sf OL-Wang} fails to remove shadows. Compared with {\sf OL-Wang}, \ol removes the shadows in their entirety, although it introduces some glares in the foreheads and cheeks. For subject D however, all the three algorithms achieve almost the same results, with shadows removed and little artifacts introduced. Overall, in terms of shadow removal and image denoising, \ol and \batch have similar performances and they greatly outperform {\sf OL-Wang}. Next we turn attention to the running times of the three algorithms. It is clear that for all the four subjects, \ol has the shorted running times and is significantly faster than {\sf Batch}. Also, {\sf OL-Wang} runs slightly slower than {\sf OL}. Therefore, in terms of both efficiency and quality of image reconstruction, we conclude that \ol achieves the best trade-off. 

\begin{remark}
Note that for both experiments on the real datasets (Section~\ref{sec:topic_learn} and \ref{sec:shadow_rmv_denoising}), the running times were averaged over ten random initializations of $\bW$ (with standard deviations shown in the parentheses). Other results shown were obtained with one initialization of $\bW$, but 
they were observed to be insensitive to different initializations. 
\end{remark}

\begin{table}[t!]\footnotesize\centering
\caption{Topics learned from the {\tt BBCNews} dataset with the {\sf  OL} and the {\sf Batch} algorithms with the KL divergence.}\vspace{-.4cm}\label{table:BBCNews}
\subfloat[{\sf OL}]{
\begin{tabular}{|c|c|c|c|c|}\hline
Business & Sports & Technology & Entertainment & Politics\\\hline
        compani &            game &           peopl &            film &          govern\\ 
           firm &            plai &          servic &            show &           elect\\ 
         market &             win &       technolog &            best &          minist\\ 
           2004 &         against &            user &           music &          labour\\ 
          share &          player &             net &            star &           parti\\ 
          price &         england &           phone &           award &           blair\\ 
         growth &            club &           mobil &              tv &            tori\\ 
        economi &            team &          comput &            band &        campaign\\ \hline
\end{tabular}
}\\\vspace{-.3cm}
\subfloat[{\sf Batch}]{
\begin{tabular}{|c|c|c|c|c|}\hline
Business & Sports & Technology & Entertainment & Politics\\\hline
        growth &            game &        technolog &    film &                  elect\\ 
           bank &             win &        user &    award &                     labour\\ 
        compani &            club &         phone &    star &                     parti\\ 
         market &            plai &        mobil &    music &                    govern\\ 
          share &           match &         comput &    best &                   minist\\ 
          price &          player &         softwar &  actor &                    blair\\ 
           firm &            team &          network &  band &                     tori\\ 
         profit &        champion &          internet & album &                     law\\ \hline
\end{tabular}
}\\\vspace{-.6cm}
\end{table}

\begin{table}[t!]\footnotesize
\centering
\caption{Topics learned from the {\tt 20NewsGroups} dataset with the {\sf OL} and the {\sf Batch} algorithms with the KL divergence.}\vspace{-.4cm} \label{table:20News}
\subfloat[{\sf OL}]{
\begin{tabular}{|c|c|c|c|c|}\hline
Space & Sports & Religion & Hardware & Sale \\\hline
nasa & games & jesus & drive & sale \\
gov & players & jewish  & vga & condition \\
look & hockey & god & computer  & buy \\
usa & season & israel & disks & excellent \\
test & team & christians & dx & original \\
engineering &  fans & bible & machine & fax\\
space & league & believe & port & pay\\
sun & month & who & system & mail\\\hline
\end{tabular}
}\\\vspace{-.3cm}
\subfloat[{\sf Batch}]{
\begin{tabular}{|c|c|c|c|c|}\hline
Space & Sport & Religion & Hardware & Sale \\\hline
space & game & god & drive & sale \\
gov & games & state  & hard & price \\
dod & team  & who & system  & offer \\
nasa & baseball & jesus & scsi & shipping \\
earth & play & believe & use & sell \\
washington &  hockey & christian & server & mail \\
sun & season & religion & drives & condition \\
look & show & bible & port & interested \\\hline
\end{tabular}
}\\\vspace{-.3cm}
\end{table}

\begin{table}[t!]\footnotesize\centering
\caption{Average running times (in seconds) of the {\sf OL} and the {\sf Batch} algorithms on the {\tt BBCNews} and the {\tt 20NewsGroups} datasets (with standard deviations shown in the parentheses).} \label{table:running_time_text}
\begin{tabular}{|c|c|c|}\hline
 & {\tt BBCNews} & {\tt 20NewsGroups}\\\hline
 {\sf OL} & 43.76 (3.42) & 71.63 (5.38) \\\hline
 {\sf Batch} &  312.58 (10.84) & 306.76 (12.49) \\\hline
\end{tabular}\vspace{-.35cm}
\end{table}

\begin{figure}[t!] 
\subfloat[IS]{\includegraphics[width=.475\columnwidth]{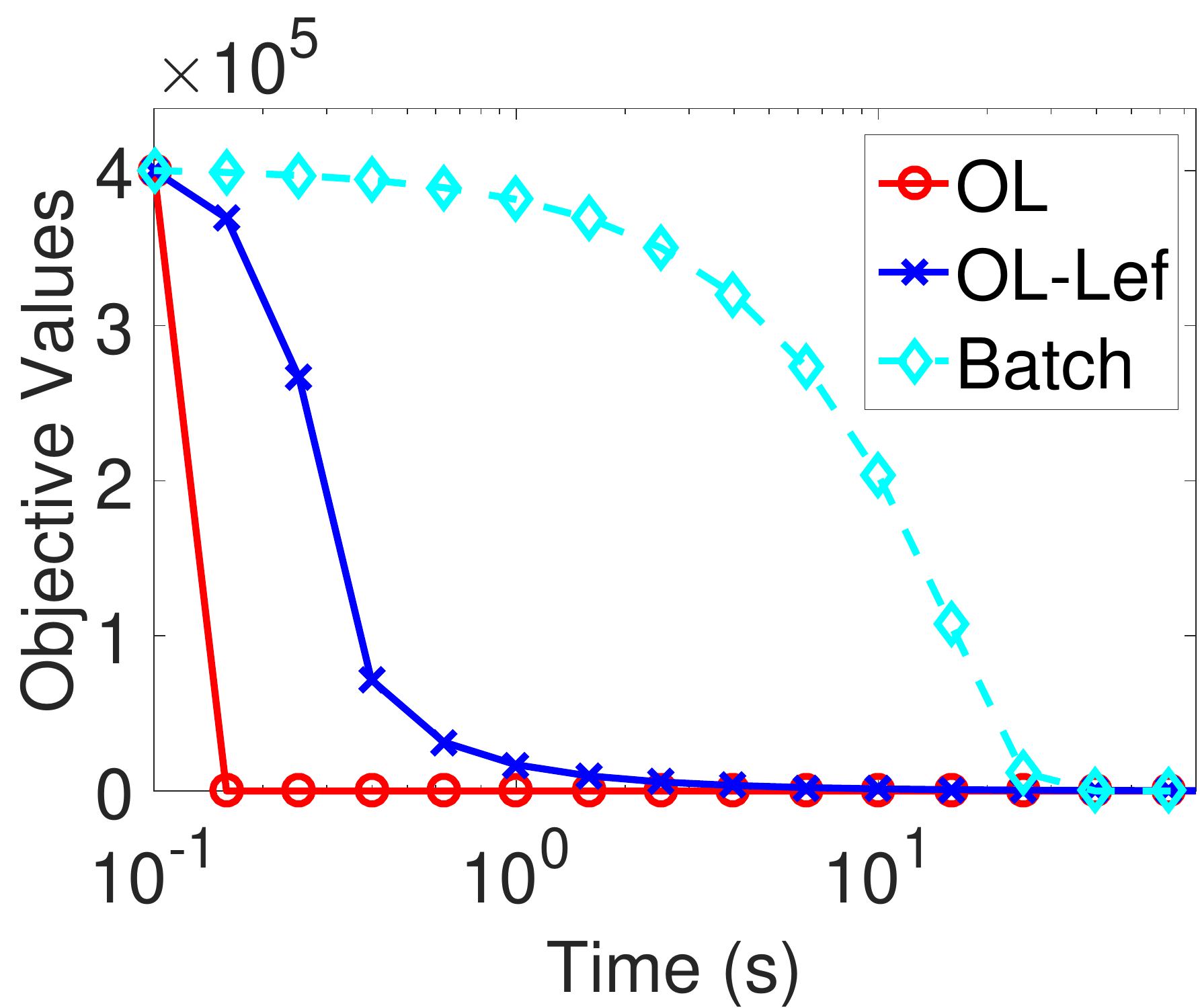}}\hfill
\subfloat[KL]{\includegraphics[width=.475\columnwidth,height=.36\columnwidth]{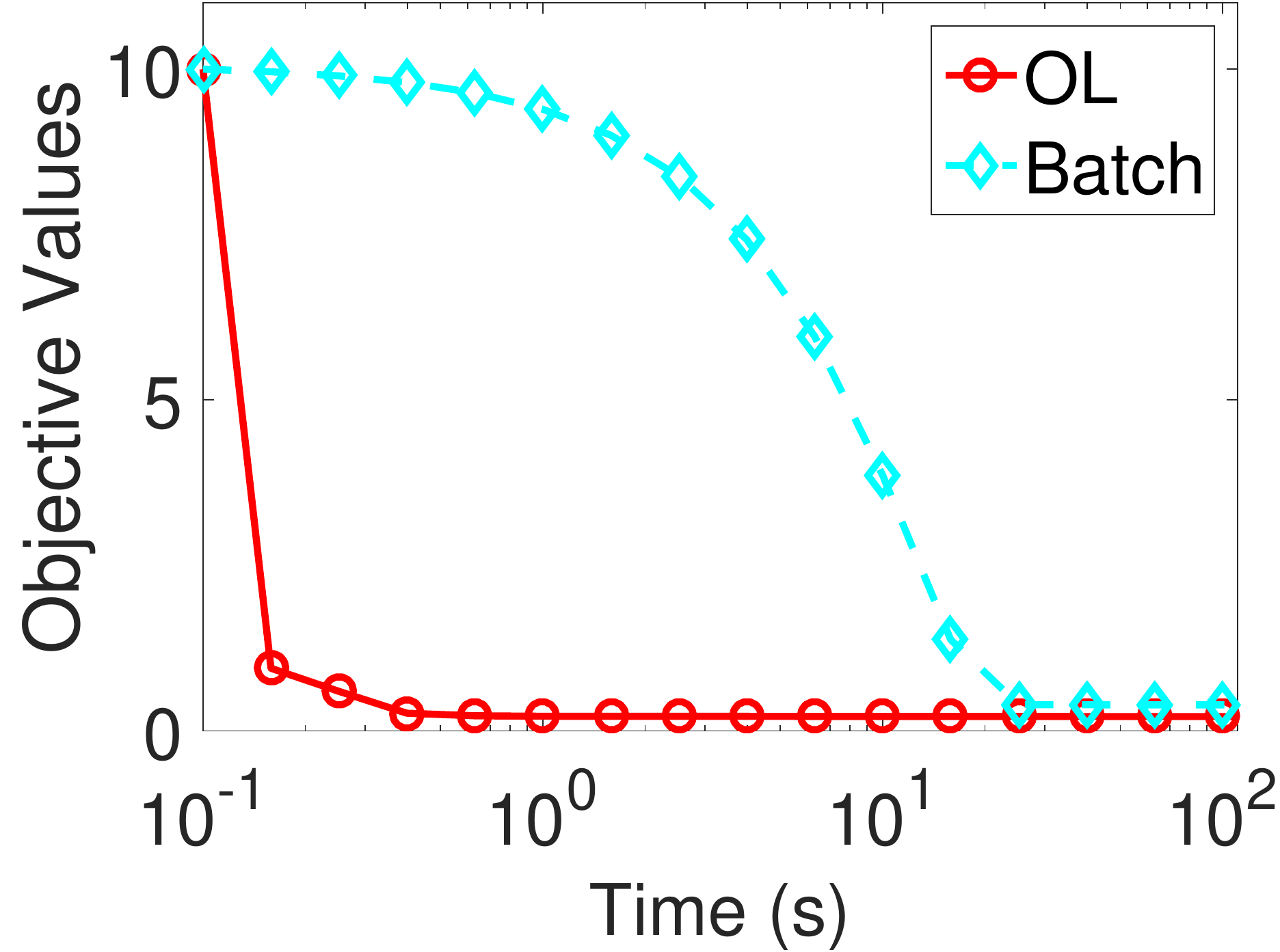}}\vspace{-.4cm}\\
\subfloat[Squared-$\ell_2$]{\includegraphics[width=.475\columnwidth]{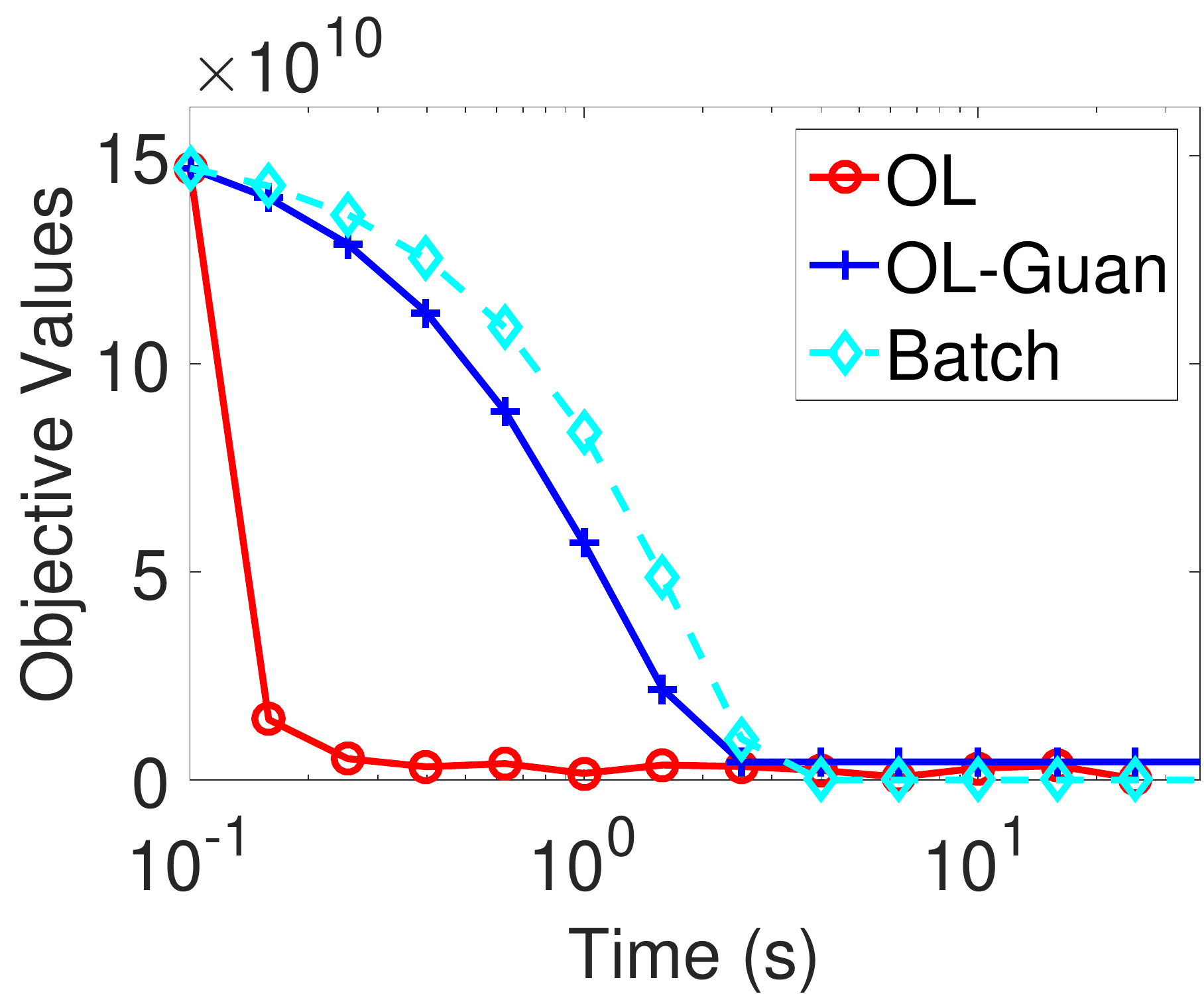}}\hfill
\subfloat[Huber]{\includegraphics[width=.475\columnwidth]{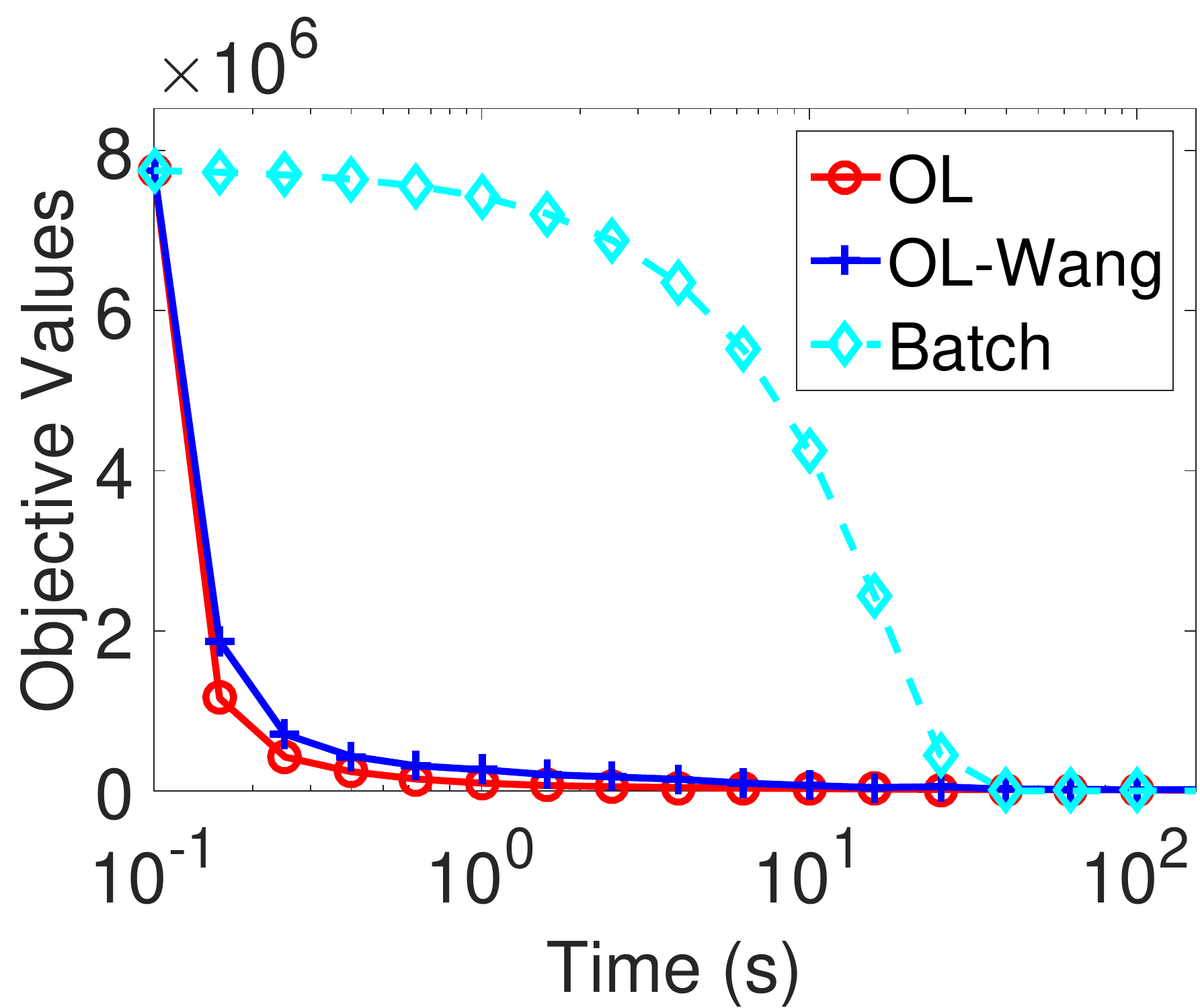}}\vspace{-.4cm}\\
\subfloat[$\ell_1$]{\includegraphics[width=.475\columnwidth]{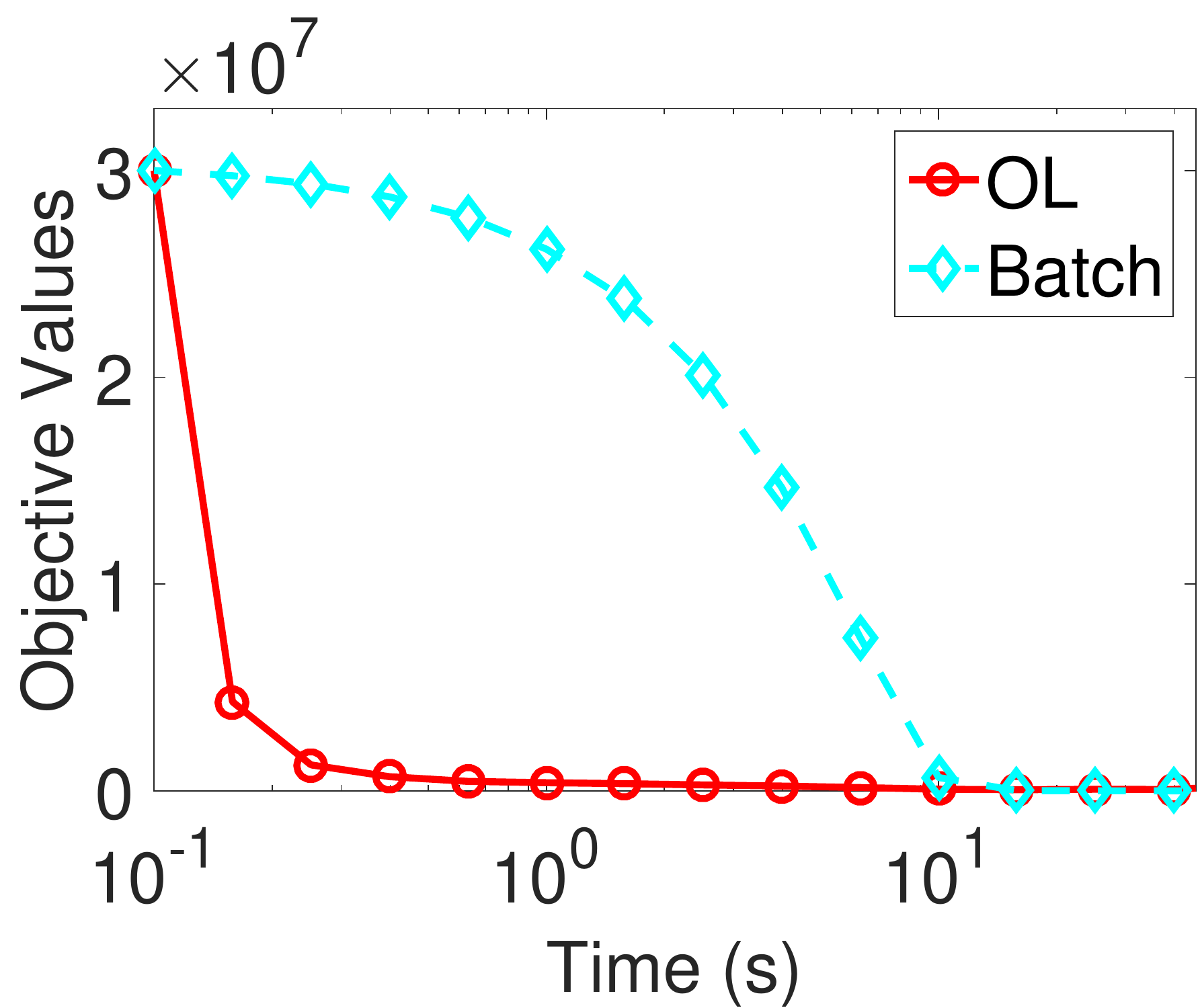}}\hfill
\subfloat[$\ell_2$]{\includegraphics[width=.475\columnwidth,height=.395\columnwidth]{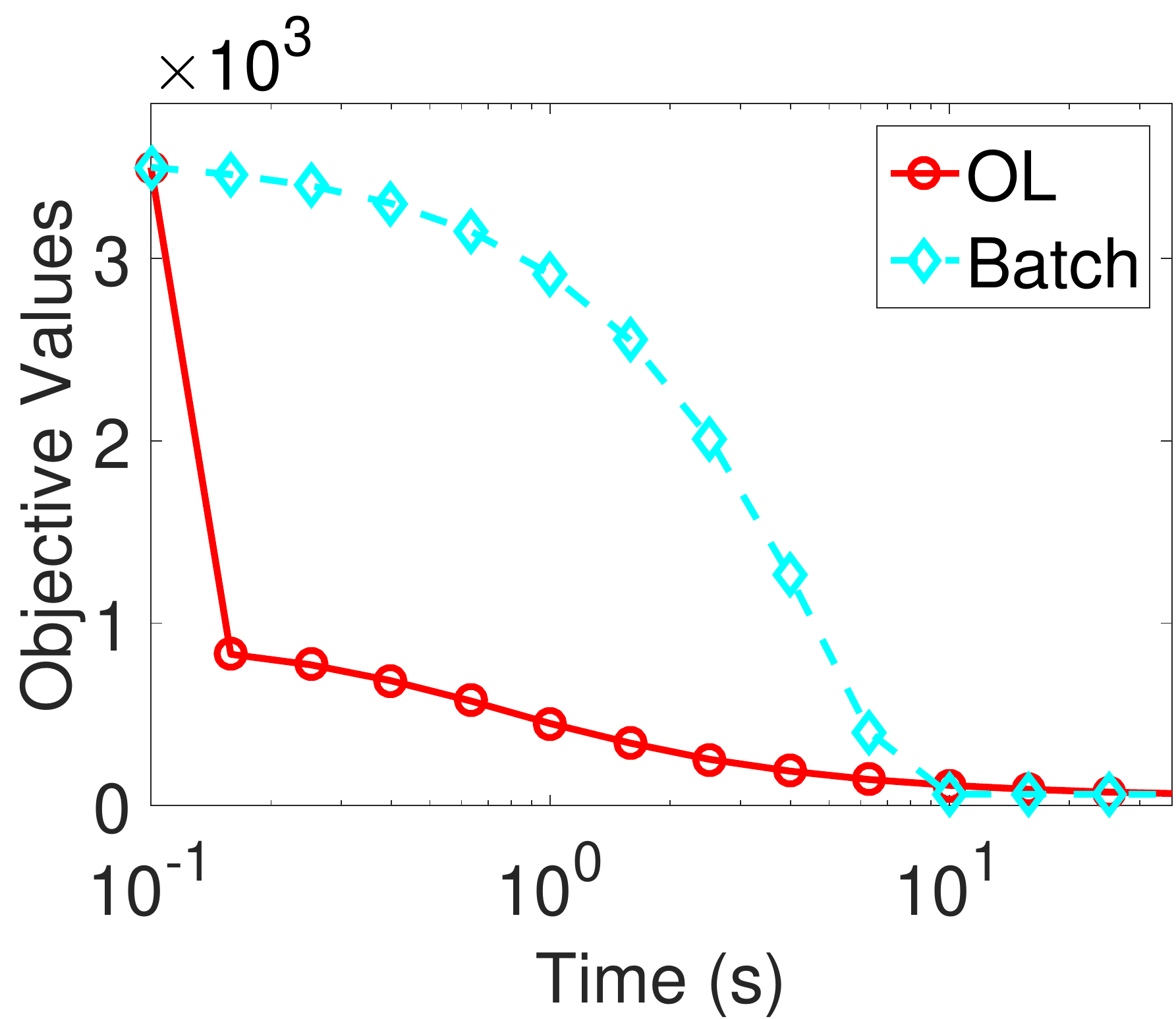}}\vspace{-.4cm}\\
\caption{Objective values versus time (in seconds) of online and batch NMF algorithms for all the divergences in $\barcalD$.}\vspace{-.5cm} \label{fig:comparison}
\end{figure}

\begin{figure}[t!] 
\subfloat[Subject A]{\includegraphics[width=.49\columnwidth]{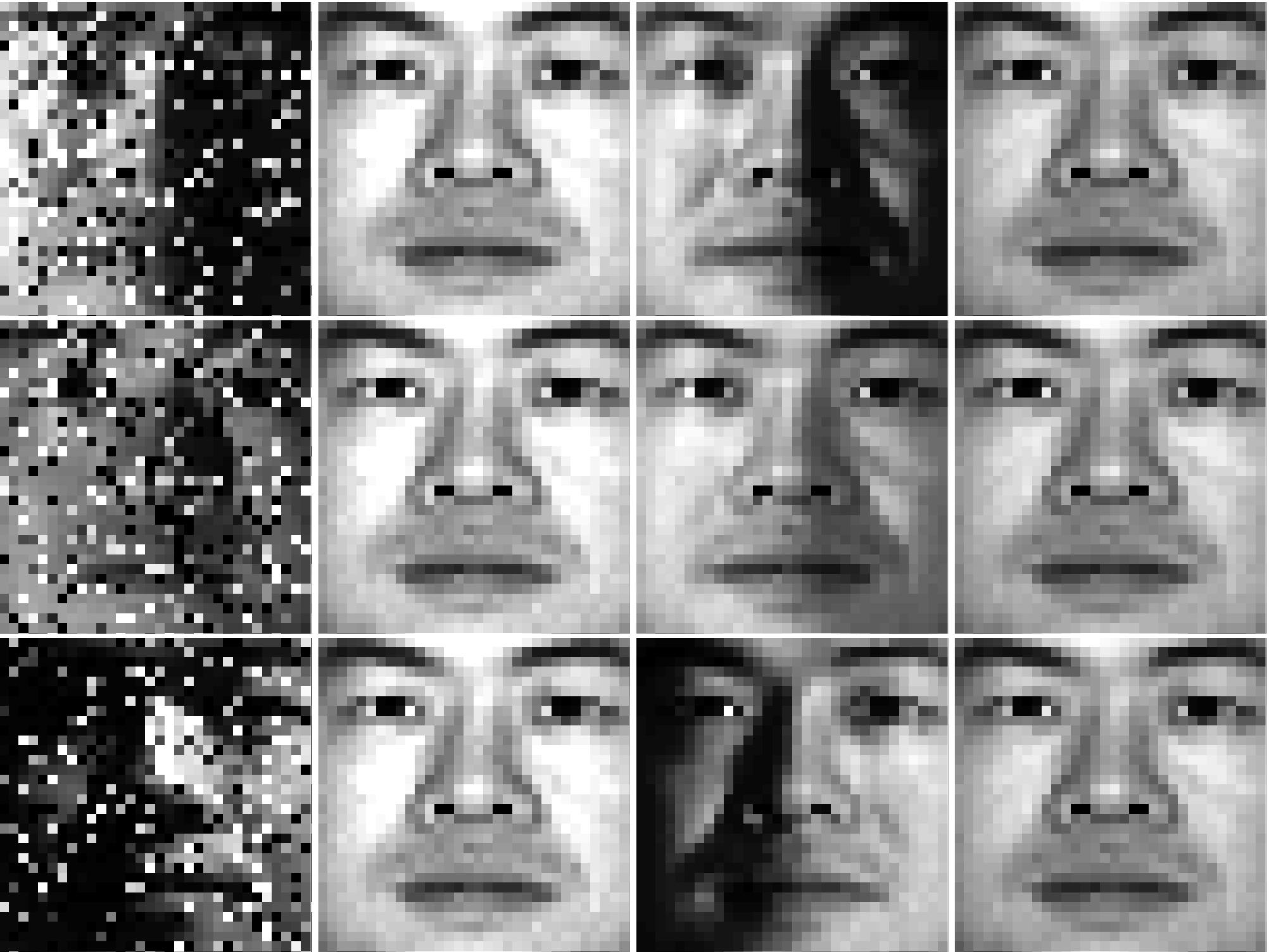}}\hfill
\subfloat[Subject B]{\includegraphics[width=.49\columnwidth]{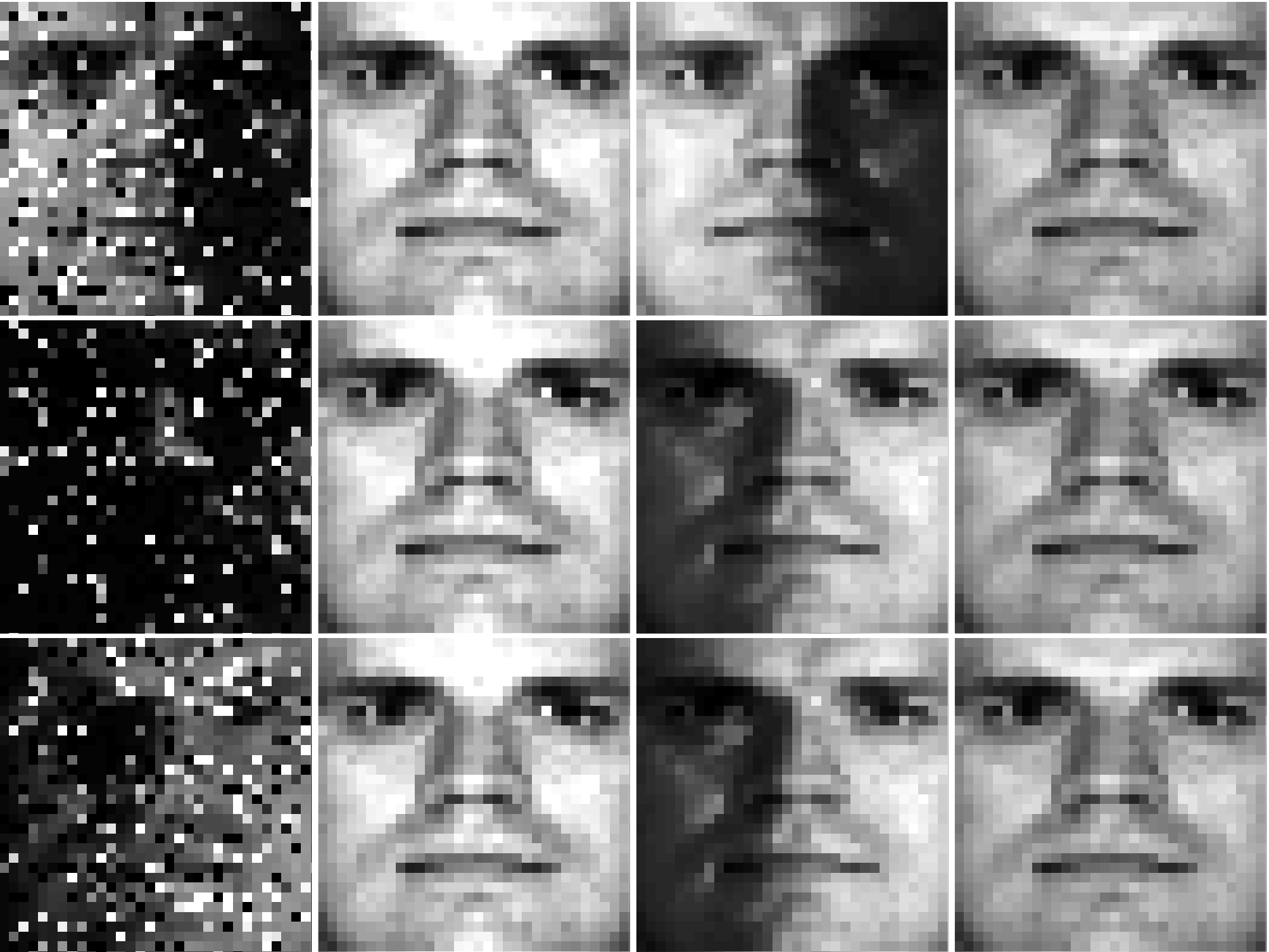}}\vspace{-.35cm}\\
\subfloat[Subject C]{\includegraphics[width=.49\columnwidth]{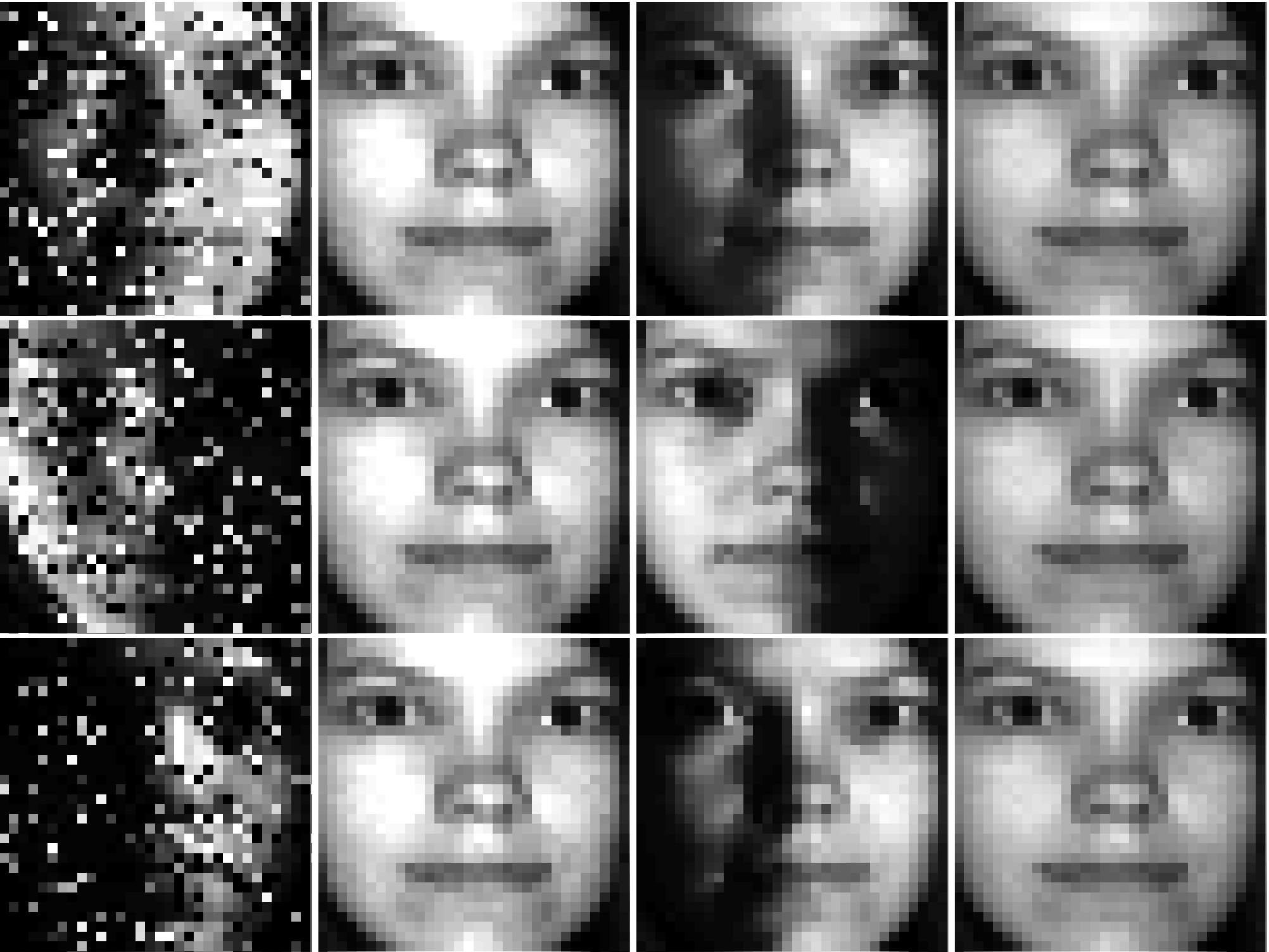}}\hfill
\subfloat[Subject D]{\includegraphics[width=.49\columnwidth]{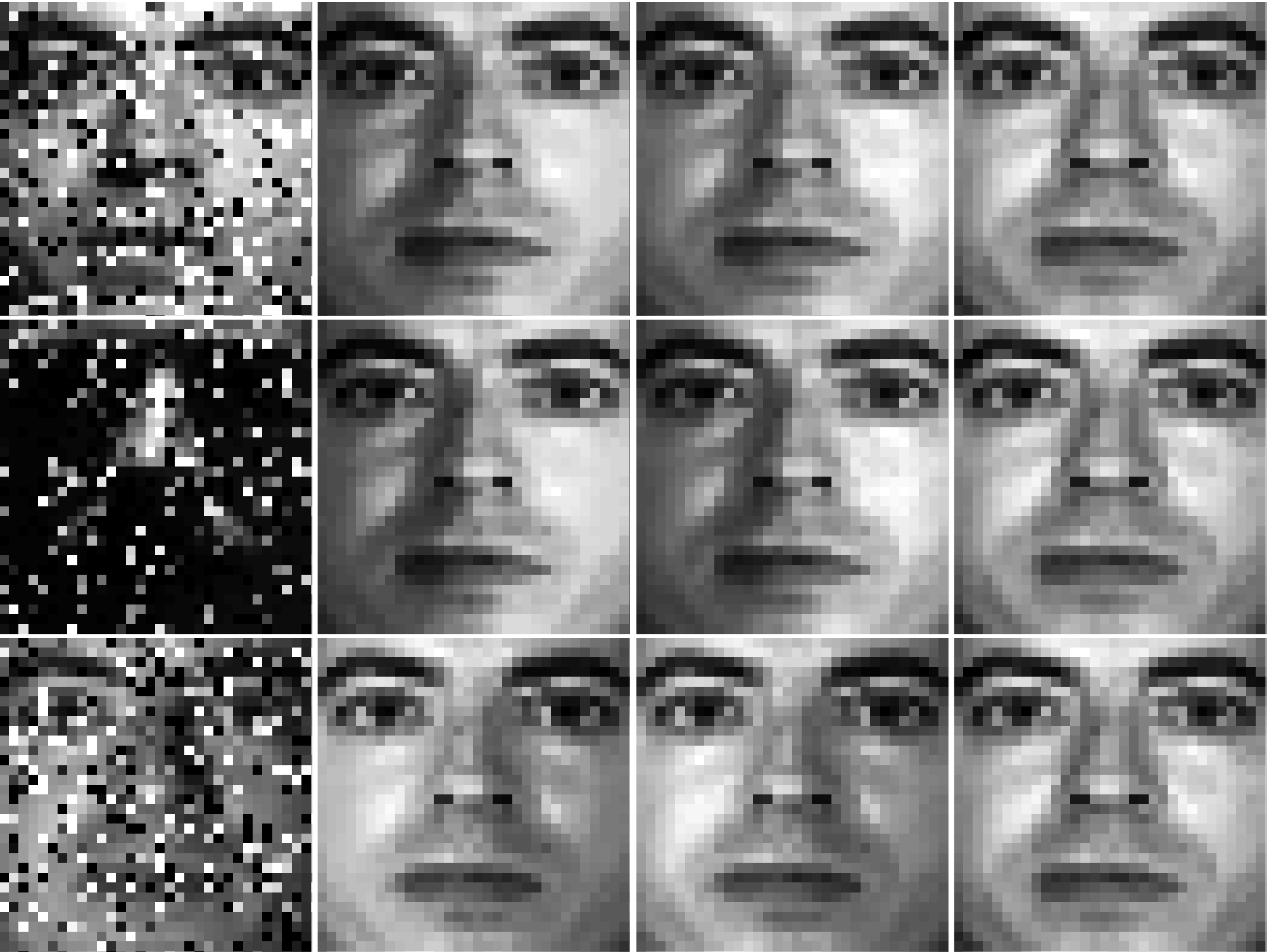}}\vspace{-.4cm}\\
\caption{Contaminated and reconstructed images for four subjects from YaleB face dataset. For each subject, from the left, the first column shows three images with salt-and-pepper noise and different illumination conditions. The second to the fourth column show the reconstructed images by the algorithms {\sf OL}, {\sf OL-Wang} and {\sf Batch} respectively.}\vspace{-.5cm}\label{fig:image_recons}
\end{figure}

\newpage
\begin{table}[t]\footnotesize
\centering
\caption{Average running times (in seconds) of the online and batch NMF algorithms with the Huber loss on the {\tt YaleB} dataset (with standard deviations shown in the parentheses).} \label{fig:runnning_time_YaleB}
\setlength\tabcolsep{2pt}
\begin{tabular}{|c|c|c|c|c|}\hline
Algorithms& Subject A & Subject B & Subject C & Subject D \\\hline
{\sf OL} & $67.86$ (4.59) & $64.35$ (5.39) & $76.45$ (4.58) & $113.09$ (6.68)\\\hline
{\sf OL-Wang} & $73.61$ (5.91) & $88.22$ (4.31) & $96.11$ (8.75)& $152.38$ (9.54)\\\hline
{\sf Batch} & $666.68$ (9.53) & $512.49$ (10.75)& $1270.45$ (12.96) & $755.69$ (7.98)\\\hline
\end{tabular}\vspace{-.5cm}
\end{table}



\newpage\onecolumn
\section*{\huge Supplemental Material for ``Online Nonnegative Matrix Factorization with Outliers''}

\renewcommand{\thedefinition}{S-\arabic{definition}}
\renewcommand{\thelemma}{S-\arabic{lemma}}
\renewcommand{\thecorollary}{S-\arabic{corollary}}
\renewcommand{\theequation}{S-\arabic{equation}}
\renewcommand{\thesection}{S-\arabic{section}}
\renewcommand{\theremark}{S-\arabic{remark}}
\renewcommand{\thefigure}{S-\arabic{figure}}

\setcounter{section}{0}
\setcounter{definition}{0}
\setcounter{lemma}{0}
\setcounter{corollary}{0}
\setcounter{equation}{0}
\setcounter{remark}{0}
\setcounter{figure}{0}

\section{Implementations of $\Pi_\calC$ and $\Pi_\calH$} \label{sec:proj_CH}
The projection operator $\Pi_\calC$ in \eqref{eq:upd_dict} can be implemented in a  straightforward manner if the data point lies in $\calC'\defeq\{\bW\in\bbR_+^{F\times K}\,|\,\norm{\bW_{i:}}_1\ge\epsilon, \forall\,i\in[F]\}$. Otherwise, if there exists $i\in[F]$ such that $\norm{\bW_{i:}}_1<\epsilon$, $\Pi_\calC$ amounts to projecting $\bW_{i:}$ onto the probability simplex in $\bbR_+^{K}$. Efficient algorithms have  been extensively discussed in the literature, for e.g.,  \cite[Section~3]{Duchi_08}. Since 
$\epsilon< 1$, the constraint $w_{ij}\le 1$, for any $j\in[K]$ is automatically satisfied after such projection. The projection onto the set $\calH$ simply involves entrywise thresholding. 

\section{Proof of Lemma~\ref{lem:conv_limit_set}}\label{sec:proof_conv_limit_set}
First we rewrite \eqref{eq:PDS_W0} as
\begin{equation}
\frac{d}{ds}W(s) = -\nabla f(W(s)) + z(s), \;W(0) = \bW_0, \;s\ge 0, \label{eq:PDS2_W0}
\end{equation}
where 
\begin{equation}
z(s)\defeq \pi_\calC\Big[W(s),-\nabla f(W(s))\Big] + \nabla f(W(s)), s\in\bbR_+. \label{eq:def_zs}
\end{equation}
From Lemma~\ref{lem:asym_equi} and Lemma~\ref{lem:finite_sum_equicont}, there exists an almost sure set $\calA\in\Omega$ such that for each $\omega\in\calA$, $\{W^t(\omega,\cdot)\}_{t\in\bbN}$ and $\{Z^t(\omega,\cdot)\}_{t\in\bbN}$ are asymptotically equicontinuous on $\bbR_+$. Due to the compactness of $\calC$, $\{W^t(\omega,\cdot)\}_{t\in\bbN}$ and $\{Z^t(\omega,\cdot)\}_{t\in\bbN}$ are also uniformly bounded. Fix $S\in(0,\infty)$. By the (generalized) Arzel\`{a}-Ascoli Theorem (see Lemma~\ref{lem:arzela}), there exists a sequence $\{t_k\}_{k\in\bbN}$ such that $t_k\uparrow\infty$, a continuous $\barW(\omega,\cdot)$ and a continuous $\barZ(\omega,\cdot)$ such that $W^{t_k}(\omega,\cdot)\convu \barW(\omega,\cdot)$ and $Z^{t_k}(\omega,\cdot)\convu \barZ(\omega,\cdot)$ on $[0,S]$. (Note that $\{t_k\}_{k\in\bbN}$, $\barW(\omega,\cdot)$ and $\barZ(\omega,\cdot)$ may depend on $S$.) 
Define 
\begin{equation}
\barG(s)\defeq-\int_0^s \nabla f\left(\barW(\tau)\right)\,d\tau, \;s\in[0,S].
\end{equation}
We now show $G^{t_k}(\omega,\cdot)\convu \barG(\omega,\cdot)$ on $[0,S]$. By Lemma~\ref{lem:cont_uniconv} and continuity of $\nabla f$, we have $\nabla f(W^{t_k}\left(\omega,\cdot)\right) \convu \nabla f\left(\barW(\omega,\cdot)\right)$ on $[0,S]$. Thus
\begin{align*}
\lim_{t\to\infty} \sup_{s\in[0,S]} \norm{G^{t_k}(\omega,s) - \barG(\omega,s)} &\le \lim_{t\to\infty} \sup_{s\in[0,S]}  \int_0^{s}\norm{\nabla f(\barW(\omega,\tau))-\nabla f(W^{t_k}(\omega,\tau))} \,d\tau\\
&\le  S \lim_{t\to\infty}\sup_{s\in[0,S]} \norm{\nabla f(\barW(\omega,s))-\nabla f(W^{t_k}(\omega,s))}\\
&= 0. 
\end{align*}
From Lemma~\ref{lem:asym_equi}, we also have $\Delta_1^{t_k}(\omega,\cdot)\convu \vecz$ and $N^{t_k}(\omega,\cdot)\convu \vecz$ on $[0,S]$. Therefore, from \eqref{eq:def_Wt}, we have
\begin{equation}
\barW(\omega,s) = \barW(\omega,0) - \int_0^s \nabla f\left(\barW(\omega,\tau)\right)\,d\tau + \barZ(\omega,s), \;s\in[0,S]. \label{eq:mean_PODE}
\end{equation}
Thus, to show $\left\{\barW(\omega,\cdot),\barZ(\omega,\cdot)\right\}$ satisfies (the integral form) of \eqref{eq:PDS2_W0} (on $[0,S]$), it remains to show $\barZ(\omega,s) = \int_0^s z(\tau)\;d\tau$, $s\in[0,S]$. 
By the definition of $\{Z^t(\omega,\cdot)\}_{t\in\bbN}$, we have $\barZ(\omega,0)=0$. Also, by the closedness of $\calC$, we have $\barW(\omega,s)\in\calC$, for all $s\ge 0$. First we define the inward normal set at $\bW\in\calC$, $\calN(\bW)$ as
\begin{equation}
\calN(\bW) \defeq \begin{cases}
\left\{\bN\in\bbR^{F\times K}\,|\,\norm{\bN}\le M,\,\lrangle{\bN}{\bW'-\bW}\ge 0, \,\forall\,\bW'\in\calC\right\}, &\,\bW\in\bdr\calC\\
\left\{\vecz\in\bbR^{F\times K}\right\}, &\,\bW\in\inter\calC
\end{cases}. \label{eq:def_normal}
\end{equation}
From \eqref{eq:def_normal}, we notice that $\calN(\bW)$ is compact and convex for any $\bW\in\calC$. 
By the definition of $Z^t(\omega,\cdot)$, it is also obvious that for any $t\in\bbN$ and $s\ge 0$, $Z^t(\omega,s)\in\calN(W^t(\omega,s))$. 

By Lemma~\ref{lem:suff_pode}, it suffices to show $\barW(\omega,\cdot)$ is Lipschitz on $[0,S]$ and for any $\tau\in[0,S]$
\begin{enumerate}
\item $\barZ(\omega,\tau)=\vecz$ if $\barW(\omega,s)\in\inter\calC$ for almost all $s\in[0,\tau]$ (in the sense of Lebesgue measure),
\item $\barZ(\omega,\tau)\in \convcl \left[\bigcup_{s\in[0,\tau]}\calN\left(\barW(\omega,s)\right)\right]$.
\end{enumerate}
First, we show $\barZ(\omega,\cdot)$ is Lipschitz on $[0,S]$. By Lemma~\ref{lem:asym_equi}, we have for any $s_0,s_1\in[0,S]$,
\begin{align*}
\norm{\barZ(\omega,s_0)-\barZ(\omega,s_1)} &= \lim_{k\to\infty} \norm{Z^{t_k}(\omega,s_0)-Z^{t_k}(\omega,s_1)}\\
&\le \lim_{k\to\infty} \norm{Y^{t_k}(\omega,s_0)-Y^{t_k}(\omega,s_1)} + \norm{\Delta_2^{t_k}(\omega,s_0)-\Delta_2^{t_k}(\omega,s_1)}\\
&\le \lim_{k\to\infty} \norm{\int_{s_0}^{s_1} Z^{t_k}(\omega,\tau)\,d\tau} + 2\sup_{s\in\bbR_+} \norm{\Delta_2^{t_k}(\omega,s)}\\
&\le \lim_{k\to\infty} \abs{s_0-s_1} \sup_{\tau\in[s_0,s_1]} \norm{Z^{t_k}(\omega,\tau)}\\
&\le M\abs{s_0-s_1}.
\end{align*}
Since $\nabla f\left(\barW(\omega,\cdot)\right)$ is bounded on $[0,S]$, by \eqref{eq:mean_PODE}, we conclude $\barW(\omega,\cdot)$ is Lipschitz on $[0,S]$. 
Next, since $\barW(\omega,s)\in\inter\calC$ for almost all $s\in[0,\tau]$, there exists $\{s_n\}_{n\in\bbN}$ in $[0,\tau]$ such that $s_n\uparrow\tau$ and $\barW(\omega,s_n)\in\inter\calC$ for all $n\in\bbN$. Hence $\barZ(\omega,s_n)=\vecz$ for all $n\in\bbN$. The continuity of $\barZ(\omega,\cdot)$ implies $\barZ(\omega,\tau)=\vecz$. To show the last claim, we leverage the upper semicontinuity of the correspondence $\calN$ (see Definition~\ref{def:correspondence}). We first show $\calN$ is upper semicontinuous on $\calC$ by Lemma~\ref{lem:suff_usc}. It suffices to show 
\begin{equation}
\bigcap_{\delta>0}\convcl\left(\bigcup_{\bW'\in\calB_\delta(\bW)}\calN(\bW')\right) \subseteq \calN(\bW),\,\forall\,\bW\in\calC, \label{eq:suff_usc}
\end{equation}
where $\calB_\delta(\bW)\defeq \{\bW'\in\calC\,|\,\norm{\bW-\bW'}<\delta\}$. Suppose \eqref{eq:suff_usc} is false, then for any $\delta>0$, there exists $\bW_0\in\calC$ and $\bN_0$ such that $\bN_0\in\convcl\left(\bigcup_{\bW'\in\calB_\delta(\bW_0)}\calN(\bW')\right)$ and $\bN_0\not\in\calN(\bW_0)$. For any $\epsilon>0$, there exists $\bN'\in\calB_\epsilon(\bN_0)$, $\lambda\in[0,1]$ and $\bW_1,\bW_2\in\calB_\delta(\bW_0)$ such that $\bN'=\lambda\bN_1+(1-\lambda)\bN_2$, where $\bN_i\in\calN(\bW_i)$, $i=1,2$. Hence for any $\bW'\in\calC$,
\begin{align*}
\lrangle{\bN_0}{\bW'-\bW_0} &= \lrangle{\bN'}{\bW'-\bW_0} + \lrangle{\bN_0-\bN'}{\bW'-\bW_0}\\
&= \lambda\lrangle{\bN_1}{\bW'-\bW_0} + (1-\lambda)\lrangle{\bN_2}{\bW'-\bW_0}+ \lrangle{\bN_0-\bN'}{\bW'-\bW_0}\\
&=\lambda\lrangle{\bN_1}{\bW'-\bW_1} + \lambda\lrangle{\bN_1}{\bW_1-\bW_0} + (1-\lambda)\lrangle{\bN_2}{\bW'-\bW_2}  \\
&\hspace{4cm}+(1-\lambda)\lrangle{\bN_2}{\bW_2-\bW_0}+ \lrangle{\bN_0-\bN'}{\bW'-\bW_0}\\
&\ge -\lambda\norm{\bN_1}\norm{\bW_1-\bW_0} - (1-\lambda)\norm{\bN_2}\norm{\bW_2-\bW_0} - \norm{\bN_0-\bN'}\norm{\bW'-\bW_0}\\
&\ge -\lambda\delta\norm{\bN_1} - (1-\lambda)\delta\norm{\bN_2} - \epsilon\norm{\bW'-\bW_0}\\
&\ge -(\delta M+\epsilon\diam\calC),
\end{align*}
where $\diam\calC\defeq \max_{\bX,\bY\in\calC}\norm{\bX-\bY}$.  The compactness of $\calC$ implies $\diam\calC<\infty$. 
Let both $\delta\to 0$ and $\epsilon\to 0$ we have $\lrangle{\bN_0}{\bW'-\bW_0}\ge 0$, for any $\bW'\in\calC$.  This contradicts  $\bN_0\not\in\calN(\bW_0)$. Thus we conclude that $\calN$ is upper semicontinuous on $\calC$. Since $\calN$ is compact-valued, $\calG(\calN)$ is closed by Lemma~\ref{lem:closed_graph}. Again, take a sequence $\{s_n\}_{n\in\bbN}$ in $[0,\tau]$ such that $s_n\uparrow\tau$. For any $n\in\bbN$, since $Z^{t_k}(\omega,s_n)\to\barZ(\omega,s_n)$, $Z^{t_k}(\omega,s_n)\in\calN(W^{t_k}(\omega,s_n))$ and $W^{t_k}(\omega,s_n)\to\barW(\omega,s_n)$, by the closedness of $\calG(\calN)$, we have $\barZ(\omega,s_n)\in\calN(\barW(\omega,s_n))$. 
Since $\barZ(\omega,s_n)\to\barZ(\omega,\tau)$, $\barW(\omega,s_n)\to\barW(\omega,\tau)$, we have $\barZ(\omega,\tau)\in\calN(\barW(\omega,\tau))\subseteq \convcl \left[\bigcup_{s\in[0,\tau]}\calN\left(\barW(\omega,s)\right)\right]$. 

Now, fix a sequence $\{S_n\}_{n\in\bbN}\subseteq(0,\infty)$ such that $S_n\uparrow\infty$,  and let $\left\{\barW_n(\omega,\cdot),\barZ_n(\omega,\cdot)\right\}$ be the (continuous) limit functions corresponding to $S_n$. Fix $i\in\bbN$. For any $S_i<S_j$, there exist $\left\{\barW_j(\omega,\cdot),\barZ_k(\omega,\cdot)\right\}$ such that $\barW_i(\omega,\cdot) = \barW_j(\omega,\cdot)$ and $\barZ_i(\omega,\cdot) = \barZ_j(\omega,\cdot)$ on $[0,S_i]$.  Thus there exists $\{\bart_k\}_{k\in\bbN}$ such that $W^{\bart_k}(\omega,\cdot)\convu \barW_\infty(\omega,\cdot)$ and $Z^{\bart_k}(\omega,\cdot)\convu \barZ_\infty(\omega,\cdot)$ on $\bbR_+$. Moreover, the continuous limit functions $\left\{\barW_\infty(\omega,\cdot),\barZ_\infty(\omega,\cdot)\right\}$ satisfy \eqref{eq:PDS2_W0} on $\bbR_+$. 

The above implies the solution set of \eqref{eq:PDS2_W0} is nonempty. Moreover, the compactness of $\calC$  implies the limit set of \eqref{eq:PDS2_W0}, $\calL(-\nabla f,\calC,\bW_0)\neq\emptyset$. For any convergent subsequence $\{\bW_{t_l}(\omega)\}_{l\in\bbN}$, there exist a non-decreasing sequence $\{t'_l\}_{l\in\bbN}\subseteq\{\bart_k\}_{k\in\bbN}$  with $t'_l\uparrow\infty$ and $\{\tau_l\}_{l\in\bbN}\uparrow\infty$ such that ${t_l} = m(\tau_l + s_{t'_l})$, for all $l\in\bbN$. 
Therefore, 
\begin{align*}
\lim_{l\to\infty} \dist\Big(\bW_{t_l}(\omega),\calL(-\nabla f,\calC,\bW_0)\Big) &= \lim_{l\to\infty} \inf_{\bW\in\calL(-\nabla f,\calC,\bW_0)} \norm{W^{t'_l}(\omega,\tau_l)-\bW} \\
&\le \lim_{l\to\infty} \inf_{\bW\in\calL(-\nabla f,\calC,\bW_0)} \norm{W^{t'_l}(\omega,\tau_l)-\barW_\infty(\omega,\tau_l)}+\norm{\barW_\infty(\omega,\tau_l)-\bW}\\
&\le \lim_{l\to\infty}\sup_{s\ge 0}\norm{W^{t'_l}(\omega,s)-\barW_\infty(\omega,s)} + \lim_{l\to\infty}\inf_{\bW\in\calL(-\nabla f,\calC,\bW_0)}\norm{\barW_\infty(\omega,\tau_l)-\bW}\\
&\eqa \lim_{l\to\infty}\dist\Big(\barW_\infty(\omega,\tau_l),\calL(-\nabla f,\calC,\bW_0)\Big)\\
&\eqb 0,
\end{align*}
where in (a) we use the fact that $W^{\bart_k}(\omega,\cdot)\convu\barW_\infty(\omega,\cdot)$ on $\bbR_+$
and in (b) we use the definition of $\calL(-\nabla f,\calC,\bW_0)$. 
Thus we conclude that $\lim_{l\to\infty}\bW_{t_l}(\omega)\in\cl\calL(-\nabla f,\calC,\bW_0)$. 
Hence we prove $\bW_t(\omega)\to\calL(-\nabla f,\calC,\bW_0)$ as $t\to\infty$.

\section{Proof of Lemma~\ref{lem:char_limit_set}}\label{sec:proof_char_limit_set}
We leverage the Lyapunov stability theory {\cite[Section~6.6]{Teschl_12}} to prove the lemma. First, define $L:\calC\to\bbR$ such that $L(\bW) \defeq f(\bW) - \min_{\bW\in\calC}f(\bW)$, $\bW\in\calC$. By Definition~\ref{def:lya_func}, we have that $L$ is a Lyapunov function (with possibly non-unique zeros on $\calC$). By \cite[Theorem 6.15]{Teschl_12} (see Lemma~\ref{lem:limit_st}),  
\begin{align*}
&\calL(-\nabla f,\calC,\bW_0)\subseteq\bigcup_{W(\cdot)\in\calP(-\nabla f,\calC,\bW_0)}\Big\{W(s)\,\Big|\,\frac{d}{ds}L(W(s))=0\Big\},
\end{align*}
where 
\begin{align*}
\frac{d}{ds}L(W(s)) = \lrangle{\nabla f(W(s))}{\pi_\calC\Big[W(s),-\nabla f(W(s))\Big]},\,s\ge 0. 
\end{align*}
Given $\lrangle{\nabla f(\bW)}{\pi_\calC[\bW,-\nabla f(\bW)]}=0$,  it is obvious that $\pi_\calC[W(s),-\nabla f(W(s))] = \vecz$, if there exists $\delta>0$ such that $\bW-\delta\nabla f(\bW)\in\calC$. Otherwise, by the convexity  of $\calC$, 
\begin{align*}
&\quad\;\dist^2(\bW-\nabla f(\bW),\calC) \; \\
&\ge \norm{\pi_\calC\left[\bW,-\nabla f(\bW)\right]+\nabla f(\bW)}^2\\
&= \norm{\pi_\calC\left[\bW,-\nabla f(\bW)\right]}^2 + \norm{\nabla f(\bW)}^2\\
&= \norm{\pi_\calC\left[\bW,-\nabla f(\bW)\right]}^2 + \norm{(\bW-\nabla f(\bW))-\bW}^2\\
&\ge \norm{\pi_\calC\left[\bW,-\nabla f(\bW)\right]}^2 + \dist^2(\bW-\nabla f(\bW),\calC),
\end{align*}
where for any $\bY\in\bbR^{F\times K}$, $\dist(\bY,\calC)\defeq \norm{\Pi_\calC\bY-\bY}$. 
Hence we conclude $\pi_\calC\left[\bW,-\nabla f(\bW)\right]=\vecz$. Thus we conclude $\calL(-\nabla f,\calC,\bW_0)\subseteq \calS(-\nabla f,\calC)$.

We show the second claim in a similar way. Before we proceed, let us first define a supporting hyperplane (see \cite[Section~2.5.2]{Boyd_04}) at any $\bW\in\bdr\calC$, $\calT_\bW$ as\footnote{Note that more  than one supporting hyperplanes may exist at $\bW\in\bdr\calC$. The supporting hyperplane that $\calT_\bW$ refers to depends on the context. }
\begin{equation}
\calT_\bW \defeq \{\bW'\in\bbR^{F\times K}\,|\,\lrangle{\bT}{\bW'-\bW}=0\},
\end{equation}
where the (outward) normal $\bT\in\bbR^{F\times K}$ of $\calT_\bW$ satisfies $\lrangle{\bT}{\bW'-\bW}\le 0$, for all $\bW'\in\calC$. 
Given $\pi_\calC\left[\bW,-\nabla f(\bW)\right]=\vecz$, we only focus on the case where $\bW\in\bdr\calC$ and for any $\delta>0$, $\bW-\delta\nabla f(\bW)\not\in\calC$, otherwise the claim trivially holds. By the definition of  $\pi_\calC\left[\bW,-\nabla f(\bW)\right]$ and convexity of $\calC$, there exists a supporting hyperplane $\calT_\bW$ such that 
\begin{equation}
\pi_\calC\left[\bW,-\nabla f(\bW)\right] = \Pi_{\calT_\bW}(\bW-\nabla f(\bW)) - \bW.
\end{equation}
Since $\pi_\calC\left[\bW,-\nabla f(\bW)\right]=\vecz$, we have $\Pi_{\calT_\bW}(\bW-\nabla f(\bW)) = \bW$. This implies that $-\nabla f(\bW)$ is the (outward) normal of $\calT_\bW$. The definition of $\calT_\bW$ implies \eqref{eq:variat}.

\section{Proof of Lemma~\ref{lem:asym_equi}}\label{sec:proof_asym_equi}
We first show $N^t\convu \vecz$ on $\bbR_+$ a.s.. Fix $t\in\bbN$. Since $\{\eta_t\bN_t\}_{t\in\bbN}$ is a martingale difference sequence (adapted to $\{\scF_t\}_{t\ge 0}$), $\left\{\bM_t \defeq \sum_{l=1}^{t} \eta_l\bN_l\right\}_{t\in\bbN}$ is a martingale. We shall prove $\{\bM_t\}_{t\in\bbN}$ converges a.s. to a random variable $\bM$. First, we see $\{\bM_t\}_{t\in\bbN}$ is square-integrable since
\begin{align*}
\sup_{t\in\bbN}\bbE\left[\norm{\bM_t}^2\right] &= \sup_{t\in\bbN}\bbE\left[\norm{\sum_{l=1}^t\eta_l\bN_l}^2\right]\\
&= \sup_{t\in\bbN}\sum_{l=1}^t\eta_l^2\bbE\left[\norm{\bN_l}^2\right] + \sum_{k\ne l} \eta_k\eta_l \bbE\left[\lrangle{\bN_k}{\bN_l}\right]\\
&\eqa \sup_{t\in\bbN}\sum_{l=1}^t\eta_l^2\bbE\left[\norm{\bN_l}^2\right]\\
&\leb {M''}^2 \sup_{t\in\bbN}\sum_{l=1}^t \eta_l^2 \;\;\\
&< \infty,
\end{align*}
where (a) follows the orthogonality of the martingale difference sequence and (b) follows from \eqref{eq:bounded_noise}. 
Moreover, by the continuities of $(\bv,\bW)\mapsto\nabla_\bW\ell(\bv,\bW)$ (on $\calV\times\calC$) and $\nabla f$ (on $\calC$) and compactness of $\calV$ and $\calC$, there exists a constant $C\in(0,\infty)$ such that $\sup_{t\in\bbN}\bbE\left[\norm{\bN_t}^2|\scF_{t-1}\right]\le C^2$ a.s.. Therefore,
\begin{align*}
\sum_{t=2}^\infty \bbE\left[\norm{\bM_t-\bM_{t-1}}^2|\scF_{t-1}\right]=&\;\sum_{t=2}^\infty \eta_t^2\;\bbE\left[\norm{\bN_t}^2|\scF_{t-1}\right]\\
\le&\; C^2\sum_{t=2}^\infty \eta_t^2 \; \\
<&\;\infty \;\;\mbox{a.s.}
\end{align*}
Thus by Lemma~\ref{lem:conv_sq_mart}, there exists a (finite) random variable $\bM$ such that $\bM_n\convas \bM$. Then there exists an almost sure set $\calA\in\Omega$ such that for all $\omega\in\calA$, 
\begin{align*}
\lim_{t\to\infty} \sup_{s\ge 0} \norm{N^t(\omega,s)} = &\lim_{t\to\infty} \sup_{s> 0} \norm{\sum_{i=t+1}^{m(s_t+s)} \eta_i\bN_i(\omega)}\\
 =&\; \lim_{t\to\infty} \sup_{j\ge {t+1}} \norm{\bM_{j}(\omega) - \bM_t(\omega)}\\
\le&\;  \lim_{t\to\infty} \sup_{j\ge {t+1}}  \norm{\bM_{j}(\omega) - \bM(\omega)} + \norm{\bM_{t}(\omega) - \bM(\omega)}\\
\le&\; 2\lim_{t\to\infty} \sup_{j\ge t} \norm{\bM_{j}(\omega) - \bM(\omega)}\\
=&\; 0.
\end{align*}
This implies $N^t\convu \vecz$ on $\bbR_+$ a.s.. Moreover, by Lemma~\ref{lem:uniconv_asymcont}, $\{N^t\}_{t\in\bbN}$ is asymptotically equicontinuous on $\bbR_+$ a.s..

We have $\Delta_1^t\convu \vecz$ on $\bbR_+$ a.s. because for all $\omega\in\calA$,
\begin{align*}
\lim_{t\to\infty} \sup_{s\ge 0} \norm{\Delta_1^t(\omega,s)} 
= &\; \lim_{t\to\infty} \sup_{s\ge 0} \norm{\int_0^s \nabla f(W^t(\omega
,\tau)) \,d\tau - \sum_{i=t}^{m(s_t+s)-1} \eta_{i+1}\nabla f(\bW_{i}(\omega))}\\
\le &\;  \lim_{t\to\infty} \sup_{j\ge t} \sup_{s'\in[s_{j},s_{j+1}]} \norm{\int_{s'}^{s_{j+1}}\nabla f(W^t(\omega,\tau)) \,d\tau }\\
\le &\; \lim_{t\to\infty} \sup_{j\ge t} \;\eta_{j+1} \norm{\nabla f(\bW_j(\omega))}\\
\le &\; M^{'} \limsup_{t\to\infty} \eta_t \\
= &\;0,
\end{align*}
where  the second last step follows from Corollary~\ref{cor:bounded}. 
By Lemma~\ref{lem:uniconv_asymcont}, $\{\Delta_1^t\}_{t\in\bbN}$ is asymptotically equicontinuous on $\bbR_+$ a.s..

By the definition of $G^t$ in \eqref{eq:def_tilG}, we observe for each $t\in\bbN$ and $\omega\in\calA$, $G^t(\omega,\cdot)$ is continuous on $\bbR_+$ and continuously differentiable on $\bbR_+\setminus\calQ$ with $\frac{d}{ds}{G}^{t}(\omega,s) = -\nabla f(W^t(\omega,s))$, $s\in\bbR_+\setminus\calQ$, where $\calQ\defeq\{s_t\}_{t\ge 0}$. 
By Corollary~\ref{cor:bounded}, we have $\sup_{t\in\bbN}\sup_{s\ge 0}\normt{\nabla f(W^t(\omega,s))}\le M'$. This implies each $G^t(\omega,\cdot)$ is Lipschitz with Lipschitz constant $L_t$ and $\{L_t\}_{t\in\bbN}$ is bounded. 
Then by Lemma~\ref{lem:Lips_eqcont}, we conclude that $\{G^t(\omega,\cdot)\}_{t\in\bbN}$ is equicontinuous on $\bbR_+$. Since for each $t\in\bbN$ and $\omega\in\calA$, $F^t(\omega,\cdot) = G^t(\omega,\cdot) + \Delta^t(\omega,\cdot)$, by Lemma~\ref{lem:finite_sum_equicont}, $\{F^t\}_{t\in\bbN}$ is asymptotically equicontinuous on $\bbR_+$ a.s..

Using a similar argument, we can show for all $\omega\in\calA$, $\Delta_2^t(\omega,\cdot)\convu \vecz$ on $\bbR_+$. By the  definition of $\bZ_t$ in \eqref{eq:def_Z}, we have for any $t\in\bbN$ and $\omega\in\calA$, $\normt{\bZ_t(\omega)} \le \normt{\nabla_\bW\ell(\bv_t(\omega),\bW_{t-1}(\omega))} \le M$. 
Hence   each $Y^t(\omega,\cdot)$ is Lipschitz with Lipschitz constant $L_t'$ and $\{L_t'\}_{t\in\bbN}$ is bounded. 
Thus again by Lemma~\ref{lem:Lips_eqcont}, $\{Y^t(\omega,\cdot)\}_{t\in\bbN}$ is equicontinuous on $\bbR_+$. Consequently, we have $\{Z^t(\omega,\cdot)\}_{t\in\bbN}$ is asymptotically equicontinuous on $\bbR_+$.

\section{Proof of Theorem~\ref{thm:main} for Class $\calD_2\setminus\calD_1$}
Since we focus on the divergences in $\calD_2\setminus\calD_1$, then $\hpartial \tild_t(\bW_{t-1})$ reduces to $\partial \tild_t(\bW_{t-1})$, namely the subdifferential defined in the convex analysis. Our proof proceeds as follows. We first show some regularity conditions of the objective function $f:\calC\to\bbR$ and its subdifferential $\partial f:\calC\rightrightarrows\bbR^{F\times K}$ in Lemma~\ref{lem:reg_subdiff}. Next,  we prove that for any $t\in\bbN$, any stochastic (noisy) subgradient in $\partial \tild_t(\bW_{t-1})$ serves as an unbiased estimator of a ``true'' subgradient in $\partial f(\bW_{t-1})$ in Lemma~\ref{lem:unbiased_sg}. Finally we define some concepts related to the projected differential inclusion and present the counterparts of Lemma~\ref{lem:conv_limit_set} and  \ref{lem:char_limit_set} in Lemma~\ref{lem:conv_limit_set2} and \ref{lem:char_limit_set2} respectively. In particular, Lemma~\ref{lem:conv_limit_set2} and \ref{lem:char_limit_set2} together establish Theorem~\ref{thm:main} for the divergences in class $\calD_2\setminus\calD_1$.
The proofs of Lemma~\ref{lem:conv_limit_set2} and \ref{lem:char_limit_set2} are omitted since they are similar to those of Lemma~\ref{lem:conv_limit_set} and  \ref{lem:char_limit_set}. For details,  we refer readers to \cite[Section~5.6]{Kushner_03} and \cite[Chapter~5]{Borkar_08}.

\begin{lemma}\label{lem:reg_subdiff}
The objective function $f$ is convex on $\calC$. Moreover, for any $\bW\in\calC$, $\partial f(\bW)$ is nonempty, closed and convex. Furthermore, $\partial f(\bW)$ is bounded on $\inter\calC$ and upper semicontinuous on $\calC$. 
\end{lemma}
\begin{proof}
Since all the divergences $d(\cdot\Vert\cdot)$ in class $\calD_2$ are jointly convex in both arguments, $d(\bv\Vert\bW\bh)$ is jointly convex in $(\bv,\bW,\bh)\in\calV\times\calC\times\calH$. Since $\calH$ is convex and compact, by Lemma~\ref{lem:min_cvx}, $\ell(\bv,\bW)$ is jointly convex in $(\bv,\bW)\in\calV\times\calC$. Consequently $f$ is convex on $\calC$ by Lemma~\ref{lem:exp_cvx}. By \cite[Section~2]{Duchi_15}, $\partial f(\bW)$ is  closed and convex on $\calC$ and furthermore, $\partial f(\bW)$ is nonempty and bounded on $\inter\calC$. Since the divergences in $\calD_2\setminus\calD_1$ only include the $\ell_1$ and $\ell_2$ distances, it is easy to check $\partial f$ is also nonempty on $\bdr\calC$. By \cite[Section~1.3.7]{Kushner_03}, we have for any $\bW\in\calC$,
\begin{equation}
\partial f(\bW) = \bigcap_{\delta>0}\convcl\left[\bigcup_{\bW'\in\calB_\delta(\bW)}\partial f(\bW')\right],
\end{equation} 
where $\calB_\delta(\bW)\defeq\{\bW'\in\calC\,|\,\norm{\bW-\bW'}<\delta\}$. Thus by Lemma~\ref{lem:suff_usc}, we conclude that $\partial f$ is upper semicontinuous on $\calC$. 
\end{proof}

\begin{lemma}\label{lem:unbiased_sg}
Given $\bv\sim\bbP$ and $\bW\in\calC$ and let $\bh^*(\bv,\bW)\defeq\min_{\bh\in\calH}d(\bv\Vert\bW\bh)$ (by Assumption~\ref{assum:sc_h}). For any $\bG(\bv,\bW)\in\partial_\bW d(\bv\Vert\bW\bh^*(\bv,\bW))$, we have $\bbE_\bv[\bG(\bv,\bW)] \in\partial f(\bW)$.
\end{lemma}
\begin{proof}
For any $\bW'\in\calC$ and  any $\bG(\bv,\bW)\in\partial_\bW d(\bv\Vert\bW\bh^*(\bv,\bW))$, we have 
\begin{align*}
d(\bv\Vert\bW'\bh^*(\bv.\bW')) \ge d(\bv\Vert\bW\bh^*(\bv.\bW)) + \lrangle{\bG(\bv,\bW)}{\bW'-\bW} \,\mbox{a.s.},\,\forall\,\bW'\in\calC,
\end{align*}
which is clearly equivalent to 
\begin{equation}
\ell(\bv,\bW') \ge \ell(\bv,\bW) + \lrangle{\bG(\bv,\bW)}{\bW'-\bW} \,\mbox{a.s.},\,\forall\,\bW'\in\calC. 
\end{equation}
Taking expectation w.r.t.\ $\bv$ on both sides, we have
\begin{equation}
f(\bW') \ge f(\bW) + \lrangle{\bbE_\bv[\bG(\bv,\bW)]}{\bW'-\bW}, \,\forall\,\bW'\in\calC.
\end{equation}
In other words, $\bbE_\bv[\bG(\bv,\bW)] \in\partial f(\bW)$.
\end{proof}

\begin{definition}[Projected differential inclusion, limit set and critical points \cite{Aubin_84,Kushner_03}] \label{def:PDI}
Given a closed and convex set $\calK$ in a real Banach space $(\calX,\norm{\cdot})$, and an upper semicontinuous and compact, convex-valued correspondence $\calG:\calK\rightrightarrows\calX$, the projected differential inclusion (PDI) (on an interval $\calI\subseteq\bbR_+$) associated with $\calK$ and $\calG$ with initial value $x_0\in\calK$ is defined as
\begin{equation}
\frac{d}{ds} x(s) \in \calG(x(s))+z(s), \,z(s)\in\calN_\calK(x(s)),\; x(0) = x_0,\;s\in\calI, \label{eq:PDS}
\end{equation}
where $\calN_\calK(x)$ denotes the (inward) normal cone of set $\calK$ at $x\in\calK$ and is defined as
\begin{equation}
\calN_\calK(x) \defeq \{p\in\calX\,|\,\lrangle{p}{x'-x}\ge 0,\,\forall\,x'\in\calC\}.
\end{equation}
Denote $\calP(\calG,\calK,x_0)$ as the solution set of \eqref{eq:PDS}. The limit set of \eqref{eq:PDS}, $\calL(\calG,\calK,x_0)$ is defined as
\begin{align*}
&\calL(\calG,\calK,x_0) \defeq \bigcup_{x(\cdot)\in\calP(g,\calK,x_0)} \Big\{y\in\calK\,\Big|\,\exists\,\{s_n\}_{n\in\bbN}\subseteq\bbR_+,\;
s_n\uparrow\infty,\,x(s_n)\to y\Big\}.
\end{align*}
Moreover, the set of critical points associated with $\calG$ and $\calK$, $\calS(\calG,\calK)$ is defined as
\begin{equation}
\calS(\calG,\calK) \defeq \left\{x\in\calK\,\Big|\,\exists\,z\in\calN_\calK(x) \mbox{ s.t. } 0 \in \calG(x)+z\right\}.
\end{equation}
\end{definition}

\begin{lemma}\label{lem:conv_limit_set2}
The stochastic process $\{\bW_t\}_{t\in\bbN}$ generated in Algorithm~\ref{algo:ONMFG} converges almost surely to $\calL(-\partial f,\calC,\bW_0)$, the limit set of the following projected dynamical system 
\begin{equation}
\frac{d}{ds}W(s) \in -\partial f(W(s))+Z(s),\,Z(s)\in\calN_\calC(W(s)), \;W(0) = \bW_0, \;s\ge 0. \label{eq:PDI_W0}
\end{equation}
\end{lemma}

\begin{lemma}\label{lem:char_limit_set2}
In \eqref{eq:PDI_W0}, we have $\calL(-\partial f,\calC,\bW_0)\subseteq \calS(-\partial f,\calC)$, i.e., every limit point of \eqref{eq:PDI_W0} is a critical point associated with $-\partial f$ and $\calC$. Moreover, each $\bW\in\calS(-\partial f,\calC)$ satisfies the following variational inequality 
\begin{equation}
 f'(\bW;\bW'-\bW) \ge 0, \,\forall\,\bW'\in\calC. \label{eq:variat_inclusion}
\end{equation}
\end{lemma}

\begin{remark}
Note that our (almost sure) convergence proof for the divergences in class $\calD_2\setminus\calD_1$ covers the proof for those in $\calD_1\cap\calD_2$ (see Section~\ref{sec:conv_analyses}) as a special case. In particular, in Section~\ref{sec:conv_analyses}, $\partial f$ is a singleton so all the regularities of $\partial f$ in Lemma~\ref{lem:reg_subdiff} are naturally satisfied. As such, the proof in this section serves as a unified way to prove convergence for all the divergences in $\calD_2$. 
\end{remark}

\section{Technical Lemmas}
\subsection{Convergence of PGD and MM algorithms}

\begin{lemma}[Adapted from {\cite[Theorem~1]{Raza_13}}] \label{lem:conv_MM}
Given a real Hilbert space $\calY$ and a function $f:\calY\to\bbR$, consider the following optimization problem 
\begin{equation}
\min_{x\in\calX}f(x), \label{eq:opt_f}
\end{equation}
where $\calX\subseteq\calY$ is nonempty, closed and convex and $f$ is differentiable on $\calX$. For any $x\in\calX$, define a differentiable function $u(x,\cdot):\calX\to\bbR$ such that $u(x,\cdot)$ is a majorant for $f$ at $x$.\footnote{By this, we mean $u(x,x)=f(x)$ and $u(x,y)\ge f(y)$ for any $y\in\calX$.} Fix an arbitrary initial point $x_0\in\calX$ and consider the sequence of iterates $\{x^k\}_{k\in\bbN}$ generated by the following MM algorithm
\begin{equation}
x^k := \min_{y\in\calX} u(x^{k-1},y),\;\forall\;k\in\bbN.
\end{equation}
Then $\{x^k\}_{k\in\bbN}$ has at least one limit point and moreover, the any limit point of $\{x^k\}_{k\in\bbN}$ is a stationary point of \eqref{eq:opt_f}.
\end{lemma}

\begin{lemma}[Adapted from {\cite[Theorem~2.4]{Calamai_87}}] \label{lem:conv_PGD}
Consider a real Hilbert space $\calY$. Let $\calX\subseteq\calY$ be a nonempty compact convex set and $f:\calY\to\bbR$ be continuously differentiable on $\calY$. Fix an arbitrary initial point $x_0\in\calX$ and consider the sequence of iterates $\{x^k\}_{k\in\bbN}$ generated by the following projected gradient algorithm
\begin{equation}
x^{k} := \Pi_\calX\Big\{x^{k-1}-\beta^k\nabla f(x^{k-1})\Big\},\;\forall\;k\in\bbN,
\end{equation}
where the sequence of step sizes $\{\beta^k\}_{k\in\bbN}$ is chosen according to the Armijo rule \cite{Armijo_66}. 
Then $\{x^k\}_{k\in\bbN}$ has at least one limit point and moreover, the any limit point\footnote{The limit point is defined in the topological sense, i.e., $\barx\in\calX$ is a limit point of $\{x^k\}_{k\in\bbN}$ if for any neighborhood $\calU$ of $\barx$, there are infinitely many elements of $\{x^k\}_{k\in\bbN}$ in $\calU$.} of $\{x^k\}_{k\in\bbN}$ is a stationary point of the optimization problem $\min_{x\in\calX}f(x)$.
\end{lemma}

\subsection{Optimal-value functions}
\begin{lemma}[The Maximum Theorem; {\cite[Theorem 14.2.1 \& Example 2]{Syd_05}}]\label{lem:maximum}
Let $\calP$ and $\calX$ be two metric spaces. Consider a maximization problem
\begin{equation}
\max_{x\in B(p)} f(p,x), \label{eq:max_problem}
\end{equation}
where $B:\calP \rightrightarrows \calX$ is a correspondence and  $f:\calP\times \calX\to\bbR$ is a function. If $B$ is compact-valued and continuous on $\calP$ and $f$ is continuous on $\calP\times \calX$, then the correspondence $S(p) = \argmax_{x\in B(p)} f(p,x)$ is compact-valued and upper hemicontinuous, for any $p\in\calP$.
In particular, if for some $p_0\in\calP$, $S(p_0)=\{s(p_0)\}$, where $s:\calP\to\calX$ is a function, then $s$ is continuous at $p=p_0$. Moreover, we have the same conclusions if the maximization in \eqref{eq:max_problem} is replaced by minimization. 
\end{lemma}

\begin{lemma}[Danskin's Theorem; {\cite[Theorem 4.1]{Bon_98}}] \label{lem:Danskin}
Let $\calX$ be a metric space and $\calU$ be a normed vector space. Let $f:\calX\times \calU\to\bbR$ have the following properties
\begin{enumerate}
\item $f(x,\cdot)$ is differentiable on $\calU$, for any $x\in\calX$.
\item $f(x,u)$ and $\nabla_u f(x,u)$ are continuous on $\calX\times\calU$.
\end{enumerate}
Let $\Phi$ be a compact set in $\calX$. Define $v(u) = \inf_{x\in\Phi}f(x,u)$ and $S(u) = \argmin_{x\in\Phi}f(x,u)$, then $v(u)$ is (Hadamard) directionally differentiable and its directional derivative along $d\in\calU$, $v'(u,d)$ is given by
\begin{equation}
v'(u,d) = \min_{x\in S(u)} \lrangle{\nabla_uf(x,u)}{d}.
\end{equation}
In particular, if for some $u_0\in\calU$, $S(u_0) = \{x_0\}$, then $v$ is (Hadamard) differentiable at $u=u_0$ and $\nabla v(u_0) = \nabla_uf(x_0,u_0)$.
\end{lemma}

\begin{lemma}[Minimization of convex functions; {\cite[Section~3.2.5]{Boyd_04},\cite{Duchi_15}}]\label{lem:min_cvx}
Let $\calX$ and $\calY$ be two inner product spaces and $\calX\times\calY$ be their product space such that for any $(x,y)$ and $(x',y')$ in $\calX\times\calY$, $\lrangle{(x,y)}{(x',y')} = \lrangle{x}{x'}+\lrangle{y}{y'}$. Consider functions $h:\calX\times\calY\to\bbR$ and $f:\calX\to\bbR$ such that
\begin{equation}
f(x) \defeq \inf_{y\in\calY} h(x,y),\,\forall\,x\in\calX.
\end{equation} 
If $h$ is convex on $\calX\times\calY$ and $\calY$ is convex, then $f$ is convex on $\calX$. If we further assume $\calS(x_0)\defeq\argmin_{y\in\calY} h(x_0,y)\neq\emptyset$, then the subdifferential of $f$ at $x_0\in\calX$, 
\begin{equation}
\partial f(x_0)=\bigcup_{y_0\in\calS(x_0)}\left\{g\in\calX\,|\,(g,g')\in\partial h(x_0,y_0), \,\mbox{where}\,\lrangle{g'}{y-y_0}=0,\forall\,y\in\calY\right\}.
\end{equation} 
\end{lemma}
\begin{proof}
\begin{align*}
g \in f(x_0) &\Longleftrightarrow f(x)\ge f(x_0) + \lrangle{g}{x-x_0}, \forall\,x\in\calX\\
&\Longleftrightarrow h(x,y) \ge h(x_0,y_0) + \lrangle{(g,g')}{(x-x_0,y'-y_0)}, \forall\,x\in\calX, \forall\,y\in\calS(x), \forall\,y_0\in\calS(x_0),\forall\,y'\in\calY,\\
&\hspace{12cm}\forall\,g'\in\calX\mbox{ s.t. } \lrangle{g'}{y'-y_0}=0\\
&\Longleftrightarrow h(x,y) \ge h(x_0,y_0) + \lrangle{(g,g')}{(x-x_0,y-y_0)}, \forall\,x\in\calX, \forall\,y\in\calY, \forall\,y_0\in\calS(x_0),\forall\,g'\mbox{ s.t. }\lrangle{g'}{y-y_0}=0\\
&\Longleftrightarrow (g,g')\in\partial h(x_0,y_0),\forall\,y_0\in\calS(x_0),\forall\,g'\mbox{ s.t. }\lrangle{g'}{y-y_0}=0,\forall\,y\in\calY.
\end{align*}
\end{proof}

\subsection{Miscellaneous}
\begin{lemma}[Leibniz Integral Rule] \label{lem:Leib_int}
Let $\calX$ be an open set in $\bbR^n$ and let $(\Omega,\calA,\mu)$ be a measure space. If $f:\calX\times\Omega\to\bbR$ satisfies
\begin{enumerate}
\item For all $x\in\calX$, the mapping $\omega\mapsto f(x,\omega)$ is Lebesgue integrable. 
\item For all $\omega\in\Omega$, $\nabla_xf(x,\omega)$ exists on $\calX$.
\item For all $x\in\calX$,  the mapping $\omega\mapsto \nabla_x f(x,\omega)$ is Lebesgue integrable. 
\end{enumerate}
Then $\int_\Omega f(x,\omega) \, d\mu(\omega)$ is differentiable on $\calX$ and  for each $x\in\calX$,
\begin{equation}
\nabla_x \int_\Omega f(x,\omega) \, d\mu(\omega) = \int_\Omega \nabla_x f(x,\omega) \, d\mu(\omega).
\end{equation}
\end{lemma}
\begin{remark}
This is a simplified version of the Leibniz Integral Rule. See \cite[Theorem 16.8]{Bill_86} for weaker conditions on $f$.
\end{remark}

\begin{lemma}[Almost sure convergence of square-integrable martingales; {\cite[Theorem 5.4.9]{Durrett_13}}] \label{lem:conv_sq_mart}
Let $\{X_n\}_{n\ge 1}$ be a martingale in a normed space $\calX$ adapted to the filtration $\{\scF_n\}_{n\ge 0}$ such that $\sup_{n\in\bbN} \bbE\left[\norm{X_n}^2\right]<\infty$.  Define the {\em quadratic variation process} $\{\lrang{X}_n\}_{n\ge 2}$ as
\begin{equation}
\lrang{X}_n \defeq \sum_{i=2}^n \bbE\left[\norm{X_i-X_{i-1}}^2|\scF_{i-1}\right],\;\forall\,n\ge 2.
\end{equation}
Then there exists a random variable $X$ such that on the set $\left\{\lim_{n\to\infty}\lrang{X}_n<\infty\right\}$, the sequence $\{X_n\}_{n\ge 1}$ converges a.s. to $X$ and $\norm{X}<\infty$ a.s.. 
\end{lemma}

\begin{lemma}[Expectation of convex functions; {\cite{Duchi_15}}]\label{lem:exp_cvx}
Let $(\calU,\scA,\nu)$ be a probability space and $h:\calX\times\calU$ be a function such that for each $u\in\calU$, $x\mapsto h(x,u)$ is convex on $\calX$, where $\calX$ is a convex set equipped with an inner product $\lrangle{\cdot}{\cdot}$. Define $f(x)\defeq\bbE_u(x,u)$, for any $x\in\calX$. Then $f$ is convex on $\calX$. Fix any $x_0\in\calX$.  Then for any $g_{x_0}(u)\in\partial_x h(x_0,u)$, $\bbE_u[g_{x_0}(u)]\in\partial f(x_0)$. 
\end{lemma}

\subsection{Asymptotic Equicontinuity and Uniform Convergence}
In this section, unless otherwise mentioned, we assume the sequences of functions $\{f_n\}_{n\in\bbN}$ and $\{g_n\}_{n\in\bbN}$ are defined on a common  metric space $(\calX,d)$ and mapped to a common metric space $(\calY,\rho)$.

\begin{lemma}[Uniform convergence implies asymptotic equicontinuity] \label{lem:uniconv_asymcont}
If a sequence of functions $\{f_n\}_{n\in\bbN}$ 
converges uniformly to a continuous function $f$ on $\calX$, then it is asymptotically equicontinuous on $\calX$.
\end{lemma}
\begin{proof}
Fix $\epsilon>0$. Since $f_n\convu f$, there exists $N\in\bbN$ such that for all $n\ge N$, $\sup_{x\in\calX} \abs{f_n(x)-f(x)}<\epsilon/6$. Fix $x_0\in\calX$. Then there exists $\delta>0$ such that $\sup_{x'\in \calN_\delta(x)} \rho(f(x),f(x')) <\epsilon/6$, where $\calN_\delta(x)\defeq\{x'\in\calX:d(x,x')<\delta\}$. Thus for all $n\ge N$, $\sup_{x'\in \calN_\delta(x)} \rho(f_n(x),f_n(x')) \le \rho(f_n(x),f(x)) + \sup_{x'\in \calN_\delta(x)} \rho(f(x),f(x')) +  \sup_{x'\in \calN_\delta(x)} \rho(f_n(x'), f(x')) < \epsilon/2$. This shows $\limsup_{n\to\infty} \;\sup_{x'\in\calN_\delta(x)} \rho(f_n(x),f_n(x')) < \epsilon$. Since this holds for all $x\in\calX$, we complete the proof. 
\end{proof}

\begin{lemma}[Lipschitzness implies equicontinuity] \label{lem:Lips_eqcont}
Given a sequence of continuous functions $\{f_n\}_{n\in\bbN}$. If each $f_n$ is Lipschitz on $\calX$ with Lipschitz constant $L_n$ and there exists $M\in(0,\infty)$ such that $\sup_{n\ge 1} L_n\le M$, then $\{f_n\}_{n\in\bbN}$ is equicontinuous on $\calX$. 
\end{lemma}

\begin{lemma}[Finite sum preserves asymptotic equicontinuity] \label{lem:finite_sum_equicont}
Let $\{f_n\}_{n\in\bbN}$ and $\{g_n\}_{n\in\bbN}$ be both asymptotically equicontinuous on $\calX$. Assume the metric $\rho$ is translation-invariant (for e.g., induced by a norm). Then $\{f_n+g_n\}_{n\in\bbN}$  is asymptotically equicontinuous on $\calX$.
\end{lemma}
\begin{proof}
Fix an $\epsilon>0$ and $x\in\calX$, there exist $\delta_1>0$ and $\delta_2>0$ respectively such that 
\begin{align*}
\limsup_{n\to\infty} \sup_{x'\in\calX:d(x,x')<\delta_1} \rho(f_n(x),f_n(x')) &< \epsilon/2,\\
\limsup_{n\to\infty} \sup_{x'\in\calX:d(x,x')<\delta_2} \rho(g_n(x),g_n(x')) &< \epsilon/2.
\end{align*}
Take $\delta=\min(\delta_1,\delta_2)$, we have
\begin{align*}
&\;\limsup_{n\to\infty} \sup_{x'\in\calX:d(x,x')<\delta} \rho((f_n+g_n)(x),(f_n+g_n)(x')) \\
\le&\; \limsup_{n\to\infty} \sup_{x'\in\calX:d(x,x')<\delta} \rho((f_n+g_n)(x),f_n(x')+g_n(x)) + \rho(f_n(x')+g_n(x),f_n(x')+g_n(x'))\\
\le&\; \limsup_{n\to\infty} \sup_{x'\in\calX:d(x,x')<\delta} \rho(f_n(x),f_n(x')) + \rho(g_n(x),g_n(x')) \\
<&\; \epsilon. 
\end{align*}
\end{proof}

\begin{lemma}[Continuous transformation preserves uniform convergence] \label{lem:cont_uniconv}
Assume $\calX$ to be compact. Let $g:\calY\to\calZ$ be a continuous function, where $(\calZ,r)$ is a metric space. If $\{f_n\}_{n\in\bbN}$ uniformly converges to a continuous function $f$ on $\calX$, then $\{g\circ f_n\}_{n\in\bbN}$ uniformly converges to $g\circ f$ on $\calX$.  
\end{lemma}
\begin{proof}
First, since $\calX$ is compact and $f$ is continuous on $\calX$, $f(\calX)$ is compact in $\calY$. Since $g$ is continuous on $\calY$, $g$ is uniformly continuous on $f(\calX)$. Fix $\epsilon>0$. there exists a $\delta>0$ such that for all $y,y'\in f(\calX)$ and $\rho(y,y')<\delta$, $r(g(y),g(y'))<\epsilon$. Since $f_n\convu f$ on $\calX$, there exits a $K\in\bbN$ such that for all $n\ge K$ and $x\in\calX$, $\rho(f(x),f_n(x))<\delta$. Consequently, $r(g(f_n(x)),g(f(x)))<\epsilon$. This implies $g\circ f_n\convu g\circ f$ on $\calX$. 
\end{proof}

\begin{lemma}[Generalized Arzel\`{a}-Ascoli Theorem \cite{Yin_05}]\label{lem:arzela}
If the sequence of functions $\{f_n\}_{n\ge 1}$ is asymptotically equicontinuous and uniformly bounded on $\calX$ (assumed to be compact), then there exists a subsequence $\{f_{n_k}\}_{k\ge 1}$ that converges uniformly to a continuous function $f$ on $\calX$. 
\end{lemma}

\subsection{Projected Dynamical Systems and Lyapunov Stability Theory}

\begin{lemma}[Adapted from {\cite[Theorem 3.1, Chapter 4]{Kushner_03}}] \label{lem:suff_pode}
Assume \eqref{eq:mean_PODE} holds with $\barZ(\omega,0)=0$ and $\barW(\omega,s)\in\calC$, for all $s\ge 0$. Denote $\lambda$ as the Lebesgue measure on $\bbR$. If $\barW(\omega,\cdot)$ is Lipschitz on $\bbR_+$ and for any $\tau>0$, 
\begin{enumerate}
\item $\barZ(\tau)=\vecz$ if $\barW(\omega,s)\in\inter\calC$ for all $s\in\calT$, where $\calT$ is any set in $[0,\tau]$ with $\lambda(\calT) = \tau$,
\item $\barZ(\tau)\in \convcl \left[\bigcup_{s\in[0,\tau]}\calN\left(\barW(\omega,s)\right)\right]$,
\end{enumerate}
where  $\calN$ is defined in \eqref{eq:def_normal}, then 
\begin{equation}
\barZ(s) = \int_0^s z(\tau)\;d\tau, 
\end{equation}
where $z:\bbR_+\to\bbR^{F\times K}$ is defined in \eqref{eq:def_zs}. 
\end{lemma}

\begin{definition}[Lyapunov function and its Lie derivative; {\cite[Section~6.6]{Teschl_12}}]\label{def:lya_func}
Consider the PDS given in \eqref{eq:PDS}. Assume the normed space $(\calX,\norm{\cdot})$ is equipped with the inner product $\lrangle{\cdot}{\cdot}$.  Fix $x_0\in\calK$ and choose a neighborhood of $x_0$ in $\calK$, denoted as $\calU(x_0)$. A continuously differentiable function $L:\calU(x_0)\to\bbR_+$ is called a Lyapunov function if $L(x_0)=0$, $L(x)>0$ for any $x\in\calU(x_0)\setminus\{x_0\}$ and for any $x(\cdot)\in\calP(g,\calK,x_0)$,
\begin{equation}
L(x(t_1)) \le L(x(t_0)), \,\forall\,t_0, t_1\in\calI, t_0<t_1, \;\mbox{s.t.}\;\{x(t_0),x(t_1)\}\subseteq \calU(x_0)\setminus\{x_0\}.
\end{equation}
Moreover, for any $x(\cdot)\in\calP(g,\calK,x_0)$, the Lie derivative of $L$ on $\calI$, $\frac{d}{ds}L(x(s))$ is given by
\begin{equation}
\frac{d}{ds}L(x(s)) = \lrangle{\nabla_x L(x(s))}{x'(s)}, \,\forall\,s\in\calI.
\end{equation}
\end{definition}

\begin{lemma}[All limit points are stationary; {\cite[Theorem 6.15]{Teschl_12}}] \label{lem:limit_st}
Consider the PDS given in \eqref{eq:PDS}. Let $L:\calU\subseteq\calK\to\bbR_+$ be a Lyapunov function with possibly non-unique zeros (i.e., $L$ may only be positive semidefinite on $\calU$). Suppose each solution $x(\cdot)\in\calP(g,\calK,x_0)$ is contained in $\calU$, then $L$ is constant on $\calL(g,\calK,x_0)\cap\calU$. In other words, the Lie derivative of $L$ vanishes on $\calL(g,\calK,x_0)\cap\calU$.
\end{lemma}

\subsection{Correspondence and Upper Semicontinuity}
For further details, see \cite[Chapter 1]{Aubin_84}. 

\begin{definition}[Correspondence and its graph] \label{def:correspondence}
Given two metric spaces $(\calX,d)$ and $(\calY,\rho)$, a correspondence $\calF:\calX\rightrightarrows\calY$ maps esch $x\in\calX$ to a subset $\calF(x)$ in $\calY$. The graph of $\calF$, $\calG(\calF)$ is defined as
\begin{equation}
\calG(\calF) \defeq \left\{(x,y)\in\calX\times\calY\,|\,y\in\calF(y)\right\}. 
\end{equation}
\end{definition}

\begin{definition}[Upper semicontinuous correspondence]\label{def:usc}
A correspondence $\calF$ as defined in Definition~\ref{def:correspondence} is called upper semicontinuous at $x_0\in\calX$ if for any open set $\calU\subseteq\calY$ such that $\calF(x_0)\subseteq\calU$, there exists an open set $\calV\subseteq\calX$ such that $x_0\in\calV$ and $\calF(x)\subseteq\calU$ for any $x\in\calV$. 
\end{definition}

\begin{lemma}[Closed graph property; {\cite[Proposition 2 \& Colrollary 1]{Aubin_84}}]\label{lem:closed_graph}
Let the correspondence $\calF$ be given in Definition~\ref{def:correspondence}.
\begin{enumerate}
\item If for each $x\in\calX$, $\calF(x)$ is closed, and $\calF$ is upper semicontinuous, then $\calG(\calF)$ is closed (in $\calX\times\calY$).
\item If $(\calY,\rho)$ is a compact metric space, and $\calG(\calF)$ is closed, then $\calF$ is upper semicontinuous. 
\end{enumerate}
\end{lemma}

\begin{lemma}[Sufficient conditions for upper semicontinuity; {\cite[Section~1.1]{Aubin_84}}] \label{lem:suff_usc}
Let the correspondence $\calF$ be given in Definition~\ref{def:correspondence}. Assume $\calF$ is compact-valued on $\calX$. Fix $x\in\calX$. If $\calF$ satisfies
\begin{equation}
\bigcap_{\delta>0}\convcl\left(\bigcup_{z\in\calB_\delta(x)}\calF(z)\right) = \calF(x),
\end{equation}
then $\calF$ is upper semicontinuous at $x$. Here $\convcl\calS$ denotes the closed convex hull of a set $\calS$ and $\calB_\delta(x)\defeq \{z\in\calX\,|\,d(x,z)<\delta\}$. 
\end{lemma}

\section{Experiment Results for Section~\ref{sec:insensitivity}}

The plots of objective values versus time of {\sf OL} with different values of $\tau$, $K$ and $a$ are shown in Figure~\ref{fig:param_tau}, \ref{fig:param_K} and \ref{fig:param_a} respectively. 

\begin{figure}[t]
\subfloat[IS]{\includegraphics[width=.475\columnwidth]{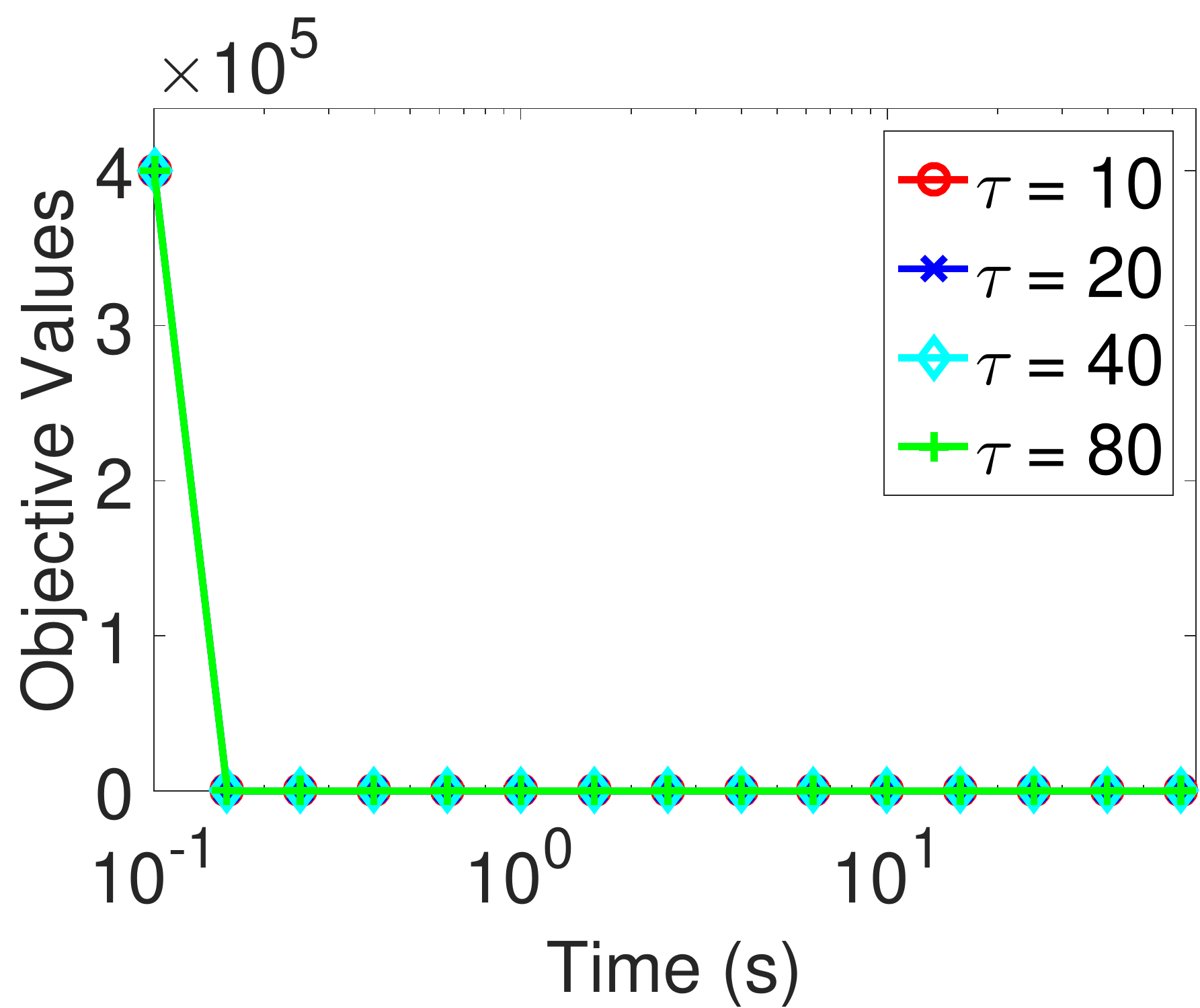}}\hfill
\subfloat[KL]{\includegraphics[width=.475\columnwidth,height=.36\columnwidth]{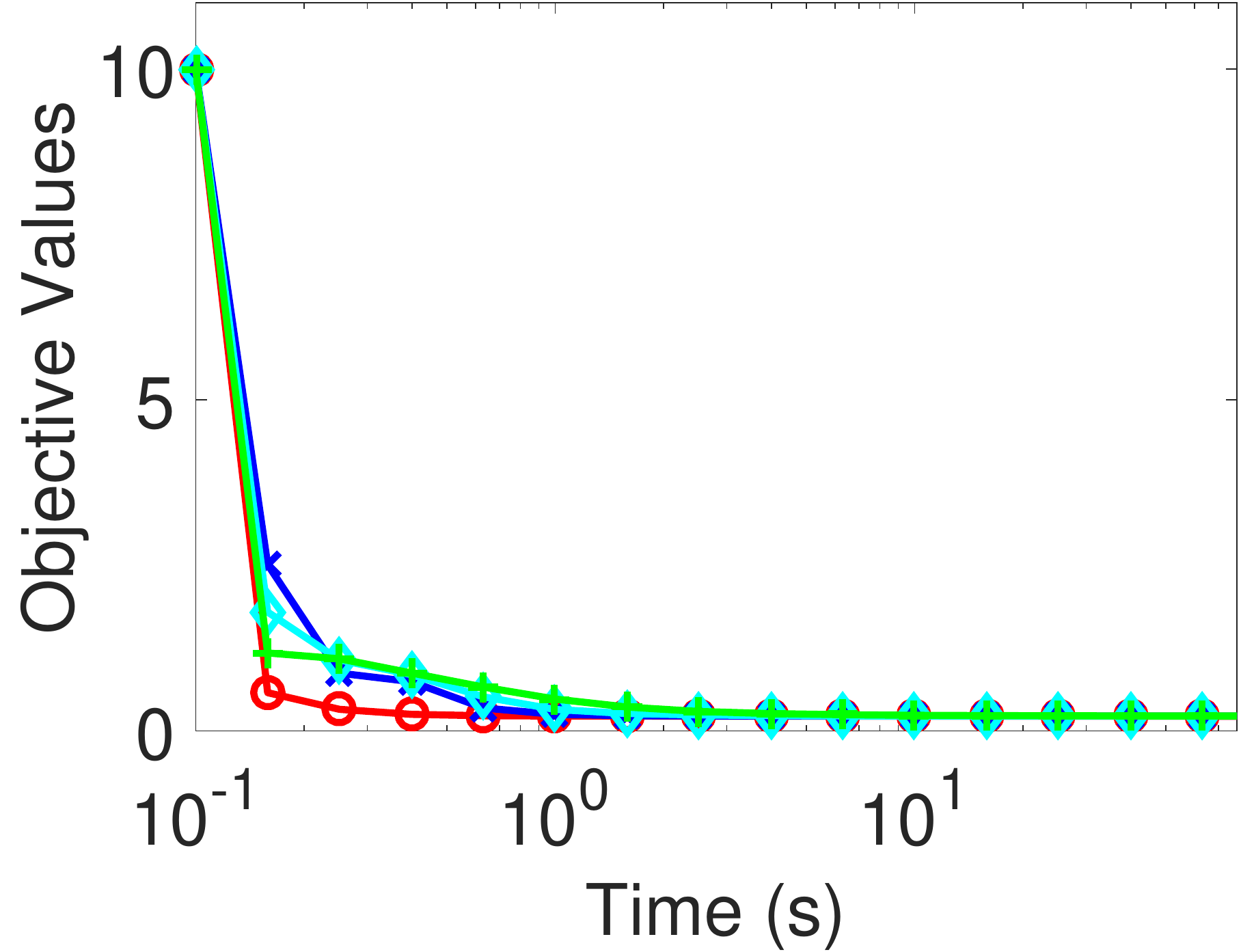}}\vspace{-.4cm}\\
\subfloat[Squared-$\ell_2$]{\includegraphics[width=.475\columnwidth]{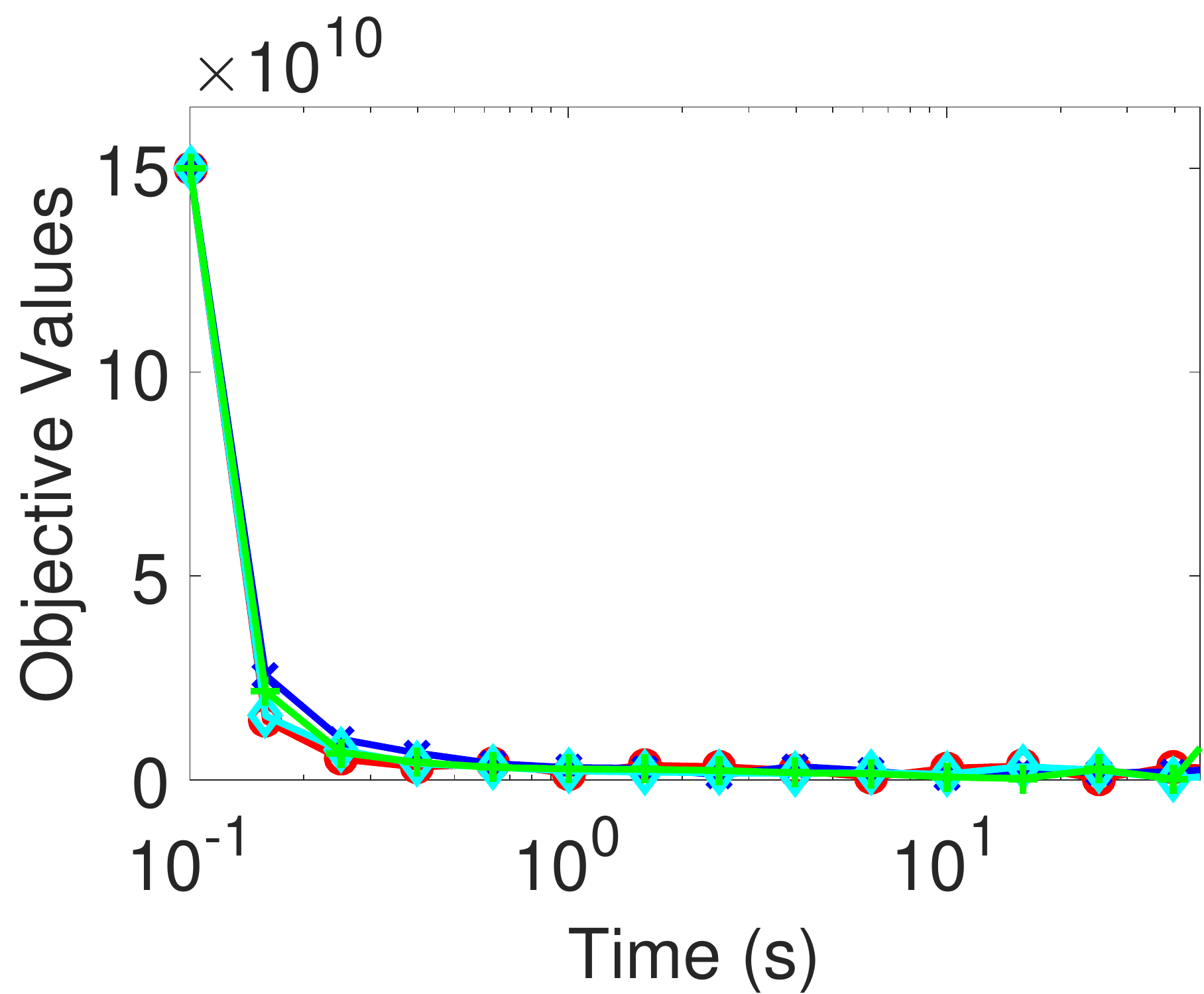}}\hfill
\subfloat[Huber]{\includegraphics[width=.475\columnwidth]{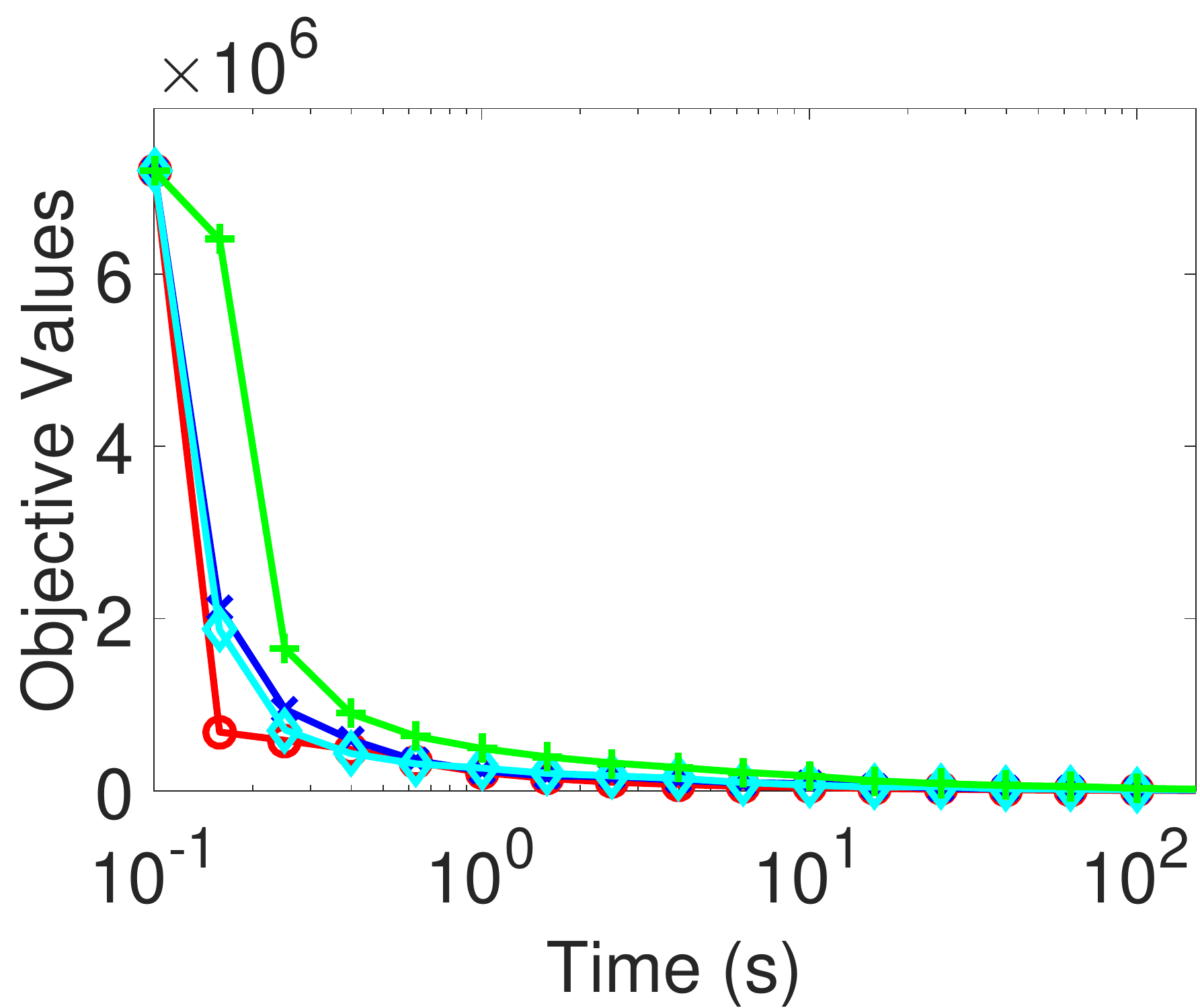}}\vspace{-.4cm}\\
\subfloat[$\ell_1$]{\includegraphics[width=.475\columnwidth]{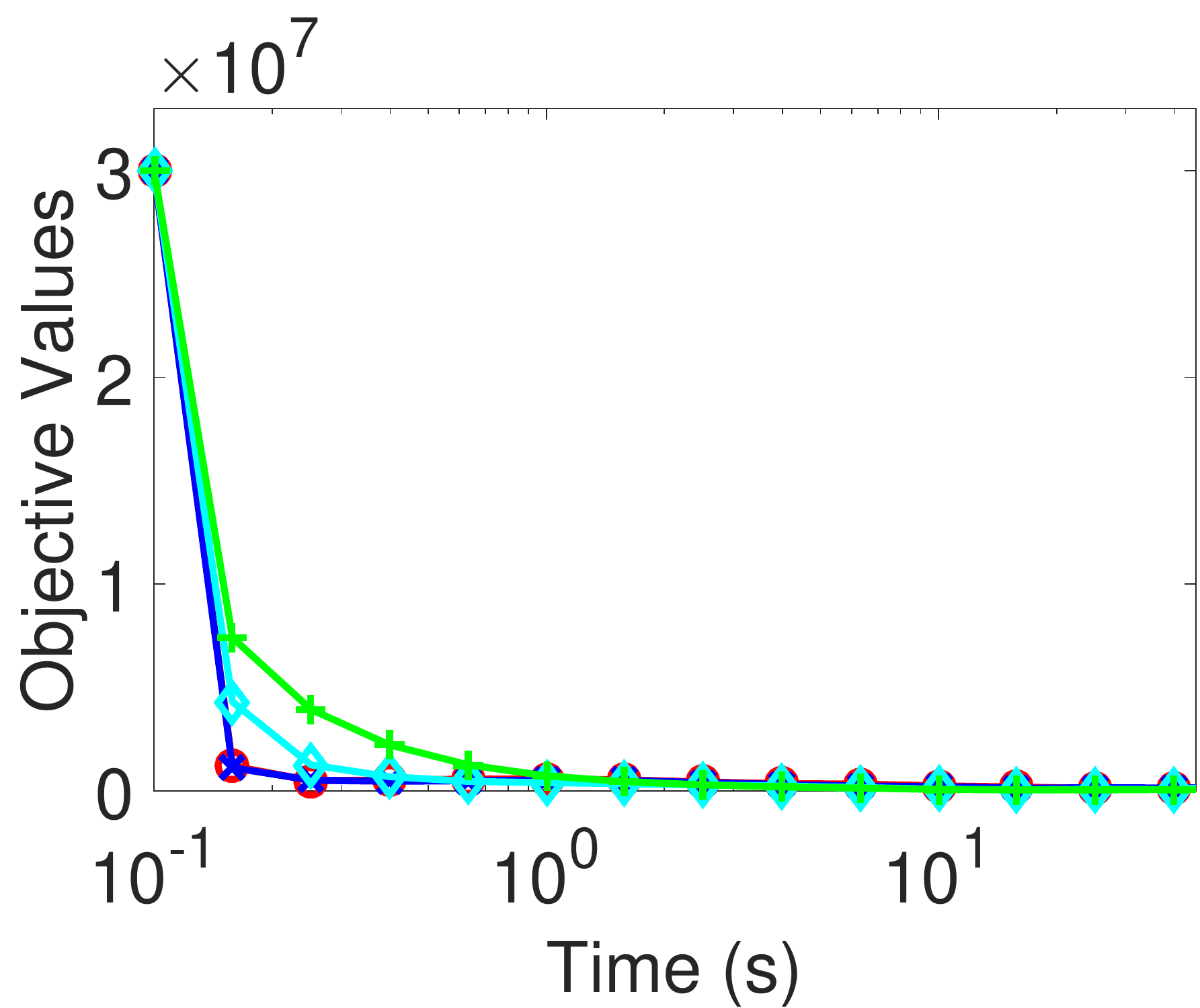}}\hfill
\subfloat[$\ell_2$]{\includegraphics[width=.475\columnwidth,height=.395\columnwidth]{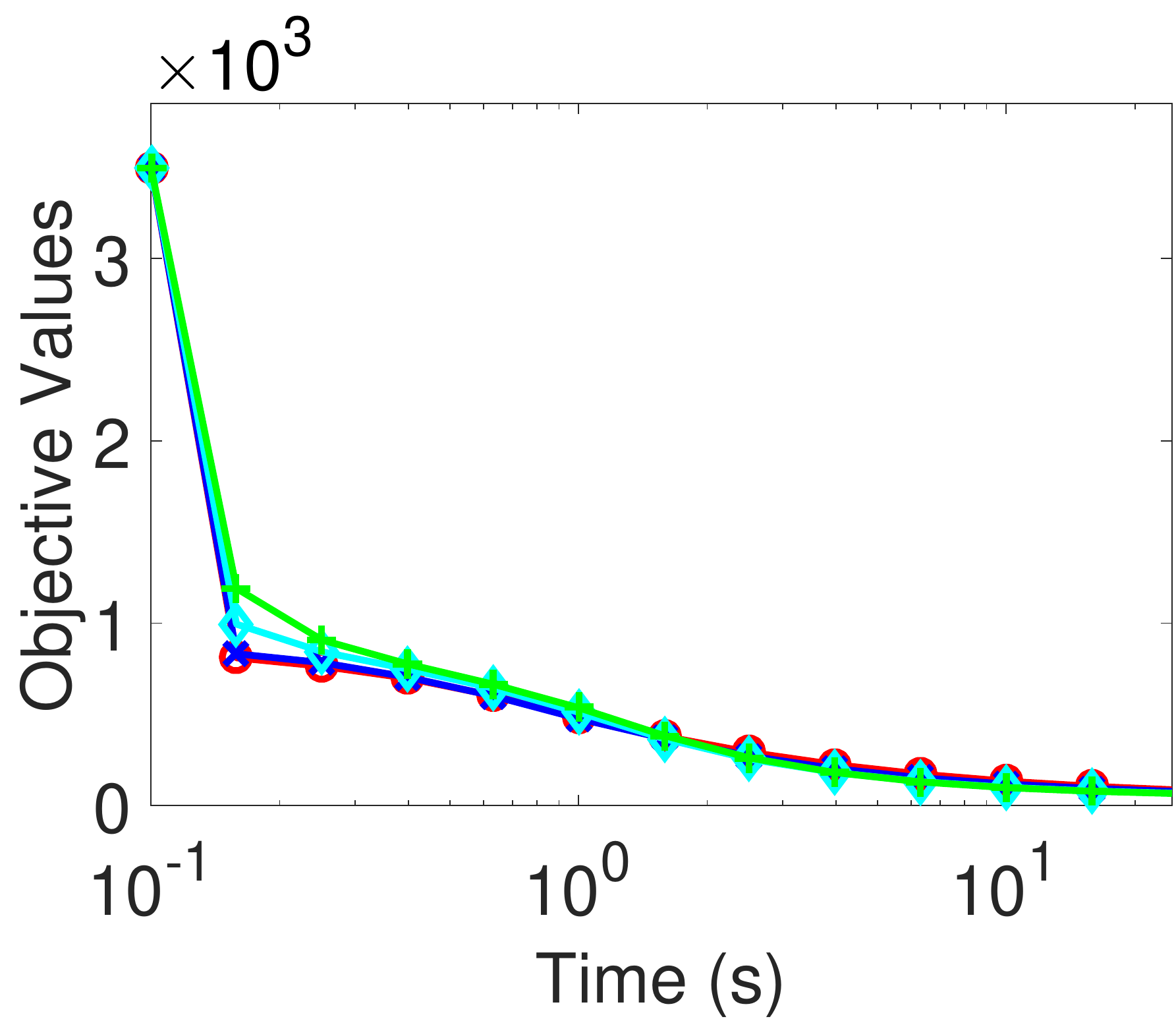}}\vspace{-.4cm}\\
\caption{Objective values versus time (in seconds) of \ol with different values of $\tau$ for all the divergences in $\barcalD$. $K$ and $a$ are in the canonical setting.}\vspace{-.5cm} \label{fig:param_tau}
\end{figure}

\begin{figure}[t]
\subfloat[IS]{\includegraphics[width=.475\columnwidth]{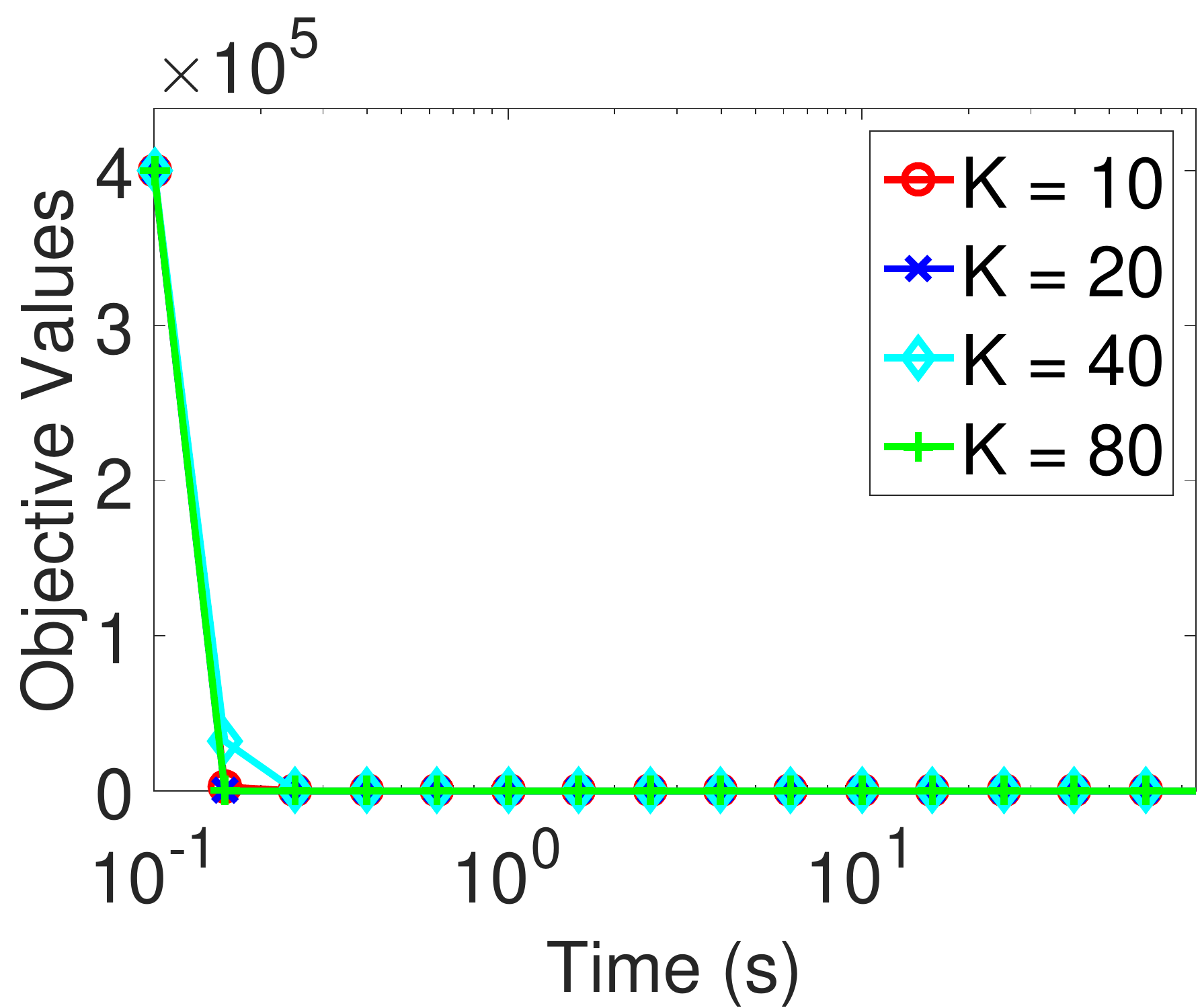}}\hfill
\subfloat[KL]{\includegraphics[width=.475\columnwidth,height=.36\columnwidth]{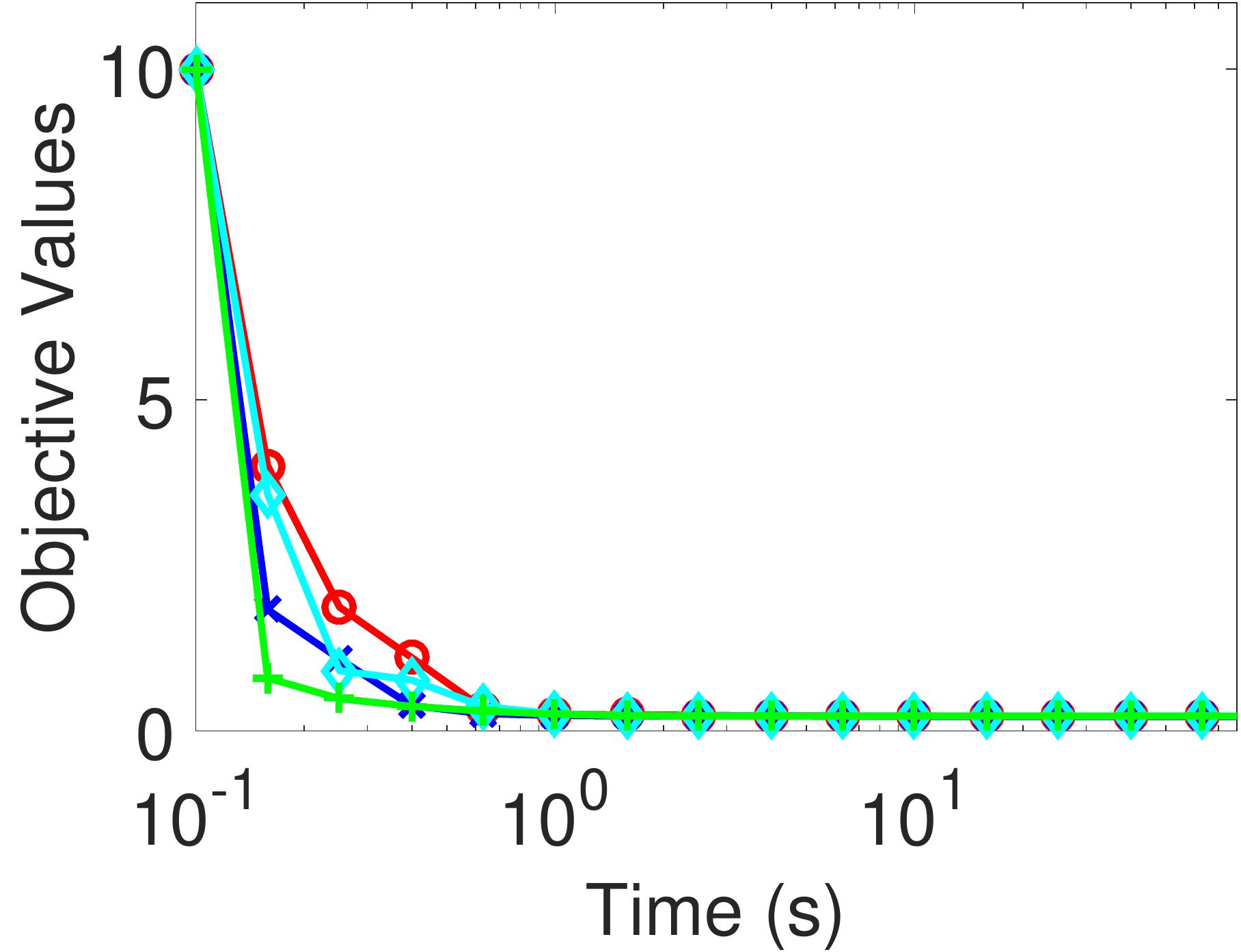}}\vspace{-.4cm}\\
\subfloat[Squared-$\ell_2$]{\includegraphics[width=.475\columnwidth]{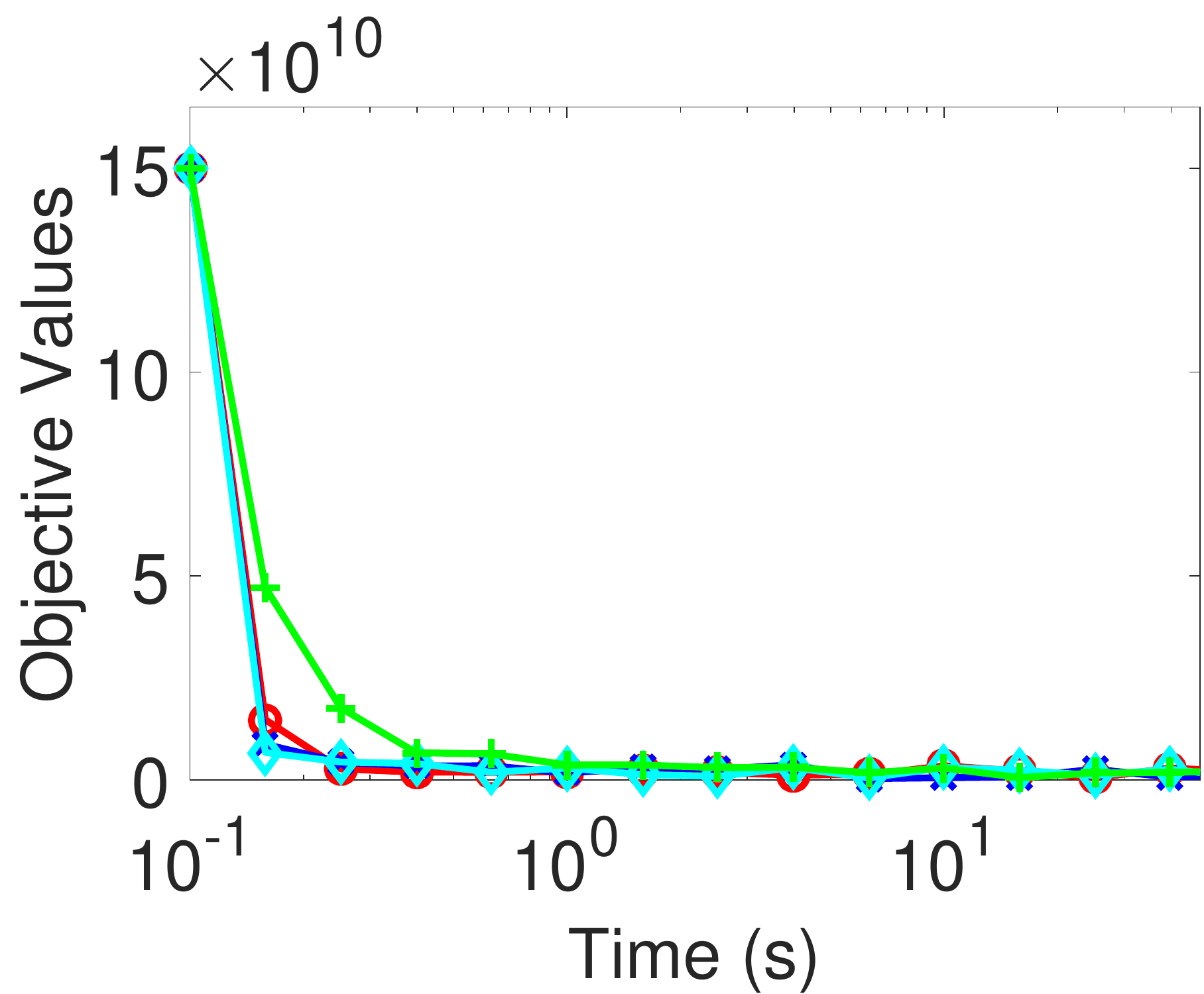}}\hfill
\subfloat[Huber]{\includegraphics[width=.475\columnwidth]{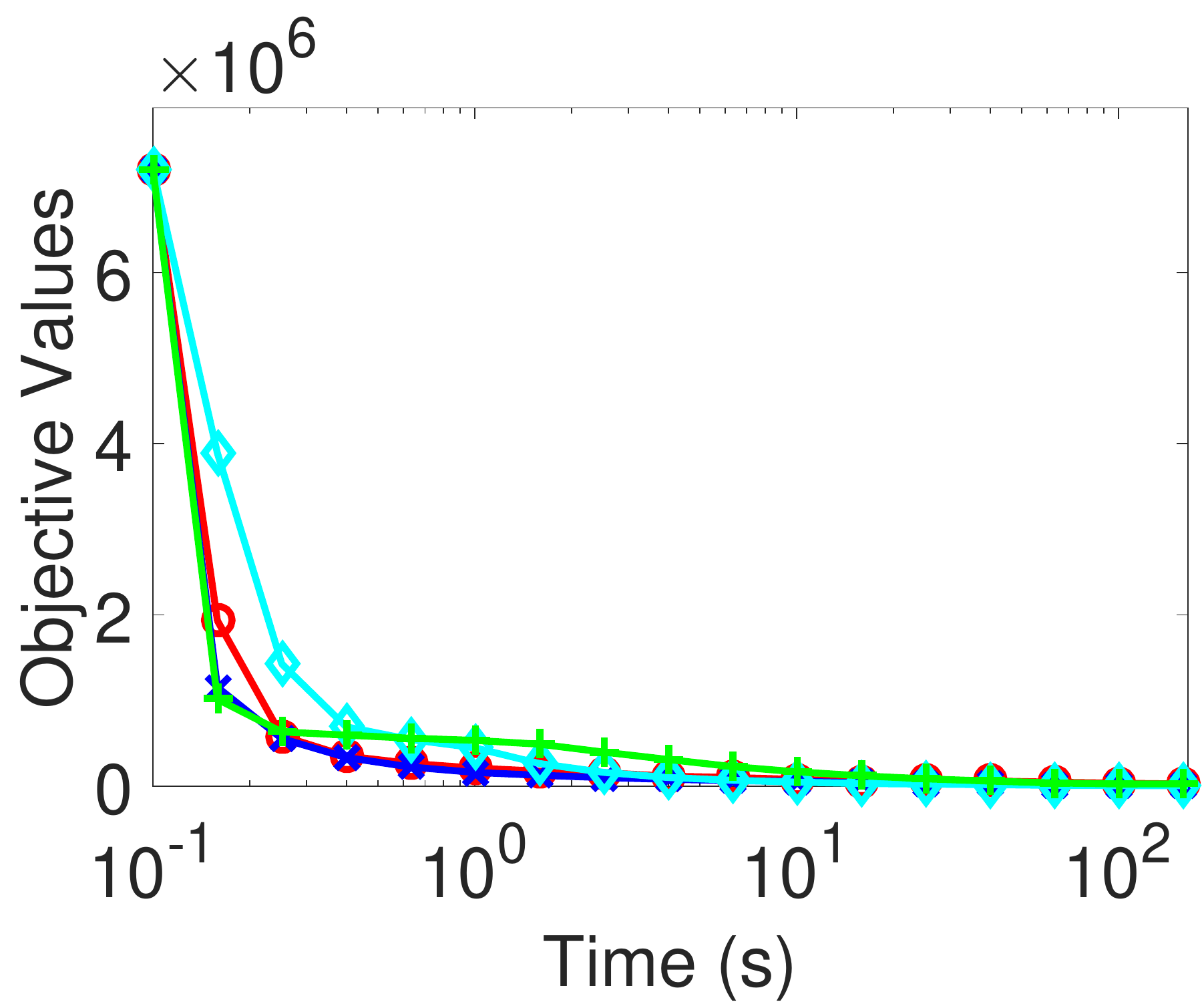}}\vspace{-.4cm}\\
\subfloat[$\ell_1$]{\includegraphics[width=.475\columnwidth]{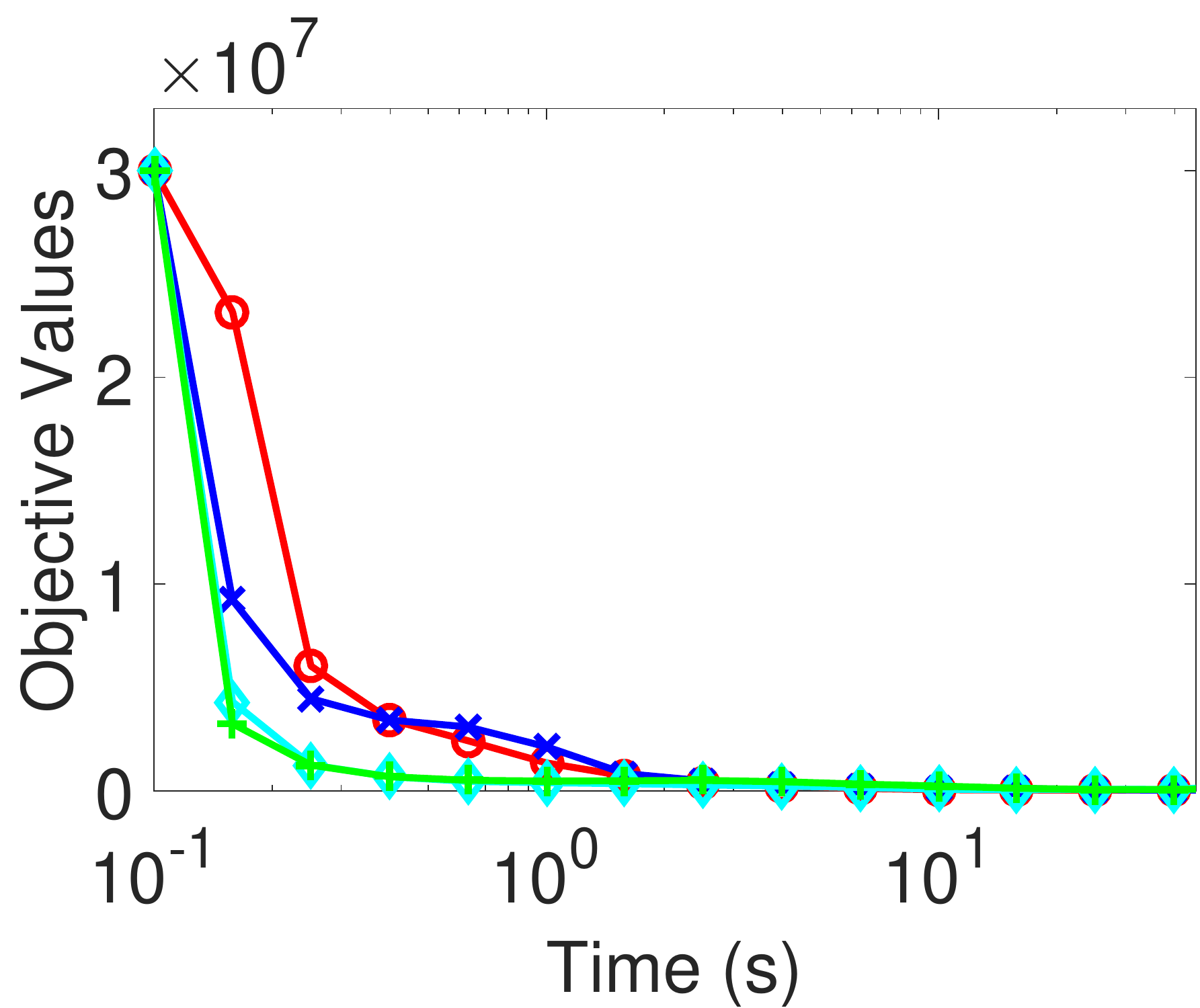}}\hfill
\subfloat[$\ell_2$]{\includegraphics[width=.475\columnwidth,height=.395\columnwidth]{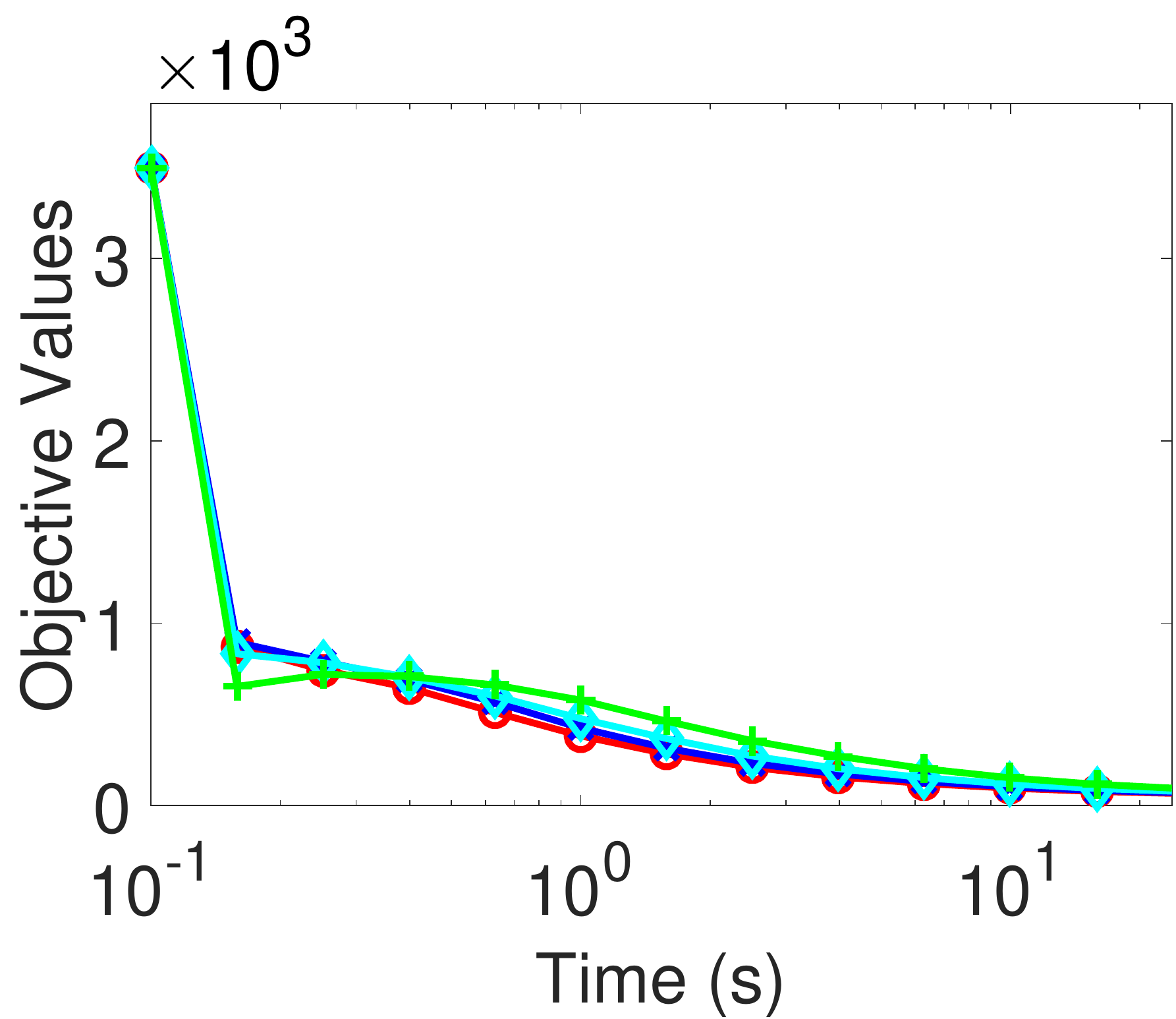}}\vspace{-.4cm}\\
\caption{Objective values versus time (in seconds) of \ol with different values of $K$ for all the divergences in $\barcalD$. $\tau$ and $a$ are in the canonical setting.}\vspace{-.5cm} \label{fig:param_K}
\end{figure}

\begin{figure}[t]
\subfloat[IS]{\includegraphics[width=.475\columnwidth]{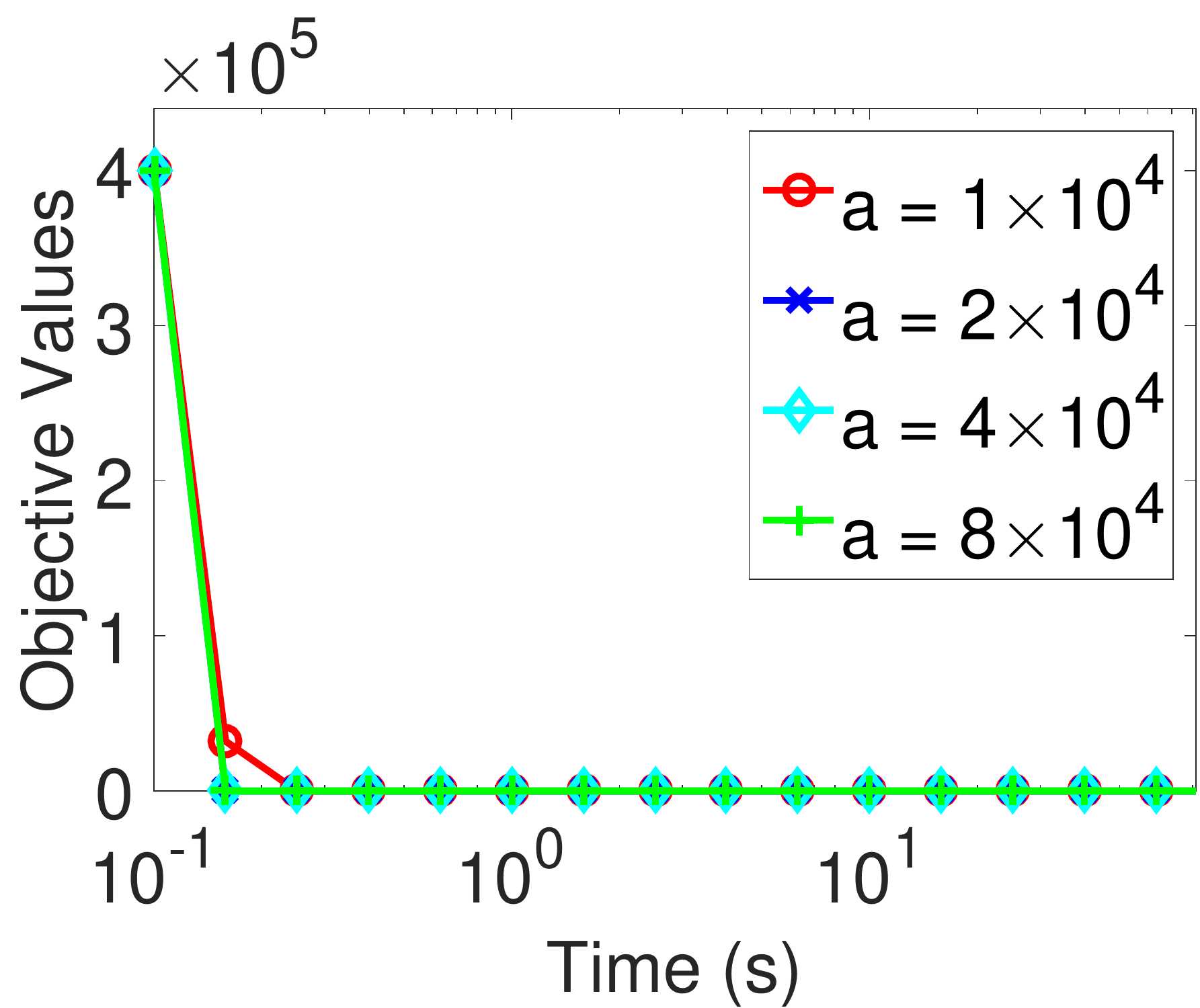}}\hfill
\subfloat[KL]{\includegraphics[width=.475\columnwidth,height=.36\columnwidth]{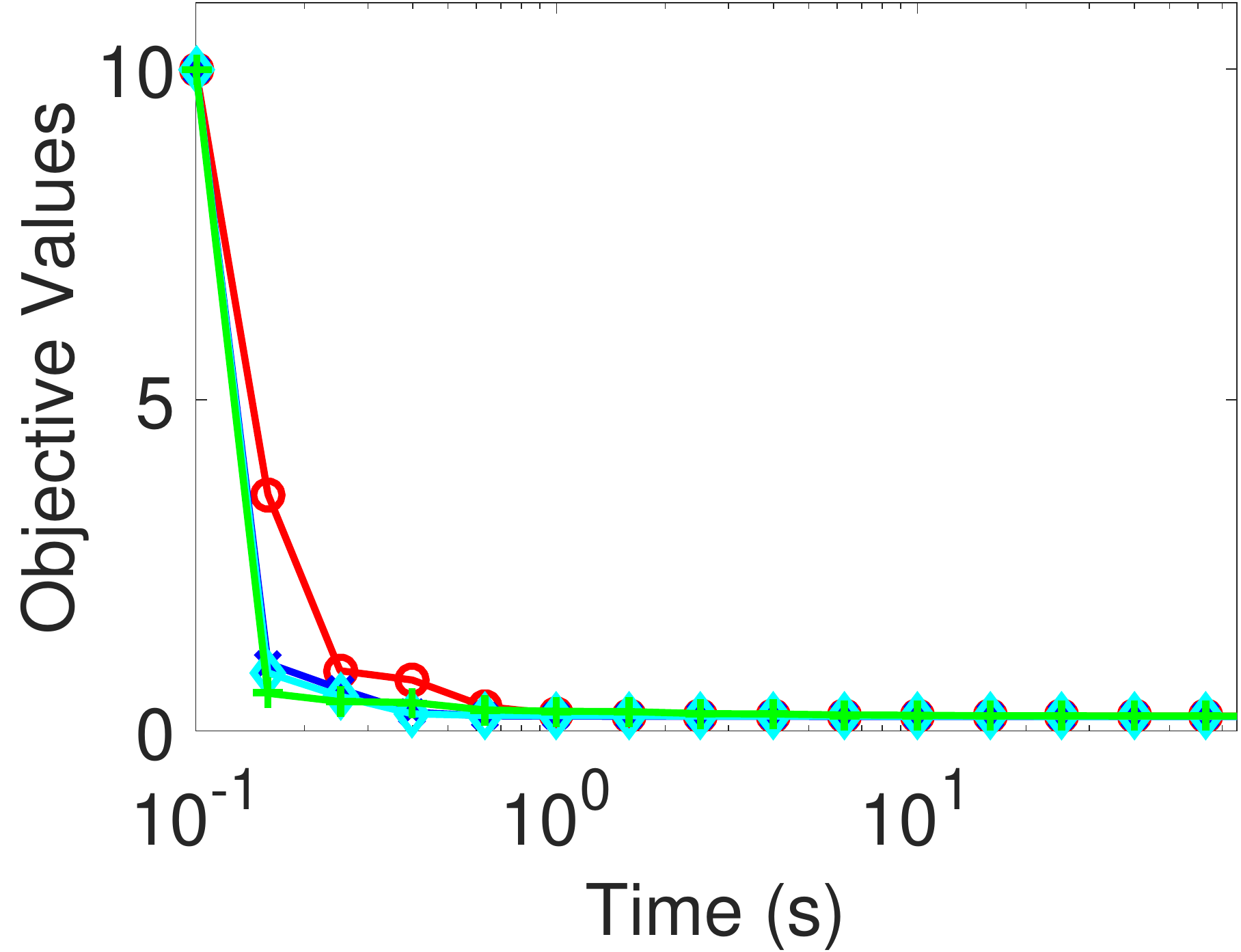}}\vspace{-.4cm}\\
\subfloat[Squared-$\ell_2$]{\includegraphics[width=.475\columnwidth]{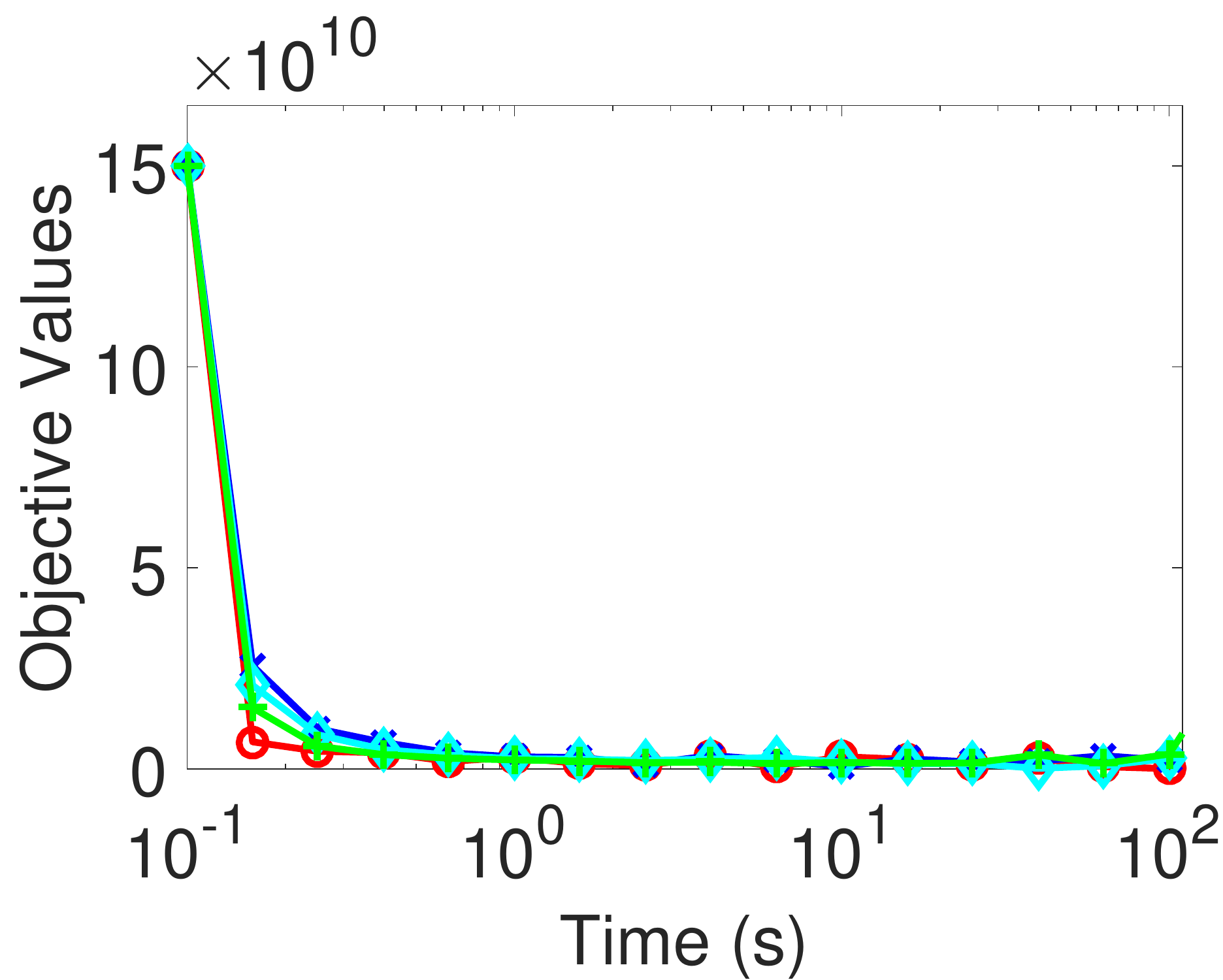}}\hfill
\subfloat[Huber]{\includegraphics[width=.475\columnwidth,height=.38\columnwidth]{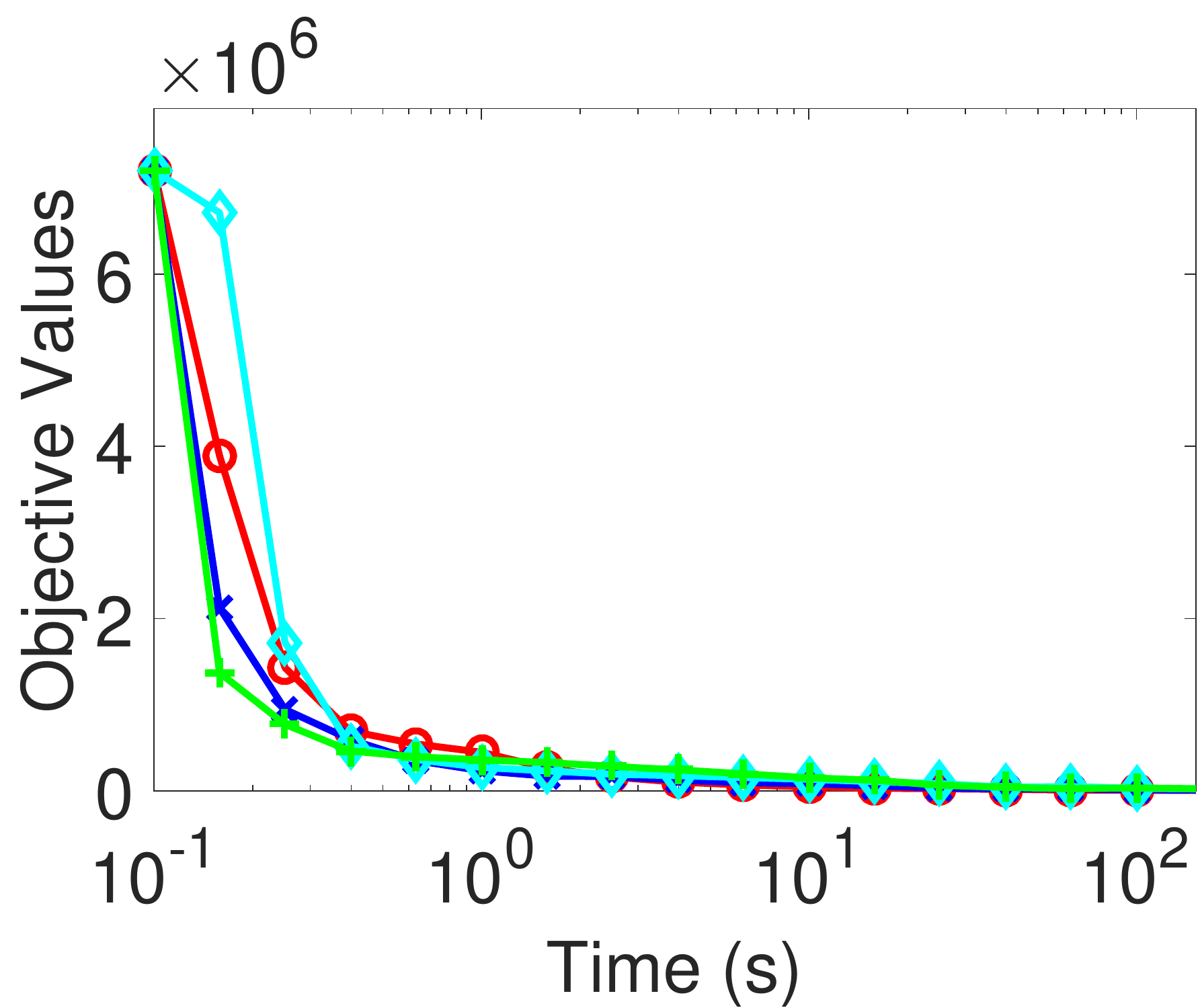}}\vspace{-.4cm}\\
\subfloat[$\ell_1$]{\includegraphics[width=.475\columnwidth]{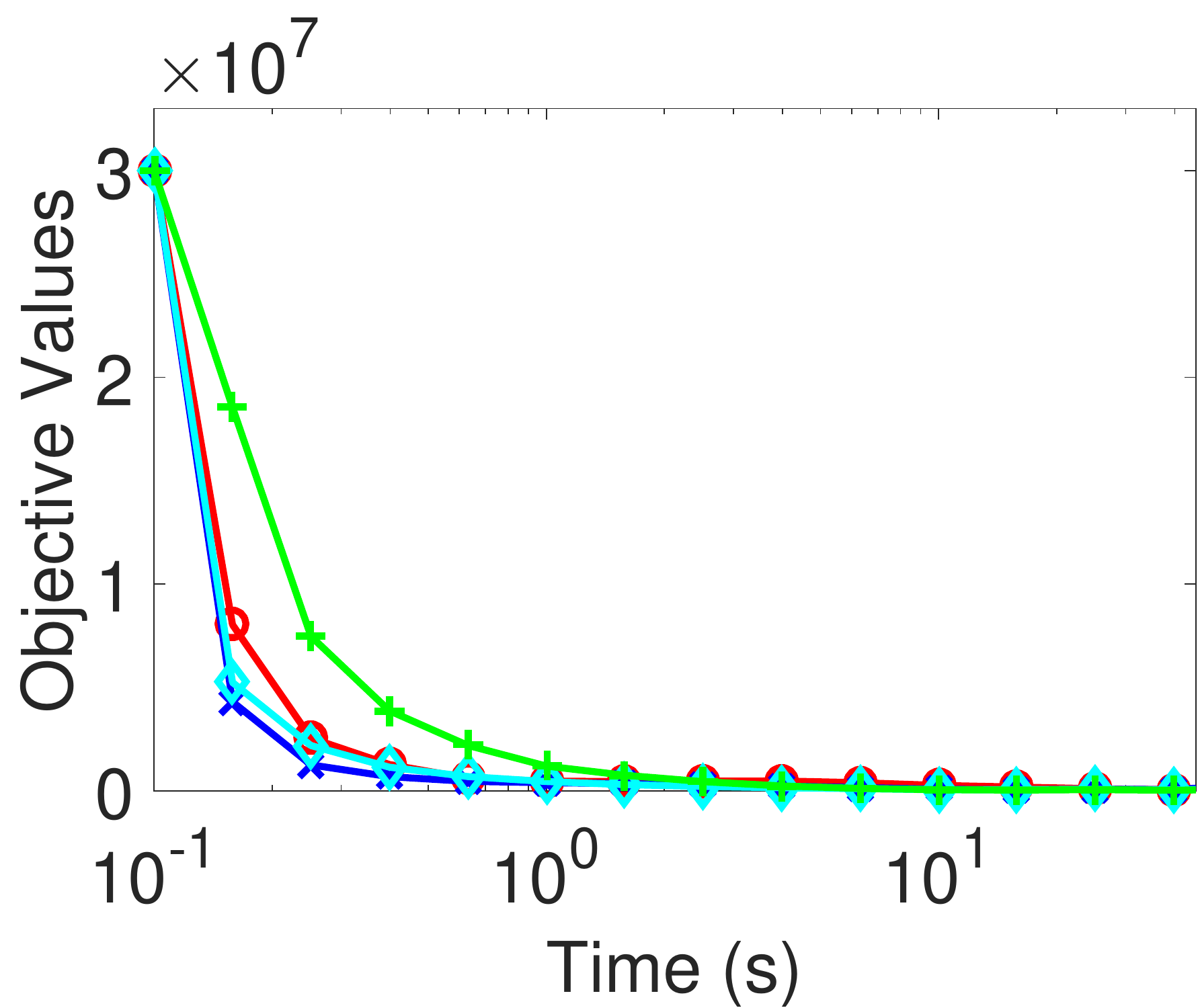}}\hfill
\subfloat[$\ell_2$]{\includegraphics[width=.475\columnwidth,height=.395\columnwidth]{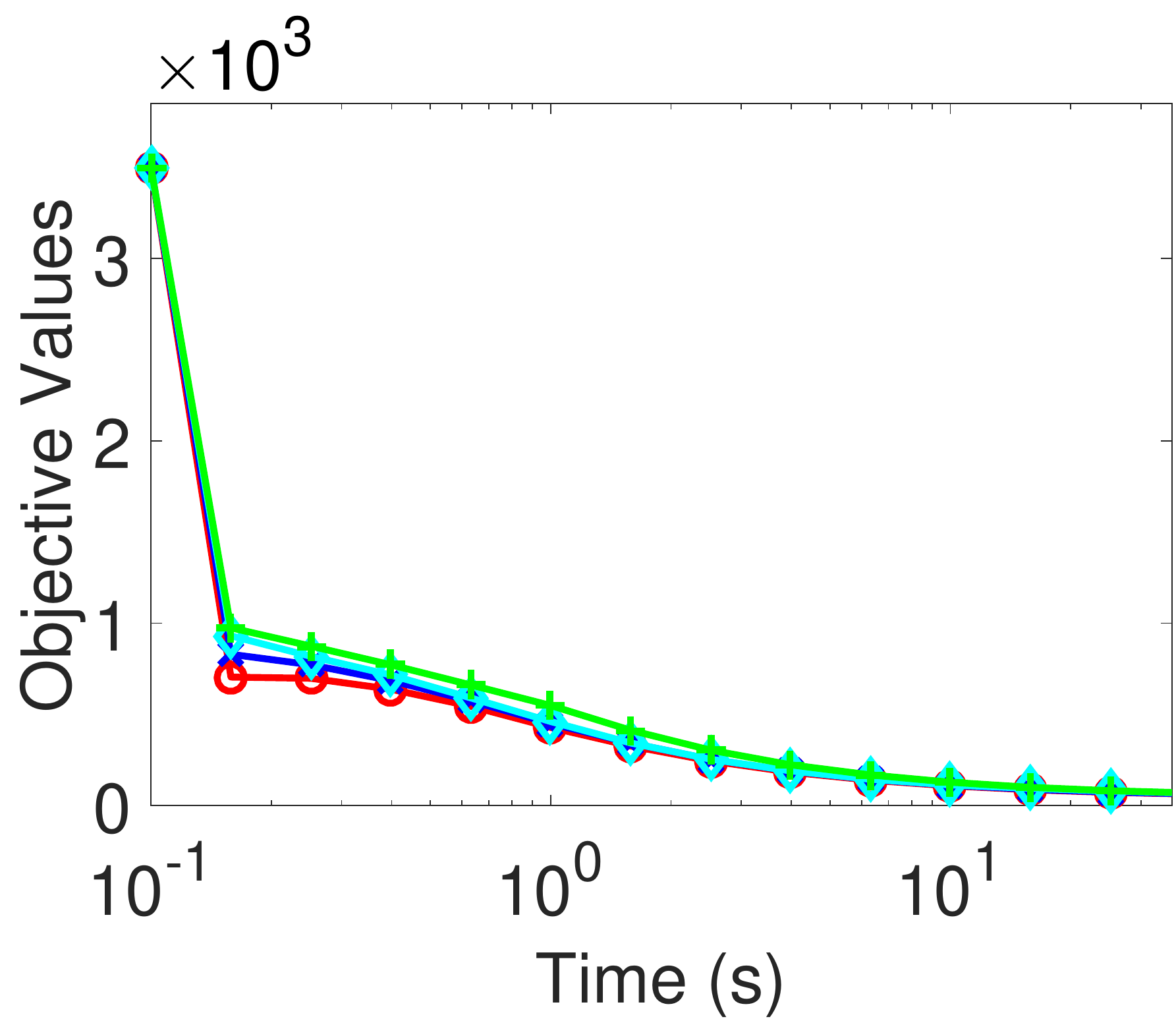}}\vspace{-.4cm}\\
\caption{Objective values versus time (in seconds) of \ol with different values of $a$ for all the divergences in $\barcalD$. $\tau$ and $K$ are in the canonical setting.}\vspace{-.5cm} \label{fig:param_a}
\end{figure}

\bibliographystyle{IEEEtran}
\bibliography{ORNMF_ref,RNMF_ref,bregman_ref,stoc_ref,math_opt,stat_ref,dataset}
\end{document}